\documentclass[twoside]{article}

\usepackage[accepted]{aistats2025}
\usepackage[round]{natbib}

\usepackage{url}
\usepackage{algorithm}
\usepackage{algorithmic}
\usepackage{xspace}
\usepackage[table, svgnames]{xcolor}
\usepackage{soul}                  % better underline and highlight
\usepackage{tikz}                  % fancy symbols
\usepackage{enumitem}              % enumerate spacing
\usepackage{tcolorbox}             % colored box
\usepackage{booktabs}              % tables
\usepackage{multirow}              % multiple rows
\usepackage{dblfloatfix}           % 2-column table/figure
\usepackage[
colorlinks=true,
linkcolor=blue,
urlcolor=violet,
citecolor=gray]{hyperref}          % hyperrefs
\usepackage{fancyhdr}              % page numbering

\usepackage{amsmath, amsfonts, amssymb, amsthm} % for mathematical typesetting
\usepackage{mathtools}
\usepackage{cleveref}
\usepackage{graphicx}
\newtheorem{theorem}{Theorem}[section]
\newtheorem{lemma}[theorem]{Lemma}
\newtheorem{proposition}[theorem]{Proposition}

\newtheorem{definition}{Definition}[section]

\Crefname{equation}{Eqn.}{Eqn.}

\definecolor{wine}{HTML}{830E0D}
\definecolor{nicered}{HTML}{B22222}
\definecolor{niceblue}{HTML}{0000FF}

\newcommand{\fancynumber}[1]{%
  \tikz[baseline=(char.base)]{
    \node[shape=circle,draw=black,fill=wine,inner sep=1pt](char){\color{white}#1};
  }%
}

\fancypagestyle{pagenumbers}{
    \fancyhead[]{}
    \fancyfoot[C]{Page~\thepage}
}

\tcbset{
    challenge/.style={
        colback=blue!10!white,
        coltext=black,           % Text color
        boxrule=1pt,             % Border width
        arc=1.2mm,                 % Border radius
        width=\linewidth,        % Box width
        boxsep=5pt,              % Space between box and text
        left=2pt, right=2pt,     % Inner left and right padding
        top=2pt, bottom=0pt,     % Inner top and bottom padding
    }
}
\tcbset{
    proposition/.style={
        colback=SkyBlue!20,      % Background color
        coltext=black,           % Text color
        boxrule=1pt,             % Border width
        arc=1mm,                 % Border radius
        width=\linewidth,        % Box width
        boxsep=5pt,              % Space between box and text
        left=2pt, right=2pt,     % Inner left and right padding
        top=2pt, bottom=0pt,     % Inner top and bottom padding
    }
}

\newcommand{\methodshort}{GAGA\xspace}

\usepackage[T1]{fontenc}  % Use Type 1 fonts instead of Type 3
\usepackage{lmodern}      % Use Latin Modern fonts (vectorized version of CM fonts)
\usepackage{textcomp}     % Ensure proper symbol rendering
\begin{document}

\runningtitle{Geometry-Aware Generative Autoencoders}
\runningauthor{Sun, Liao, MacDonald, Zhang, Liu, Huguet, Wolf, Adelstein, Rudner, Krishnaswamy}


\twocolumn[

\aistatstitle{Geometry-Aware Generative Autoencoders for Warped Riemannian Metric Learning and Generative Modeling on Data Manifolds}
\vspace*{-10pt}
\aistatsauthor{
Xingzhi Sun$\footnotetext{Equal contrbution.}^{*\heartsuit}$ \hfill
Danqi Liao$^{*\heartsuit}$ \hfill
Kincaid MacDonald$^{\heartsuit}$ \hfill
Yanlei Zhang$^{\diamondsuit}$ \hfill
Chen Liu$^{\heartsuit}$ \hfill
\textbf{Guillaume Huguet}$^{\diamondsuit}$\hspace{2pt}\\[2pt]
\textbf{Guy Wolf}$^{\diamondsuit}$\hspace{2pt}
\textbf{Ian Adelstein}$^{\heartsuit \dag}$\hspace{2pt}
\textbf{Tim G. J. Rudner}$^{\clubsuit \dag}$\hspace{2pt}
\textbf{Smita Krishnaswamy}$^{\heartsuit \diamondsuit \dag}$\footnotetext{Corresponding authors: \url{ian.adelstein@yale.edu}, \url{tim.rudner@nyu.edu}, \url{smita.krishnaswamy@yale.edu}.}
\\[2pt]
% \url{ian.adelstein@yale.edu} \quad \url{tim.rudner@nyu.edu} \quad \url{smita.krishnaswamy@yale.edu}
}

\aistatsaddress{
$^{*}$Equal contribution. $^{\dag}$Corresponding authors.\\[2pt]
$^{\heartsuit}$Yale University \quad
$^{\spadesuit}$New York University \quad 
$^{\diamondsuit}$Mila - Quebec AI Institute and Universite de Montr\'eal}
]

\begin{abstract}
\vspace*{-10pt}Rapid growth of high-dimensional datasets in fields such as single-cell RNA sequencing and spatial genomics has led to unprecedented opportunities for scientific discovery, but it also presents unique computational and statistical challenges. Traditional methods struggle with geometry-aware data generation, interpolation along meaningful trajectories, and transporting populations via feasible paths. 
% To address these issues, we introduce Geometry-Aware Generative Autoencoder (\methodshort), a novel framework that combines extensible manifold learning with generative modeling. \methodshort constructs a neural network embedding space that respects the intrinsic geometries discovered by manifold learning, enabling the derivation of a novel warped Riemannian metric.
To address these issues, we introduce Geometry-Aware Generative Autoencoder (\methodshort), a novel framework that combines extensible manifold learning with generative modeling. \methodshort constructs a neural network embedding space that respects the intrinsic geometries discovered by manifold learning and learns a novel \emph{warped} Riemannian metric on the data space.
% that characterizes the data geometry in the data space.
This warped metric is derived from both the points on the data manifold and negative samples off the manifold, allowing it to characterize a meaningful geometry across the entire latent space.
Using this metric, \methodshort can uniformly sample points on the manifold, generate points along geodesics, and interpolate between populations across the learned manifold. \methodshort shows competitive performance in simulated and real-world datasets, including a 30\% improvement over SOTA in single-cell population-level trajectory inference.

\end{abstract}

\begin{figure}[tb!]
    \centering
    \includegraphics[width=\columnwidth]{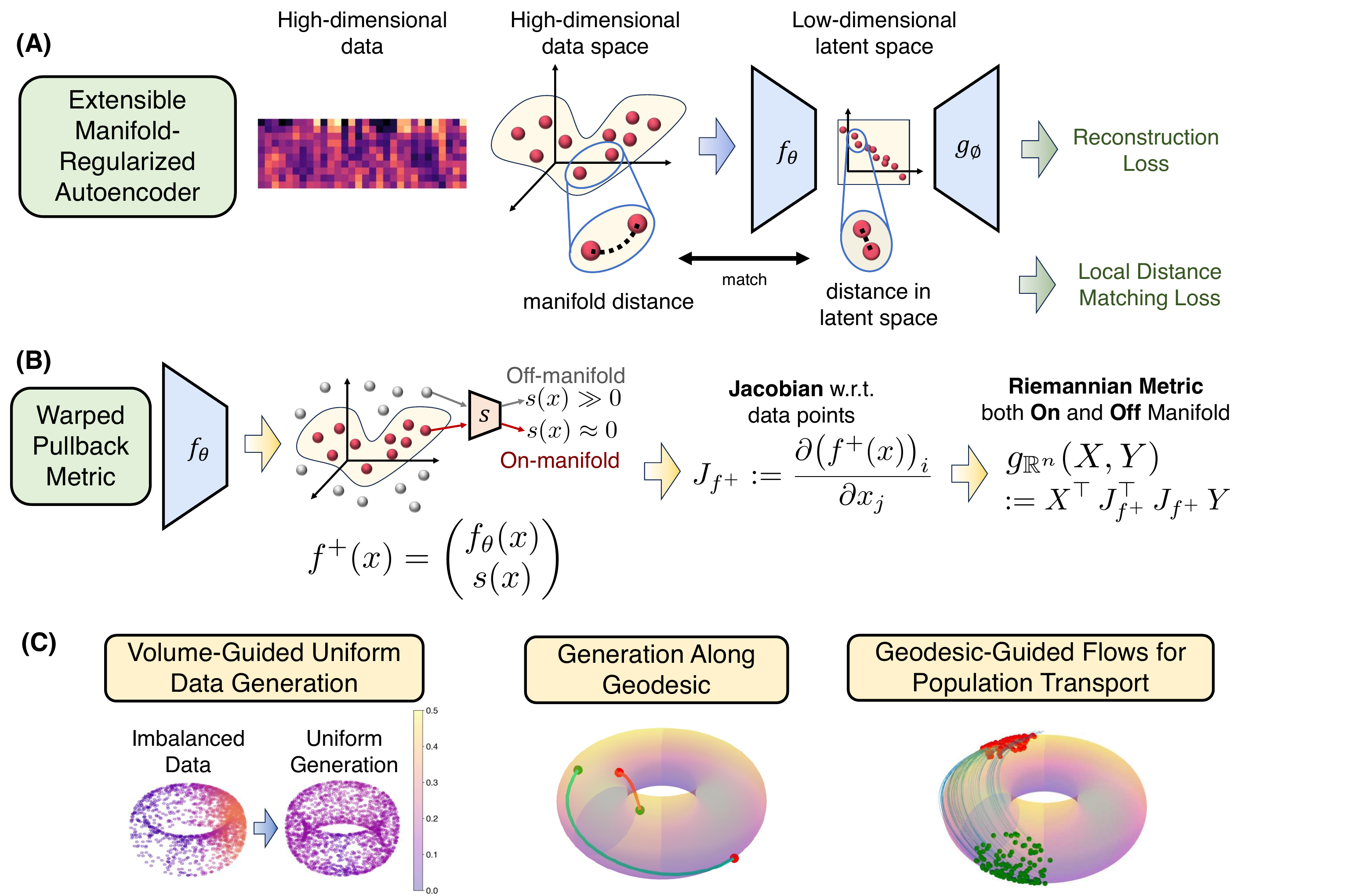}
    \vspace{-16pt}
    \caption{The Geometry-Aware Generative Autoencoder~(\methodshort) framework. \textbf{(A)} Training the networks. \textbf{(B)} Obtaining the warped pullback metric. \textbf{(C)} Challenging applications enabled by \methodshort.}
    \label{fig:schematic}
\vspace{-12pt}
\end{figure}

% Do not add the following sections to table of contents.
% Table of contents for Appendix only.
\addtocontents{toc}{\protect\setcounter{tocdepth}{-1}}

\vspace*{-2pt}
\section{INTRODUCTION}
\vspace*{-4pt}

Recent scientific discoveries are increasingly driven by the analysis of high-dimensional data across various fields, including single-cell RNA sequencing (scRNA-seq), spatial genomics, and many others~\citep{scRNAseq_discovery1, scRNAseq_discovery2, ATAQseq_discovery1, Spatial_discovery1, Spatial_discovery2}. These high-dimensional datasets offer unprecedented opportunities to explore complex physical and biological systems, but they also pose unique computational and statistical challenges.

\fancynumber{1}~First, it is difficult to generate new data points that faithfully follow the underlying data geometry (for example, to combat inconsistent or undersampling in parts of the data manifold) in the absence of explicit analytical forms describing the data, especially when data imbalance complicates the process~\citep{Imbalanced_data}. \fancynumber{2}~Second, interpolating between two samples along a meaningful trajectory, which is valuable for understanding transitions such as developmental progressions, remains challenging due to the complex and non-linear structure of the data~\citep{Distance_high_dim}. \fancynumber{3}~Third, aligning or transporting populations across different experimental conditions, time points, or biological states is a fundamental challenge, as traditional matching methods often fail to capture the complex dependencies and interactions inherent in high-dimensional spaces~\citep{Distribution_matching_challenges}.

When working with high-dimensional data, it is useful to consider the manifold hypothesis, which posits that such data often reside on a lower-dimensional manifold embedded within the high-dimensional data space~\citep{Manifold_hypothesis}. Building on this foundation, we propose a novel framework called the Geometry-Aware Generative Autoencoder~(\methodshort) to simultaneously address all three challenges.

\methodshort combines the power of extensible manifold learning with generative modeling. It first learns a generalizable neural network embedding space that respects the geometries discovered by non-linear dimensionality reduction techniques~(Figure~\ref{fig:schematic} panel A). 
% Then, it derives a non-Euclidean pullback metric in the original data space~(Figure~\ref{fig:schematic} panel B). 
Then, it derives a novel \emph{warped} pullback metric on the original data space~(Figure~\ref{fig:schematic} panel B). Uniquely, this metric is created as much by points not in the dataset as by points that are in the data. The warped metric is learned by embedding negative samples \textbf{off} the manifold and points \textbf{on} the manifold far away from each other in the latent space. This creates an implicit penalty for data generation and  geodesic computations, effectively nudging geodesics to stay within the data density, and generated points to stay within dimensions of the data.

% \xin{
% Notably, we warp the latent space by embedding negatively sampled off-manifold points, from regions where no data is measured, far away from the embeddings of the observed data on the manifold. This extension transforms the pullback metric into a warped Riemannian metric that covers the entire data space, ensuring accurate model behavior both within and beyond the observed data regions.
% }

Using this learned warped Riemannian metric, \methodshort can \fancynumber{1}~generate data across the data manifold guided by local volume, \fancynumber{2}~interpolate between two points along the manifold geodesics, and \fancynumber{3}~transport populations along these geodesics. These applications are illustrated in Figure~\ref{fig:schematic} panel C, and are described in Sections~\ref{sec:data_generation}, \ref{sec:interpolation}, and \ref{sec:population_transport}, respectively. In this way, \methodshort effectively addresses the challenges of geometry-aware data generation, interpolation, and population transport within a unified framework.

In summary, our main contributions are as follows:\vspace*{-3pt}
\begin{enumerate}[itemsep=-1pt, topsep=0pt]
    \item Designing a geometry-aware generative autoencoder that combines manifold learning with generative modeling.
    \item Proposing a novel \emph{warped} pullback metric to create a meaningful geometry on the entire data space, allowing \methodshort to stay on the manifold when generating points. 
    \item Introducing a new generative method that leverages the learned Riemannian pullback metric to achieve uniform sampling from the data manifold, interpolating data along geodesics, and transporting populations along geodesic paths.
    \item Demonstrating that the proposed methods work well on both simulated and real biological data.
\end{enumerate}

\newpage

\section{BACKGROUND}
% {\bf Manifold Learning and Diffusion Geometry.}~
\vspace*{-4pt}

\paragraph{Manifold Learning}
The \textit{Manifold Hypothesis} states that data often lie \textit{on} or \textit{near} a low-dimensional manifold within high-dimensional space~\citep{Manifold_hypothesis}. Manifold learning methods such as Diffusion Maps~\citep{DiffusionMaps}, PHATE~\citep{PHATE}, DSE~\citep{DSE}, DYMAG~\citep{DYMAG}, CUTS~\citep{CUTS}, and HeatGeo~\citep{HeatGeo} use diffusion probabilities to recover the geometry of the manifold despite the sparsity and noise in the data.
For details, see \Cref{appx:mfd_learning}.
\vspace{-4pt}

\paragraph{Riemannian Manifolds and Metrics}
An $n$-dimensional manifold $\mathcal{N}$ is a space locally resembling $\mathbb{R}^n$, and a Riemannian metric $g$ endows each tangent space $T_x\mathcal{N}$ with an inner product
$
g_x(X,Y)=X^Tg(x)Y,
$
with $g(x)$ an $n\times n$ matrix. The length of a tangent vector $X$ is $\|X\|=\sqrt{g_x(X,X)}$, and for a smooth curve $c:[0,T]\to\mathcal{N}$ the length is
$
L(c)=\int_0^T\sqrt{g_{c(t)}\bigl(\dot{c}(t),\dot{c}(t)\bigr)}\,dt.
$
If $\mathcal{N}$ is parametrized by $f(z)$, $z\in\mathcal{D}$, its volume is given by
$
\int_{\mathcal{D}}\sqrt{\det g(x)}\,dx.
$

A key component of our method is the \textit{Riemannian pullback metric}. Given a map $f:\mathcal{M}\to(\mathcal{N},g)$, its differential $df_x:T_x\mathcal{M}\to T_{f(x)}\mathcal{N}$ allows us to define
$
f^*g(X,Y)=g(df_xX,df_xY),
$
which equips $\mathcal{M}$ with the geometry inherited from $(\mathcal{N},g)$. For a detailed discussion, see \Cref{appx:reimannian}.

\vspace*{-2pt}
\section{METHODS}
\vspace*{-4pt}

In this section, we will describe the autoencoder and derive the Riemannian pullback metric~(Section~\ref{sec:dist_match}). Then, we will show solutions to the three challenges: geometry-aware data generation~(Section~\ref{sec:data_generation}), interpolation along meaningful trajectories~(Section~\ref{sec:interpolation}), and population transport~(Section~\ref{sec:population_transport}). Proofs for all lemmas and propositions are provided in \Cref{appx:proof}.

% \Cref{sec:dist_match} presents the distance-matching autoencoder, with an auxiliary dimension informed by a discriminator.
% This provides a pullback metric on the data space that is matched to the data metric for points on the data manifold, and induces a large distance to the manifold for points off the manifold.
% \Cref{sec:interpolation} solves the problem of learning geodesics between points on the manifold using this pullback metric.
% \Cref{sec:population_transport} generalizes this to generating geodesics between populations of points using geodesic-guided flow matching.

% \subsection{Learning a Geometry-Aware Data Encoding with pullback metric that puts off manifold points away from the embedding \xin{work on this title} \danqi{Geometry-Aware encoding for both on-manifold and off-manifold points}}
\subsection{Geometry-Aware Encoding for Both On-Manifold and Off-Manifold Points}
\label{sec:dist_match}
% Traditional autoencoders encode data points $x \in \mathcal{X}$ into a low-dimensional latent space and then decode encoded points back into the original space.
% However, latent representation distorts the geometry of the data manifold, specifically, the points that are closer to each other on the data manifolds are not guaranteed to have shorter distance in the latent space.
% To solve this problem, we propose an autoencoder that locally matches the Eulidean distances in the latent space with the geodesic distances on the manifold.

% We first train an autoencoder to learn a latent space whose local Euclidean distances correspond to the data manifold distances. These distances can be obtained from many existing manifold-learning techniques, including Diffusion Maps, PHATE, and HeatGeo. We then define a \textit{warped} Euclidean metric on data space that allows us to produce a metric that corresponds to local Euclidean distances on the manifold, and a warping that imposes large distances for points off the manifold. This metric allows us to compute on-manifold geodesics for data generation in later sections.

We first train an autoencoder to learn a latent space whose local Euclidean distances correspond to the data manifold distances. These distances can be obtained from many existing manifold-learning techniques, including PHATE and HeatGeo. We then derive a \textit{warped} metric on data space that allows us to produce a pullback Riemannian metric on the data manifold and impose large distances for points off the manifold. This warped metric enables us to compute on-manifold geodesics for data generation in later sections.

% Specifically, this geometry match can be quantified as: the latent manifold and the data manifold share an equivalent Riemannian metric.
%Specifically, this geometry match can be quantified as: through the encoder pullback from the metric on the latent manifold, we recover the metric of the data data manifold.
% \xin{the pullback of the encoder is the same as the metric on the data (to rephrase), no need to mention M and N?}
%Intuitively, we want to ensure that the geodesic distances on the two manifolds are matched. 

The following result from Riemannian geometry states that by matching data manifold distances in latent space (i.e., learning a local isometry), we construct the desired pullback metric on the data manifold.

%that in the ideal case, matching geodesic distances is an equivalent condition for matching Riemannian metrics through pullback, further supporting this intuition.

% \scalebox{0.96}{
% \begin{minipage}{1.02\linewidth}
\begin{proposition}
\label{prop:dist_match}
    For Riemannian manifolds $(\mathcal M,g_{\mathcal M}),(\mathcal N,g_{\mathcal N})$ and diffeomorphism $f:\mathcal M\to\mathcal N$, if $f$ is a local isometry, i.e., there exists $\epsilon>0, $ such that for any $x_0,x_1\in\mathcal M, d_{\mathcal M}(x_0,x_1)<\epsilon\implies d_{\mathcal M}(x_0,x_1)=d_{\mathcal N}(f(x_0),f(x_1))$, then we have $g_{\mathcal M}=f^*g_{\mathcal N}.$
\end{proposition}    
% \end{minipage}
% }
\vspace{-6pt}

% \xin{this is ok, will try to make it more about the pullback}
% We propose an autoencoder that respects the geometry of the data.
% Specifically, the local Euclidean distance between two points in the latent space is matched with their geodesic distance on the data manifold.

To implement this construction, we define an autoencoder consisting of an encoder $\smash{f_\theta}$ and a decoder $\smash{h_\phi}$, both parameterized by neural networks.
% $f_\theta$ maps point $\smash{x\in\mathbb R^n}$ to a low-dimensional latent embedding $\smash{z\in\mathbb R^m}$, and $\smash{h_\phi}$ maps $z$ back to $\smash{\hat x\in\mathbb R^n}$ 
The autoencoder is jointly optimized with a reconstruction objective (\Cref{loss:recon}) and a local distance matching objective (\Cref{loss:distance_matching}).\vspace*{-3pt}
% \scalebox{0.76}{
% \begin{minipage}{1.3\linewidth}
\begin{align}
\label{loss:recon}
    \mathcal{L}_\mathrm{Recon}(\theta,\phi)
    = &\frac{1}{N}\sum_{i=1}^N||x_i-h_{\phi}(f_{\theta}(x_i))||_2^2\\
\label{loss:distance_matching}
    \mathcal{L}_\mathrm{Dist}(\theta)
    =& \frac{1}{N}\sum_{i<j}e^{-\zeta d(x_i,x_j)} \ell_{\textrm{SE}}(x_i, x_j, \theta) ,
\end{align}
% \end{minipage}
% }
where
\begin{align}
    \ell_{\textrm{SE}}(x_i, x_j, \theta)
    =
    \left(||f_{\theta}(x_i)-f_{\theta}(x_j)||_2-d(x_i,x_j)\right)^2,
\end{align}
$x_1, \ldots, x_N$ are the data samples, and $d(x_i, x_j)$ is the manifold distance between points $x_i$ and $x_j$ obtained via selected manifold-learning methods.
The hyperparameter $\zeta > 0$ and the term $e^{-\zeta d(x_i, x_j)}$ weigh the penalty towards the more important local geometry of the data manifold.

In summary, we minimize the following objective (\Cref{loss:total}) with respect to encoder and decoder parameters $\theta$ and $\phi$ to obtain geometry-aware embeddings.\vspace*{-3pt}
\begin{align}
\label{loss:total}
    \mathcal{L(\theta, \phi)}=\lambda_1 \mathcal{L}_\mathrm{Dist}(\theta) + \lambda_2 \mathcal{L}_\mathrm{Recon}(\theta, \phi)
\end{align}
% \xin{by construction this balances the dist and recon, thereby getting a good emb... (make the reason explicit for readers)}
This objective balances distance matching and reconstruction with hyperparameters $\lambda_1, \lambda_2$.
It results in an embedding that matches the data geometry and retains the information needed to reconstruct the data.

\paragraph{Pullback metric}
Next we show how to compute the pullback metric via the Jacobian of the encoder. The pullback (via the encoder) of the Euclidean metric from latent space yields a non-Euclidean data space metric, capturing local distances on the data manifold.

% \xin{describe the pullback metric}
\begin{definition}
\label{def:gm}
    The pullback of the Euclidean metric from latent space to the data manifold $\mathcal M$ is defined by $g_{\mathcal M}(X,Y):=X^\top J_{f}^\top J_{f} Y$, where $X,Y\in T_{x}\mathcal M$ are tangent vectors at $x\in\mathcal M$, $J_f := \partial f_{\theta}(x)_i / \partial x_j$ is the Jacobian of $f_\theta$ at $x$.
\end{definition}

% \paragraph{Off-manifold penalty}

\paragraph {Warping the Local Euclidean Metric}
Although the construction above produces a pullback metric on the entire data space, it is only accurate near the training data, i.e., along the data manifold. For points off of the manifold, we  the local Euclidean metric to create large distances between on-and-off manifold points.  In order to achieve this,  we create a special embedding for both on-manifold points $x_i$ and off-manifold points $\check x_i$. These points are embedded in a latent space with an auxiliary dimension, where the value of that dimension represents the deviation from the manifold: it is nearly zero for on-manifold points and large for off-manifold points. 
% We only assume its range to be all non-negative real numbers, but do not assume additional properties to make it a valid distance measure.

Suppose we have a function $s(x)$ such that $s(x)\approx 0$ for $x$ on the manifold, and $s(x)$ increase as $x$ moves away from the manifold. Let
% \scalebox{0.84}{
% \begin{minipage}{1.15\linewidth}
\begin{align}
\label{expn:extn}
    \hspace*{-3pt}f^{+}(x) =
    \left(\begin{matrix}
        f_\theta(x)
        \\
        \beta s(x)
    \end{matrix}\right)
    \text{, where }
        \beta \text{ is a hyperparameter}
\end{align}
% \end{minipage}
% }
\begin{definition}
\label{def:gr}
    % The extended pullback of the Euclidean metric from latent space to the full data
    The pullback of the warped local Eulidean metric on the full space $\mathbb R^n$ is defined by $g_{\mathbb{R}^{n}}(X,Y):=X^\top J_{f^{+}}^\top J_{f^{+}} Y$, where $X, Y \in T_{x} \mathbb{R}^n$ are tangent vectors at $x\in \mathbb R^n$, $J_{f^{+}} := \partial (f^{+}(x))_i / \partial x_j$ is the Jacobian of $f^{+}$ at $x$.
\end{definition}

Points off the manifold, where $s(x)$ is large, are placed into an extended dimension of latent space, far from the on-manifold points. 
%Furthermore, $r$ guarantees that if a point is away from the data manifold, $r$ maps it far away from where the manifold is mapped. 
Formally, we have:

\begin{lemma}
\label{lem:bilip}
% \xin{"l" is hard to parse, so change to alpha in this and future propositions. also avoid name clash with the equation 2.}
    If there exists $\alpha \in \mathbb{R}$ such that for any 
    % $x, \check x, \alpha||x-\check x||\leq |w_\psi(x)-w_\psi(\check x)|$.
    $x, \check x, \alpha||x-\check x||\leq |s(x)-s(\check x)|$.
    Then for any $x,\check x, ||f^{+}(x)-f^{+}(\check x)||\geq \alpha \beta ||x-\check x||$.
    Furthermore, denoting \mbox{$\mathcal D_{\mathcal M}(y):=\inf\nolimits_{x\in\mathcal M}||x-y||$} and %the distance from the point $y$ to the manifold $\mathcal M$
    \mbox{$\mathcal D_{f^{+}(\mathcal M)}(y):=\inf\nolimits_{x\in\mathcal M}||f^{+}(x)-f^{+}(y)||$}, then for any $\check x,$ we have $ \mathcal D_{f^{+}(\mathcal M)}(\check x)\geq \alpha \beta \mathcal D_{\mathcal M}(\check x)$.
\end{lemma}

This lemma assumes that $s$ satisfies a Lipschitz condition, meaning it grows moderately. In our framework, $s$ is a Wasserstein-GAN-style discriminator, and we enforce Lipschitz continuity via weight clipping and spectral normalization (see \Cref{appx:lipshitz}). This approach is supported by WGAN \citep[Section 3]{WGAN}, spectral normalization \citep[Section 2.1]{miyato2018spectral}, and Lipschitz discriminators \citep[Section 3]{tong2022fixing}.
In practice, we obtain such function $s(x)$ by training a discriminator with negative sampling. See \Cref{appx:auxiliary_dimension} for details.

% \xin{It would require bi-Lipschitz, which is not necessarily true, esp due to the variance loss!}

% \subsection{\mbox{\hspace*{-1pt}Designing\hspace*{1.5pt}a\hspace*{1.5pt}Globally\hspace*{1.5pt}Meaningful\hspace*{1.5pt}Riemannian\hspace*{1.5pt}Metric}}

% We can now use \Cref{expn:extn} to extend the pullback metric on the data manifold to the entire data space:

% \begin{definition}
% \label{def:gm}
%     The pullback of the Euclidean metric from latent space to the data manifold $\mathcal M$ is defined by $g_{\mathcal M}(X,Y):=X^\top J_{f}^\top J_{f} Y$, where $X,Y\in T_{x}\mathcal M$ are tangent vectors at $x\in\mathcal M$, $J_f := \partial f_{\theta}(x)_i / \partial x_j$ is the Jacobian of $f_\theta$ at $x$.
% \end{definition}

% Using the map defined in \Cref{expn:extn} and the definition of a pullback metric, we can obtain a pullback metric on the data space that allow effective computation of geodesics on the data manifold.

% \xin{explain why we have those 2 defns, maybe make them m}
% \begin{definition}
% \label{def:gr}
%     % The extended pullback of the Euclidean metric from latent space to the full data
%     The warped Eulidean metric on the full space $\mathbb R^n$ is defined by $g_{\mathbb{R}^{n}}(X,Y):=X^\top J_{f^{+}}^\top J_{f^{+}} Y$, where $X, Y \in T_{x} \mathbb{R}^n$ are tangent vectors at $x\in \mathbb R^n$, $J_{f^{+}} := \partial (f^{+}(x))_i / \partial x_j$ is the Jacobian of $f^{+}$ at $x$.
% \end{definition}

% \xin{explain context around these definitions, they are stray. e.g. this is a general def, this is using some }
Note that $g_\mathcal M$ is defined only on the tangent space of $\mathcal M$, whereas the warping allows $g_{\mathbb R^{n}}$ to be defined on the tangent space of the entire data space $\mathbb R^n$.
% ------------------------------------------------------------

\subsection{Using the Learned Pullback Metric to Generate Uniformly on the Manifold}
\label{sec:data_generation}
\vspace{-6pt}

\begin{tcolorbox}[challenge]
Tackling Challenge 1: \textit{Volume-Guided Generation}.
\end{tcolorbox}

\begin{figure}[tb!]
    \centering
    \includegraphics[width=1.0\columnwidth]{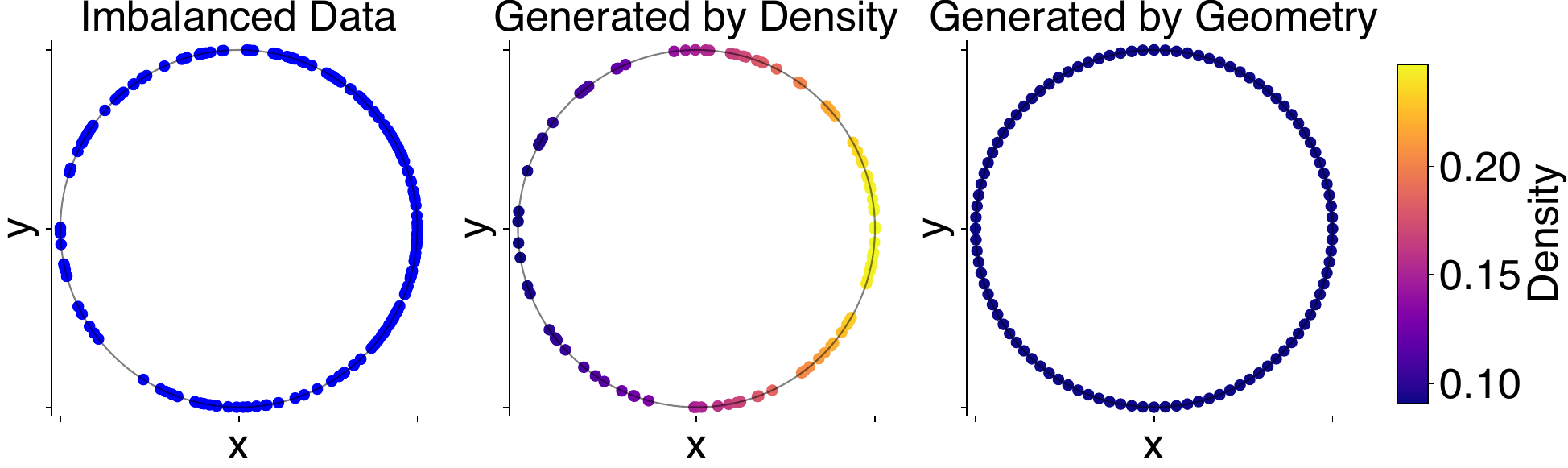}
    \caption{Density-based vs geometry-based generation. Left: Data has sampling imbalance. Middle: Density-based methods, e.g. Diffusion Model and Flow Matching, maintain this bias. Right: Geometry-aware generation alleviates imbalance by generating points uniformly across the manifold.}
    \label{fig:von_mises}
    \vspace{-8pt}
\end{figure}

% Many real-world datasets have the problem of unbalanced sampling along the manifold; for example, a single cell dataset that contains multiple batches can have different batch sizes due to varied experiment setups. 

Here we present a method for sampling uniformly across the data manifold. Notice that this method of generation is \textbf{markedly different} from generative methods that match distributions (and practically mainly the modes of the distribution) such as GANs and diffusion models. Here, rather than sampling from a \textit{probability distribution}, we sample from the \textit{geometry} or the shape of the data evenly. To do this, we utilize the pullback metric that we defined in the previous section to create a volume element that is useful for generation. By utilizing this learned metric, \methodshort enables us to correct for sampling biases and imbalances, ensuring uniform coverage of the manifold during data generation. See \Cref{fig:von_mises} for an illustration of this difference. 

We begin by defining the volume distribution function, which represents a uniform distribution on the manifold based on its intrinsic geometry.

\begin{definition}
\label{def:vol_dens}
    Let $g_{\mathcal M}$ be the Riemannian metric, of the manifold, define the \emph{volume distribution function} $p_{\text{vol}}(x)=\frac{1}{Z} \sqrt{\det g_{\mathcal M}(x)}$, where $Z=\int_{x\in \mathcal M}\sqrt{\det g_{\mathcal M}(x)}dx$, as the normalized volume element normalized to sum to 1. The corresponding probability distribution is defined as the \textit{uniform distribution on the manifold}.
\end{definition}

The intuition behind \Cref{def:vol_dens} can be illustrated with the example shown in \Cref{fig:spiral_vol}. Consider a spiral, which is a one-dimensional manifold. In this case, points are uniformly distributed along the spiral such that the curve lengths between adjacent points are equal. This is achieved by placing more points where the curve length (i.e., volume) is larger, ensuring that the point density remains consistent along the entire manifold. Essentially, the number of points per unit curve length remains constant, which makes the point density proportional to the volume element.

\begin{figure}[tb!]
    \centering
    \includegraphics[width=0.7\columnwidth]{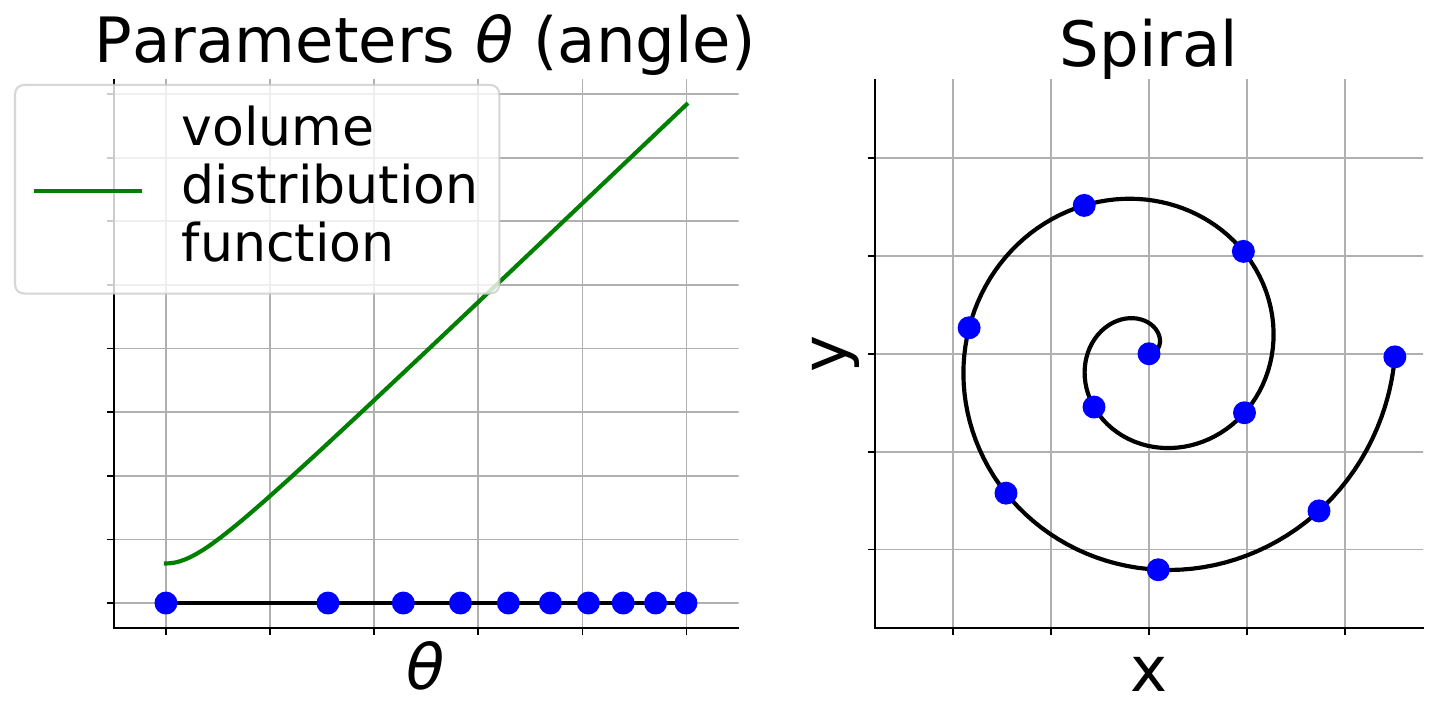}
    \vspace{-8pt}
    \caption{Demonstration of uniform sampling on a spiral (a 1D manifold). Left: In the space parameterized by polar angle, data (blue points) are distributed with density proportional to the volume distribution function (green curve), and may appear non-uniform. Right: In fact, corresponding data on the manifold (blue points) are equally spaced w.r.t. geodesic distance, and are therefore ``uniformly distributed''.}
    \label{fig:spiral_vol}
    \vspace{-8pt}
\end{figure}

Next, we propose an algorithm for generating uniformly on the manifold using Langevin dynamics, combined with the pullback metric learned by \methodshort. Our approach leverages Langevin dynamics to sample points while following the volume distribution function derived from the pullback metric, ensuring that generated points remain faithful to the manifold's intrinsic geometry. Specifically, we solve the following stochastic differential equation~(SDE).\vspace*{-3pt}
\begin{align}
\label{eqn:ld_vol}
\begin{split}
    dX_t &=-\nabla f_{\mathrm{target}}(X_t)dt+\sqrt{2}dW_t \\
    f_{\mathrm{target}}(x) &=\lambda s(x)-\log(f_{vol}(x)) ,
\end{split}
\end{align}
where $W_t$ represents Brownian motion. $f_{vol}(x):=\big|\prod_{i=1}^{d}\sigma_i(x)\big|$, and $\sigma_i(x), i=1,\dots, d$ are the singular values of the Jacobian matrix $J_f(x)$. This corresponds to the volume distribution function defined in \Cref{def:vol_dens} (up to a normalization factor). By multiplying the $d$ singular values, we obtain the square root of the pseudo-determinant, since $J_{f}$ has rank $d$, thereby avoiding degeneracy. The function $s$, as used in \Cref{expn:extn}, is designed to be close to 0 on the manifold and increases as $x$ moves away from the manifold, and its gradient will pull the generated points towards the manifold.
In practice, we use a Gaussian process to obtain $s$, as described in \Cref{appx:gp}.
The hyperparameter $\lambda > 0$ controls the balance between the volume distribution function and the manifold constraint.
In practice, we discretize this process using the Unadjusted Langevin Algorithm (ULA).
% where $f_{vol}(x)=\big|\prod_{i=1}^{d}\sigma_i(x)\big|$, with $\sigma_i(x), i=1,\dots, d$ being the singular values of the Jacobian matrix $J_f(x)$. The function $s$, as used in \Cref{expn:extn}, is designed to be close to 0 on the manifold and increases as $x$ moves away from the manifold, and its gradient will pull the generated points towards the manifold. The hyperparameter $\lambda > 0$ controls the balance between the volume distribution function and the manifold constraint, and $W_t$ represents Brownian motion. 
% In practice, we discretize this process using the Unadjusted Langevin Algorithm (ULA).

% \xin{if out of space, can move this algorithm to appendix, like the one for flow matching. Besides, this algorithm is probably well-known so I am not sure if we want it here.}

% \begin{proposition}
%     % \xin{Isometry leads to volume element preservation}
%     % \xin{mention Nash embedding theorem on existence of isometry}
%     For Riemannian manifolds $(\mathcal M,g_{\mathcal M}),(\mathcal N,g_{\mathcal N})$ and local isometry $f:\mathcal M\to\mathcal N$ (as described in \Cref{prop:dist_match}), we have 
% \end{proposition}

\begin{proposition}
\label{prop:volume_guidence}
    % \xin{the sampling converges to uniform volume distribution function}
    Suppose $f_{\mathrm{target}}(x)=\lambda s(x)-\log(f_{vol})(x)$ is $\alpha$-strongly convex for some constant $\alpha>0$, i.e. $\nabla^2 f(x)\succeq \alpha I$, then the distribution of $X$ in \Cref{eqn:ld_vol} converges exponentially fast in Wasserstein distance to a distribution supported on the data manifold, whose restriction on the manifold is proportional to the volume distribution function.
\end{proposition}
\Cref{prop:volume_guidence} demonstrates an exponential convergence rate of volume-guided generation in Wasserstein distance. Additionally, in \Cref{appx:vol_conv}, we present a proposition establishing exponential convergence rates in total variation distance.
% \begin{proposition}
% \label{prop:volume_guidence}
%     Suppose manifold $\mathcal M$ is compact with $\nabla\log(f_{vol}(x))$ Lipshitz continuous, and function $s(x)$ satisfies $s(x)=0,\forall x\in\mathcal M$, $\nabla s$ is Lipshitz continuous and $\exists K\geq 0,x^*\in\mathcal M,\langle x - x^*, \nabla s(x) \rangle \geq \mu \|x - x^*\|^2 - K, \quad \forall x \in \mathbb{R}^n$, then the dynamic defined in \Cref{eqn:ld_vol} converges exponentially fast in Wasserstein distance to a distribution supported on the data manifold, whose restriction on the manifold is proportional to the volume distribution function.
% \end{proposition}

% \begin{proof}
%     Proof directly follows \citep{dalalyan2017theoretical}.
% \end{proof}
% commented this to save space.

% \xin{proof directly follows from \citep{dalalyan2017theoretical}, but I will spend some time thinking about 1. if we can make the assumption more realistic using the manifold's properties, and 2. describe the function it converges to more explicitly. Otherwise, if $\alpha$-SLC is unrealistic, I also found \citep{xu2018global} which has some more complicated error bounds on nonconvex optimization.}

\begin{algorithm}[tb!]
\label{alg:ula}
\caption{Volume-Guided Generation}
\begin{algorithmic}
  \STATE {\bfseries Input:} $s(x), f_{vol}(x)$, initial sample $\mathbf{x}_0$, step size $\eta$, number of steps $N$, threshold $\epsilon$
  \STATE {\bfseries Output:} Filtered final sample $\mathbf{x}_{N,\text{filtered}}$
  \STATE Initialize $\mathbf{x} \gets \mathbf{x}_0$
  \FOR{$t = 1$ to $N$}
    \STATE Sample Gaussian noise $\boldsymbol{\xi}_t \sim \mathcal{N}(0, I)$
    \STATE $\mathbf{x}_{t+1} \gets \mathbf{x}_t - \eta\nabla(\lambda s(x) - \log(f_{vol}(x))) + \sqrt{2\eta} \cdot \boldsymbol{\xi}_t$
  \ENDFOR
  \STATE $\mathbf{x}_{N,\text{filtered}} \gets \{\mathbf{x} \in \mathbf{x}_N : s(\mathbf{x}) < \epsilon\}$
  % \STATE Filter the final sample: $\mathbf{x}_{N,\text{filtered}} \gets \{\mathbf{x} \in \mathbf{x}_N : s(\mathbf{x}) < \epsilon\}$
  \STATE \textbf{Return} $\mathbf{x}_{N,\text{filtered}}$
\end{algorithmic}
\end{algorithm}

\subsection{Generating along Manifold Geodesics}
\label{sec:interpolation}
\vspace{-6pt}

\begin{tcolorbox}[challenge]
Tackling Challenge 2: \textit{On-Manifold Interpolation}.
\end{tcolorbox}

We now turn to the problem of generating  the geodesic between a pair of points on the data manifold. This is useful when points in a manifold could represent the time evolution of a system, such as in single cell sequencing. It has been shown that such data usually follow the manifold hypothesis~\citep{mfd_review}, and that geodesic generation can model cellular trajectories such as those taken during differentiation. 

One could try to find the curve which minimizes length with respect to the metric $g_{\mathcal M}$.
However, this metric is only accurate on the manifold, and such shortest paths might cut through data space.
Indeed, we need to minimize length under the condition that the curve stays on the manifold.
The main result of this section shows that this constrained optimization problem is actually solved by minimizing arc length with respect to the warped metric $g_{\mathbb R^{n}}$.
Intuitively, this metric imposes large penalties for deviating from the manifold, as off-manifold points are embedded into the dimension-extended latent space, forcing the shortest path onto the manifold.

We begin with a neural-network parameterized interpolation curve.
% \xin{maybe just use $c_\eta(x_0,x_1,t)$ and put the actual parameterization and reasoning in the appendix?}
for any $x_0,x_1\in\mathcal M$, we define
a neural network-parameterized interpolation curve $c_\eta(x_0,x_1,\cdot):[0,1]\to\mathbb R^n$
% \xin{should replace all the $\forall,\exists$ to words, make sure of consistency}
% \begin{align}
%     c_\eta(x_0,x_1,t)
%     =
%     tx_1+(1-t)x_0+(1-(2t-1)^2) \gamma_\eta(x_0,x_1,t) ,
%     \label{expn:geob}
% \end{align}
satisfying $c_\eta(x_0,x_1,0)=x_0, c_\eta(x_0,x_1,1)=x_1$.
More details on parameterization are provided in \Cref{append:curve_param}. %\TR{broken ref}
We minimize\vspace*{-3pt}
% the energy functional (in Riemannian geometry) under $g_{\mathbb R^n}$
%\xin{For the ease of proof of the proposition, I use the energy functional instead of the curve length (do not take the square root.). This still yields a geodesic, but even better, it has uniform speed. We will need to update the setup in the experiments, to see empirically if this still gives good results, or if it is even better.}
\begin{align}
\label{eq:geo_curve_loss}
    \mathcal L_\mathrm{Geo}(\eta, x_0,x_1)
    =
    \frac{1}{M} \sum_{m=1}^{M} {g_{\mathbb R^{n}} (\dot c_\eta, \dot c_\eta)}(x_0,x_1,t_m)
\end{align}
where $0=t_0<t_1<...<t_M=1$ are sampled time points. 
Note that \Cref{eq:geo_curve_loss} is a discretization of the integral $\int_{0}^{1}{g_{\mathbb R^{n}} (\dot c_\eta, \dot c_\eta)}(x_0,x_1,t)dt$.
% \xin{energy is not a standard term, change to something else or make the following explanation more explicit}
In \citet{do1992riemannian}, this is defined as the energy of the curve, and minimizing the energy is equivalent to minimizing the curve length~\citep[Chapter 9, Proposition 2.5]{do1992riemannian}.
% Minimizing $\mathcal L_\mathrm{geo}(\eta)$, we obtain the geodesic between $x_0,x_1$.

The following proposition demonstrates that geodesic computation on $\mathcal M$ can be achieved by minimizing arc length with respect to the metric $g_{\mathbb R^{n}}$.
% \begin{proposition}\label{prop:geod}
%     Assume $f_\theta(\mathcal M)=f_\theta(\mathbb R^{n}).$
%     $\forall x_0,x_1\in\mathcal M$, $c\subset\mathbb R^n$ is the geodesic w.r.t. $g_{\mathbb R^n}$, we have $c\subset\mathcal M$.
%     Moreover, $c$ is the geodesic between $x_0$ and $x_1$ on $\mathcal M$ w.r.t. $g_\mathcal M$.
%     \xin{The assumption of surjection is too strong, but we can switch to epsilon error -- say, by smoothness of neural network, the image of f is close when it is "slightly off the manifold", and write everything with some epsilon bound.}
% \end{proposition}

\begin{lemma}
\label{prop:geod}
    % Assume that the $\omega$-thickening of $\mathcal{M} \subset \mathbb{R}^n$, $\mathcal{M}^{\omega} := \{ x \in \mathbb{R}^n : \inf_{{m \in \mathcal{M}}} d(x,m) < \omega\}$ (the set of points with distance to the manifold smaller than $\omega$, same below) maps into a subset of the $\epsilon$-thickening of $f(\mathcal{M})$, where $\epsilon$ can be chosen such that for every $x \in f(\mathcal{M})$, $B_{\epsilon} \cap f(\mathcal{M})$ has only one connected component. 
    % Here $B_\epsilon(x):=\{y\in\mathbb R^n: ||y-x||<\epsilon\}$ denotes the set of points with distance to $x$ smaller than $\epsilon$.
    % Then, for any smooth $c:[0,1]\to\mathbb R^n$, satisfying $c(0)=x_0,c(1)=x_1$ (a curve connecting these two points), there exists a smooth $c':[0,1]\to\mathcal M$ (a curve \textbf{on the manifold} connecting these two points), satisfying $c'(0)=x_0,c'(1)=x_1$, such that $\mathcal L_{\text{Geo}}(c')\leq\mathcal L_{\text{Geo}}(c)-\alpha^2\beta^2 \frac{1}{M}\sum_{m=1}^M (D_{\mathcal M}(c(t_m))-D_{\mathcal M}(c(t_{m-1})))^2+\xi$ (i.e. $c'$ is ) where $\alpha$ is in the assumption of \Cref{lem:bilip} and $\xi$ is a fixed positive constant independent on $x_t$ and $\beta$.
    % $\mathcal D_{\mathcal M}$ denotes the distance from a given point to the manifold, as defined in \Cref{lem:bilip}.
Assume that the $\omega$-thickening of the manifold $\mathcal{M} \subset \mathbb{R}^n$, defined as
$
\mathcal{M}^{\omega} := \{ x \in \mathbb{R}^n : \inf_{m \in \mathcal{M}} d(x, m) < \omega \},
$
(i.e., the set of points whose distance from $\mathcal{M}$ is less than $\omega$) maps into a subset of the $\epsilon$-thickening of $f(\mathcal{M})$. Here, the $\epsilon$-thickening is defined analogously, with $\epsilon$ chosen such that for every $x \in f(\mathcal{M})$, the ball
$
B_\epsilon(x) := \{ y \in \mathbb{R}^n : \|y - x\| < \epsilon \}
$
intersects $f(\mathcal{M})$ in exactly one connected component.

Then, for any smooth curve $c:[0,1] \to \mathbb{R}^n$ connecting $x_0$ and $x_1$ (i.e., $c(0) = x_0$ and $c(1) = x_1$), there exists a smooth curve $c':[0,1] \to \mathcal{M}$ lying entirely \textbf{on the manifold} (with $c'(0) = x_0$ and $c'(1) = x_1$) such that
$
\mathcal{L}_{\text{Geo}}(c') \leq \mathcal{L}_{\text{Geo}}(c) - \alpha^2 \beta^2 \frac{1}{M} \sum_{m=1}^M \bigl( \mathcal{D}_{\mathcal{M}}(c(t_m)) - \mathcal{D}_{\mathcal{M}}(c(t_{m-1})) \bigr)^2 + \xi.
$
Here, $\alpha$ is defined as in \Cref{lem:bilip}, $\mathcal{D}_{\mathcal{M}}$ denotes the distance from a point to $\mathcal{M}$ (also as in \Cref{lem:bilip}), and $\xi$ is a fixed positive constant independent of $x_t$ and $\beta$.
\end{lemma}

\scalebox{0.96}{
\begin{minipage}{1.03\linewidth}
\begin{proposition}
\label{prop:geod_minim}
When $\mathcal L_{\text{Geo}}$ is minimized, $\max\nolimits_{m=1,\dots,M} \mathcal D_{\mathcal M}(c(t_m))\le \sqrt{\xi}/(\alpha\beta)$, i.e., for sufficiently large $\beta$, $c(t)$ is close to the manifold with a maximum distance of $\sqrt{\xi}/(\alpha\beta)$.
Furthermore, let $c'(t)$ be a geodesic between $x_0$ and $x_1$ under the metric $g_{\mathcal M}$, we have $\smash{\frac{1}{M}} \smash{\sum_{m=1}^{M} {g_{\mathcal M} (\dot c, \dot c)}(x_0,x_1,t_m)} \leq \smash{\frac{1}{M}} \sum_{m=1}^{M} {g_{\mathcal M} (\dot c', \dot c')}(x_0,x_1,t_m)+\xi'\sqrt{\xi}/(\alpha\beta)$ for some positive constant $\xi'$. That is, $c$ approximately minimizes the energy (and hence curve length) under $g_{\mathcal M}.$
\end{proposition}    
\end{minipage}
}

\Cref{prop:geod} shows that for any curve connecting two points on the manifold, there exists an alternative curve where the loss difference is controlled by the difference in their distances from the manifold.
\Cref{prop:geod_minim} then uses this result to establish that a necessary condition for minimizing the loss is that the curve remains sufficiently close to the manifold. Moreover, it demonstrates that the minimizer’s energy is nearly equal to that of the true geodesic. Thus, the minimizer is (approximately) 1) on the manifold and 2) of minimal length, and is therefore the geodesic.

In summary, minimizing \Cref{eq:geo_curve_loss} yields the geodesic on $\mathcal{M}$ between $x_0$ and $x_1$ with respect to the pullback metric $g_{\mathcal{M}}$. This is achieved by minimizing the curve length under the warped pullback metric.

% \begin{algorithm}[htbp]
%   \caption{Mini-batch OT Geodesic Flow Matching}
%   \label{alg:fm}
% \begin{algorithmic}
% \STATE {\bfseries Input:} Starting and ending populations $\mathcal X,\mathcal Y$, encoder $f$, dimension-extended encoder $r$, $t=(t_1,...,t_M)$
% \WHILE{Training}
% \COMMENT{Sample batches of size $b$ \textit{i.i.d.} from the datasets}
% \item Sample $\{x_1,...,x_l\}\subset \mathcal X,\{y_1,...,y_l\}\subset \mathcal Y$
% \item $\mu\gets \frac{1}{l}\sum_{i=1}^l I (x=x_i),\nu\gets \frac{1}{l}\sum_{i=1}^l I (x=y_i)$
% \item $\pi^*=\underset{\pi\sim{\Gamma(\mu,\nu)}}{\arg\min}(\frac{1}{l}\sum_{i=1}^l\pi(x'_i,y'_i)||f(x'_i)-f(y'_i)||^2)^{1/2}$
% \item Sample $(x_{j_1}, y_{j_1}),...,(x_{j_l}, y_{j_l}) \overset{i.i.d.}{\sim} \pi^*$
% % \item Curve $c_\eta\gets c_\eta(x_{j_i}, y_{j_i}, t)$
% % \item Flow $v_\nu \gets v_\nu(x_{j_i},t)$
% \COMMENT{Compute geodesic and velocity-matching losses}
% \item $L\gets \frac{1}{l}\sum_{i=1}^l(\lambda_3\mathcal L_{\text{geo}}(\eta,x_{j_i},y_{j_i})+\lambda_4 \mathcal L_{\text{FM}}(\nu,\eta,x_{j_i},y_{j_i}))$ 
% % \xin{update loss names, update with params}
% \item $\eta,\nu \gets \mathrm{GradientDescentUpdate}(\eta,\nu , \nabla L)$
% \ENDWHILE
% \STATE \Return $\nu$
% \end{algorithmic}
% \end{algorithm}

\begin{algorithm}[!t]

  % \caption{Mini-batch OT Geodesic Flow Matching}
  \caption{Geodesic-Guided Flow Matching}\label{alg:flow_matching}
  \scalebox{0.89}{
    \begin{minipage}[t]{1.1\textwidth}
    \begin{algorithmic}
    \STATE {\bfseries Input:} Starting and ending populations $\mathcal X, \mathcal Y$, encoder $f$, \\dimension-extended encoder $r$, $t=(t_1,...,t_M)$
    \WHILE{Training}
        \STATE \textbf{Sample batches of size $b$ \textit{i.i.d.} from the datasets}
        \STATE Sample $\{x_1,...,x_l\}\subset \mathcal X, \{y_1,...,y_l\}\subset \mathcal Y$
        \STATE $\mu \gets \frac{1}{l}\sum_{i=1}^l I(x=x_i), \nu \gets \frac{1}{l}\sum_{i=1}^l I(x=y_i)$
        \STATE $\pi^* = \underset{\pi \sim \Gamma(\mu, \nu)}{\arg\min} \left(\frac{1}{l}\sum_{i=1}^l \pi(x'_i, y'_i) ||f(x'_i) - f(y'_i)||^2 \right)^{1/2}$
        \STATE Sample $(x_{j_1}, y_{j_1}), ..., (x_{j_l}, y_{j_l}) \overset{i.i.d.}{\sim} \pi^*$
        % \STATE Curve $c_\eta\gets c_\eta(x_{j_i}, y_{j_i}, t)$
        % \STATE Flow $v_\nu \gets v_\nu(x_{j_i},t)$
        \STATE \textbf{Compute geodesic and velocity-matching losses}
        \STATE $L \gets \frac{1}{l}\sum_{i=1}^l (\lambda_3 \mathcal L_{\text{geo}}(\eta, x_{j_i}, y_{j_i}) + \lambda_4 \mathcal L_{\text{FM}}(\nu, \eta, x_{j_i}, y_{j_i}))$
        % \xin{update loss names, update with params}
        \STATE $\eta, \nu \gets \mathrm{GradientDescentUpdate}(\eta, \nu, \nabla L)$
    \ENDWHILE
    % \STATE \Return $\nu$
    \STATE \textbf{Output:} $\nu$
    \end{algorithmic}
    \end{minipage}
}
\end{algorithm}

% \newpage

\subsection{Population Interpolation along geodesics}
\label{sec:population_transport}
\vspace{-6pt}

\begin{tcolorbox}[challenge]
Tackling Challenge 3: \textit{Population Transport}.
\end{tcolorbox}

In the previous section, we achieved point-wise geodesic computation, learning the geodesic between a pair of points.
More generally, we aim to generate population-level geodesics.
Given two distributions on the manifold, we want to generate geodesics between populations sampled from these distributions, minimizing the expected total length of the geodesics.
This equates to solving the dynamical optimal transport problem \citep{TrajectoryNet, benamou2000computational}, where the cost is the curve length on the manifold.

To solve this, we first find the optimal pairing of points from the starting and ending distributions to minimize total geodesic length, then compute those geodesics.
To generalize to new points, we learn a vector field matching the time derivatives (speed) of the geodesics.
Given a point sampled from the first distribution, we can generate the geodesic by integrating the vector field starting from the point.

% We want to learn the geodesics between points of the optimal coupling, such that the total length of the geodesics is minimized. Moreover, given a new point sampled from the first distribution, we want to generate the optimal geodesic from this point to the population following the second distribution.

% Finally, we use a flow model to learn a vector field corresponding to the optimal geodesic trajectories between two populations $\mathcal X,\mathcal Y$.
Specifically, we define a neural network
$
% \begin{align}
    v_\nu(x_0,t)\in\mathbb R^{n},
% \end{align}
$
and the flow matching loss for any joint distribution $\pi$ and curve $c_\eta$ as the following.\vspace{-3pt}
\begin{align}
\label{eq:flow_loss}
\begin{split}
    &\mathcal{L}_\mathrm{FM}(\nu,\eta, x_0,x_1)
    \\&
    =\mathbb{E}_{\pi(x_0, x_1)}||v_\nu(t,x_0)-\frac{d}{d t}c_\eta(t, x_0,x_1)||^2    
\end{split}
\end{align}
\vspace{-12pt}

When this loss is minimized, $v_\nu$ is the vector field that matches the time derivatives of the curves. 

% \xin{condense the equations (also check other ones}
In each training step, we sample starting and ending points from the two distributions, and solve the optimal transport problem where the ground distance is the Euclidean distance in the latent space. 
This optimal transport plan $\pi$ would minimize the total geodesic length between $(x_0,x_1)\sim\pi$, because \methodshort is trained so that the Euclidean distance in the latent space is matched to the geodesic distance on the data manifold.
We then parameterize the interpolation curves $c_{\eta}$ as in \Cref{sec:interpolation}, and minimize the following loss which balances the loss \Cref{eq:geo_curve_loss} that the $c_{\eta}$ are the geodesics, and the aforementioned flow matching loss (\Cref{eq:flow_loss}).\vspace*{-3pt}
% and train the geodesics between starting and ending points provided by the optimal transport based on the distances in the latent space, which is matched to the geodesic distances on the manifold, and minimize the loss
\begin{align}
\begin{split}
    &
    \mathcal L_{\mathrm{GFM}}(\nu,\eta,x_0,x_1)
    \\
    &
    =
    \lambda_3 \mathcal L_\mathrm{geo}(\eta,x_0,x_1)+\lambda_4 \mathcal{L}_\mathrm{FM}(\nu,\eta,x_0,x_1)
\end{split}
\label{eqn:loss_flow_matching}
\end{align}
\vspace{-12pt}

Further details are provided in \Cref{alg:flow_matching}.

After training, we generate the geodesics by integrating the vector field $v_\eta$.
Given an initial point $x_0$, we can generate points along the geodesics starting from it with $x(t)=x_0 + \int_{0}^t v_\nu (x_0,\tau)d\tau$.

% \vspace{-12pt}
% \begin{align}
%     x(t)=x_0 + \int_{0}^t v_\nu (x_0,\tau)d\tau
% \label{expn:optimal_geods}
% \end{align}

Finally, the following proposition shows that our method generates desired population-level geodesics.
\begin{proposition}
\label{prop:flow_matching}
    Given starting and ending distributions $p,q$,
    at the convergence of \Cref{alg:flow_matching}, 
    \begin{align*}
        x(t)=x_0 + \int_{0}^t v_\nu (x_0,\tau)d\tau, x_0\sim p,t\in [0,1]
    \end{align*}
    are geodesics between the two distributions following the optimal transport plan that minimizes the total expected geodesic lengths. 
\end{proposition}

\vspace*{-2pt}
\section{EMPIRICAL RESULTS}
\vspace*{-4pt}
% {\bf Geometry-aware Autoencoder.}~
\paragraph{Geometry-aware Autoencoder}

First, we empirically show that \methodshort preserves manifold distances of data in the latent space by evaluating \methodshort on Splatter~\citep{zappia2017splatter}, a synthetic single-cell RNA sequence dataset.
 % which simulates cell populations with diverse cell types, structures, and differentiation patterns. 
 
Single-cell RNA sequence data are high-dimensional, noisy, and sparse and have been demonstrated to reside on low-dimensional manifolds, making them ideal datasets for evaluating our method \citep{heimberg2016low, mfd_review}.

% and compare it against a standard autoencoder without distance matching objective on DEMaP \citep{PHATE} and DRS.
% % % Further details are provided in \Cref{appdx:experiment_details_ae}.
% which uses parametric models to simulate cell
% populations with multiple cell types, structures, and differentiation patterns. 
% Specifically, we evaluate our method on single-cell data of group and path structures with fixed 0.5 dropout probability and with multiple biological coefficient of variation (bcv) parameters. A higher bcv corresponds to a lower signal-to-noise ratio. 

The encoder was evaluated by Denoised Embedding Manifold Preservation~(DEMaP) described in \citep{PHATE} which measures the correlation between Euclidean distances in latent space and ground truth manifold distances in data space. 
% It is well
% The decoder is evaluated by reconstruction quality. Due to the high noise in the genes, we propose a new criterion called Denoised Reconstruction Score~(DRS) which computes the correlation between reconstructed genes and genes denoised by MAGIC~\citep{van2018recovering}. 
% All experiments are repeated with 5 random seeds. 
% See \Cref{appdx:experiment_details_ae} for details on the Splatter dataset and our evaluation criteria.

The results show that our distance matching loss is important for preserving manifold distances, as evidenced by higher DEMaP scores averaged across different noise levels~(\Cref{tab:demap_drs}). 

\methodshort can also effectively reconstruct high-dimensional features through the decoder. See \Cref{appdx:experiment_details_ae} for results on reconstruction, details on the Splatter dataset, and our evaluation criteria.

\begin{table}[tb!]
    \setlength{\tabcolsep}{5pt}
    \vspace*{-5pt}
    \centering
    \caption{Average DEMaP for Autoencoders (AE) and GAGA on simulated single-cell datasets over different noise settings.}
    % \scalebox{0.82}{
    \begin{tabular}[b]{cccc}
      \toprule
      & \multirow{2}{*}{Objective} & Cellular & \multirow{2}{*}{DEMaP ($\uparrow$)} \\
      & & State Space & \\
      \midrule
      AE & $\mathcal{L}_\text{Recon}$ & Clusters & 0.347{\scriptsize \textcolor{gray}{$\pm$0.117}} \\
      \methodshort & $\mathcal{L}_\text{Recon}, \mathcal{L}_\text{Dist}$ & Clusters & \textbf{0.645}{\scriptsize \textcolor{gray}{$\pm$0.195}} \\
      \midrule
      AE & $\mathcal{L}_\text{Recon}$ & Trajectories & 0.433{\scriptsize \textcolor{gray}{$\pm$0.135}} \\
      \methodshort & $\mathcal{L}_\text{Recon}, \mathcal{L}_\text{Dist}$ & Trajectories & \textbf{0.600}{\scriptsize \textcolor{gray}{$\pm$0.191}}\\ 
      \bottomrule
    \end{tabular}
    % }
    \vspace{-8pt}
    \label{tab:demap_drs}
\end{table}

\paragraph{Volume-guided Generation on Manifold}
% {\bf Volume-guided Uniform Generation on Manifold} 
% \xin{TODO: Toy data \& real data}
We assessed the effectiveness of volume-guided generation on both simulated and real data.
% \xin{1. the toy manifold has known vol elem, 2. we compute KDE of the generated data and compare 3. add table}

% We first illustrated our method on three toy datasets: hemisphere saddle, and paraboloid, where the analytical forms of the ground truth volumes are known. To simulate imbalanced sampling, we generated points following Gaussian distributions in the parameter space and then mapped them to the data space. We used these points as training data and applied \methodshort to generate new points following the geometry of the manifold.

We first illustrated our method on three toy datasets: hemisphere saddle, and paraboloid.
On these manifolds, the volume element is known, which we used as ground truth.
We evaluated the generation by computing the kernel density estimation and comparing it with the ground truth (see \Cref{appx:vol_eval} for details).

We generate imbalanced data by sampling from Gaussian distribution in the parameter space. In \Cref{fig:toy_unif} (B,C,D) and \Cref{tab:toy_unif} we show that the densities of the points generated by \methodshort are closer to the ground truth volume elements compared to the original data points, indicating that \methodshort largely reduces data imbalance.
In addition, \Cref{fig:toy_unif} (A) shows that the generated points stay on the data manifold and cover the sparse regions well in the original data.
The complete result figure can be found in \Cref{appdx:results_unif_gen_full}.
We also compared our method with Riemannian Flow Matching \cite{chen2023flow} in \Cref{appx:RFM} to demonstrate the faithfulness of generated points to the data geometry.

\begin{figure}[b!]
    \vspace{-10pt}
    \centering
    \includegraphics[width=\columnwidth]{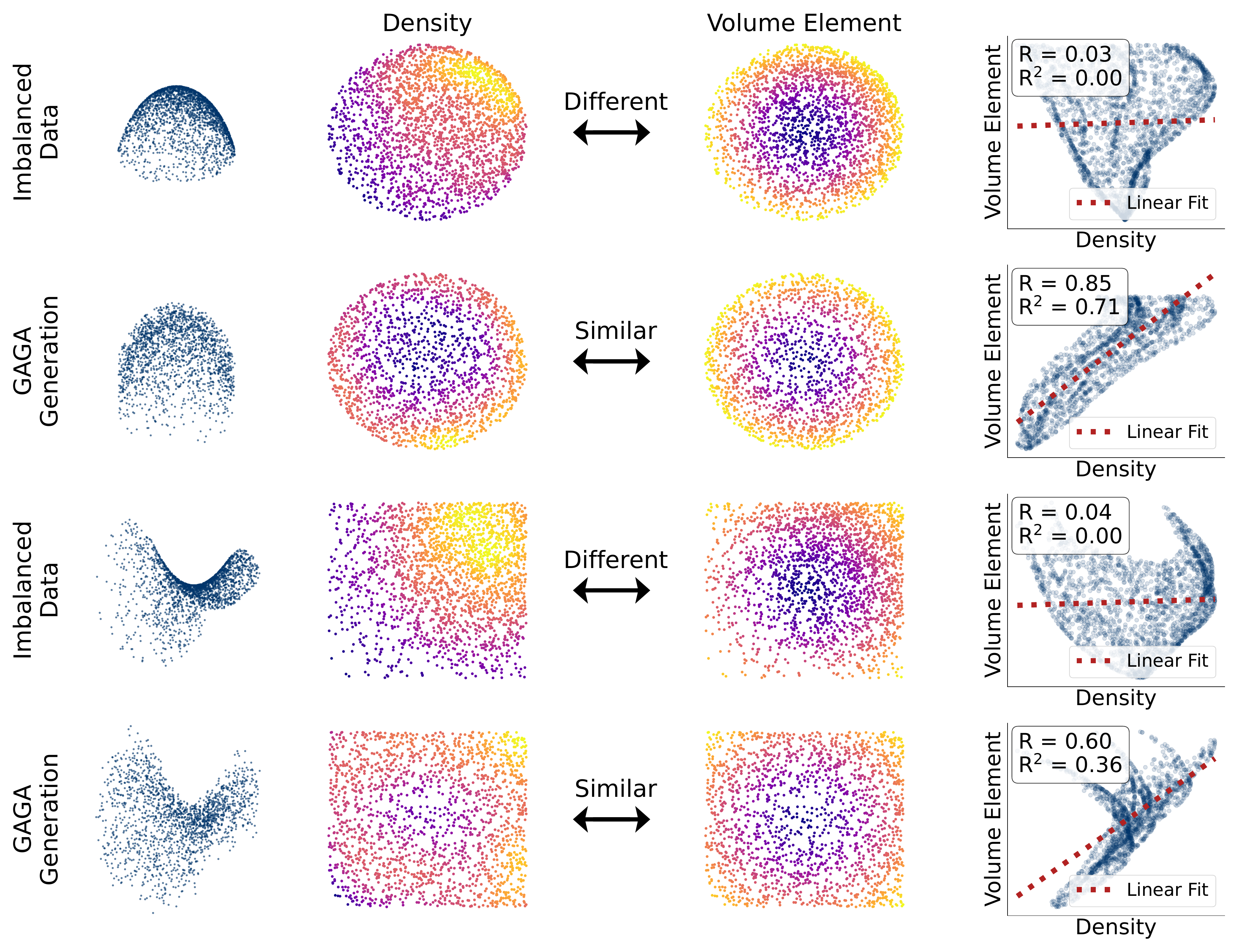}
    \vspace{-18pt}
    \caption{Geometry-aware generation with \methodshort on hemisphere and saddle. \textbf{(A)} Generated points remain on the manifold, and are more evenly distributed compared to raw data. \textbf{(B)} Kernel density estimation. \textbf{(C)} Ground truth volume elements computed analytically. \textbf{(D)} In raw data, density does not correlate to volume element, indicating data imbalance. \methodshort generation corrects the imbalance indicated by higher correlation between volume element and density.}
    \label{fig:toy_unif}
    \vspace{-10pt}
\end{figure}

\begin{table}[bt!]
\setlength{\tabcolsep}{8.1pt}
\centering
\caption{Pearson correlation and $R^2$ score between data density and ground truth volume element. \methodshort greatly reduces data imbalance.}\label{tab:toy_unif}
% \scalebox{0.88}{
\begin{tabular}{llcc}
    \toprule
    {Manifold}  & {Data}          & {R}  & {R$^2$} \\
    \midrule
    \multirow{2}{*}{Hemisphere} & Original Data & 0.03 & 0.00  \\
                                & GAGA Generation & \textbf{0.85} & \textbf{0.71}  \\
    \midrule
    \multirow{2}{*}{Saddle}     & Original Data & 0.04 & 0.00  \\
                                & GAGA Generation & \textbf{0.60} & \textbf{0.36}  \\
    \midrule
    \multirow{2}{*}{Paraboloid} & Original Data & 0.04 & 0.00  \\
                                & GAGA Generation & \textbf{0.66} & \textbf{0.44}  \\
    \bottomrule
\end{tabular}
% }
\end{table}

\begin{figure}[bt!]
    \centering
    \includegraphics[width=\columnwidth]{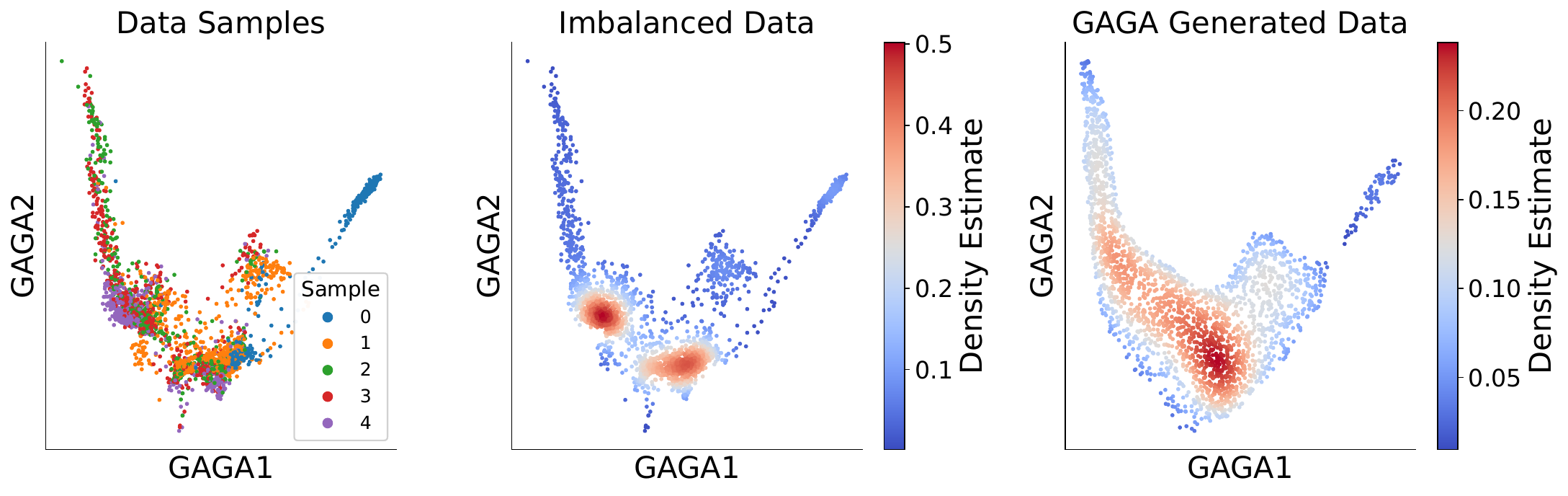}
    \vspace{-14pt}
    \caption{Geometry-aware generation with \methodshort on Embryoid Body data. Left: The dataset includes measurements from five experiments. Middle: The data is sparse and imbalanced. Colors indicate density estimation. Right: \methodshort reduces sampling imbalance.}
    \label{fig:eb_unif}
 \vspace{-8pt}
 \end{figure}

Next, we applied volume-guided generation to the Embryoid Body dataset~\citep{PHATE}, a real-world single-cell dataset that captures cellular evolution over the course of 27 days~(\Cref{fig:eb_unif} left). The data is largely imbalanced, with two density peaks, as shown in \Cref{fig:eb_unif} middle panel. Due to sampling bias, the data points in sample 4 exhibit a very high density, as significantly more data points were measured from this sample.
Moreover, there are sparse areas and ``holes'' in the data manifold.

% Next, we applied volume-guided generation to the Embryoid Body dataset~\citep{PHATE}, a real-world single-cell dataset that captures cellular evolution over the course of 27 days~(\Cref{fig:eb_unif}). The data is largely imbalanced, with two density peaks, as shown in \Cref{fig:eb_unif} left panel. Due to sampling bias, the data points in sample 4 exhibit a very high density, as significantly more data points were measured from this sample.
% Moreover, there are sparse areas and ``holes'' in the data manifold.

After volume-guided generation with \methodshort, the data imbalance is significantly mitigated. Without deviating from the manifold, the density peaks are less spiky and the ``holes'' are properly filled in the \methodshort-generated data~(\Cref{fig:eb_unif} right panel) compared to the original Embryoid Body data~(\Cref{fig:eb_unif} middle panel).

\paragraph{Generating along Geodesics on Manifold}
% {\bf Generating along Geodesics on Manifold.}~
To evaluate \methodshort's performance on generating geodesics on data manifold, we started with four toy manifolds: ellipsoid, torus, saddle, and hemisphere in $\mathbb{R}^3$. To make these datasets more challenging, we added Gaussian noise of different scales to the original data and rotate them to higher dimensions using a random rotation matrix. The ground truth geodesic lengths were obtained analytically if the solution is available or by using Dijkstra’s algorithm on the noiseless data otherwise. See \Cref{appdx:exp_details_geodesics} for details.
% Similarly, we used synthetic and real data for evaluation.

On the synthetic dataset, we compared our method with Dijkstra's algorithm, and a baseline that directly uses the metric without warping.
% and a method that uses the density loss in  \citet{TrajectoryNet}.
More baseline comparisons and details are provided in \Cref{appdx:exp_details_geodesics}.

% \begin{table}[tb!]
%     \vspace*{-5pt}
%     \setlength{\tabcolsep}{2.5pt}
%     \centering
%     \caption{Average MSE between predicted geodesic lengths and ground truth on data with different dimensions and noise settings.}
%     \scalebox{0.92}{
%     \begin{tabular}[b]{lcccc}
%       \toprule
%       & Djikstra's & \methodshort & Local & Density \\ 
%       \midrule
%       Ellipsoid & \underline{4.40}{\scriptsize \textcolor{gray}{$\pm$6.6}} & \textbf{3.76}{\scriptsize \textcolor{gray}{$\pm$7.1}} & 143.70{\scriptsize \textcolor{gray}{$\pm$246.5}} & 5.36{\scriptsize \textcolor{gray}{$\pm$9.0\phantom{0}}} \\
%       Hemisphere & 4.83{\scriptsize \textcolor{gray}{$\pm$6.2}} & \textbf{0.47}{\scriptsize \textcolor{gray}{$\pm$0.6}} & \phantom{0}43.20{\scriptsize \textcolor{gray}{$\pm$65.7\phantom{0}}} & \underline{0.60}{\scriptsize \textcolor{gray}{$\pm$0.6\phantom{0}}} \\
%       Saddle & \textbf{1.87}{\scriptsize \textcolor{gray}{$\pm$3.5}} & \underline{4.11}{\scriptsize \textcolor{gray}{$\pm$8.8}} & \phantom{0}55.59{\scriptsize \textcolor{gray}{$\pm$76.8\phantom{0}}} & 5.30{\scriptsize \textcolor{gray}{$\pm$12.4}} \\
%       Torus & \underline{5.01}{\scriptsize \textcolor{gray}{$\pm$7.9}} & \textbf{4.09}{\scriptsize \textcolor{gray}{$\pm$6.3}} & 271.84{\scriptsize \textcolor{gray}{$\pm$295.3}} & 5.39{\scriptsize \textcolor{gray}{$\pm$9.5\phantom{0}}} \\
%       \bottomrule
%     \end{tabular}
%     }
%     \vspace{-12pt}
%     \label{tab:geodesic_mse}
% \end{table}

\begin{table}[tb!]
    \vspace*{-5pt}
    \setlength{\tabcolsep}{6.5pt}
    \centering
    \caption{Average MSE between predicted geodesic lengths and ground truth on simulated data with different dimensions and noise settings.}
    % \scalebox{0.92}{
    \begin{tabular}[b]{lcccc}
      \toprule
      Manifold & Djikstra's & No Warping & \methodshort \\ 
      \midrule
      Ellipsoid & \underline{4.40}{\scriptsize \textcolor{gray}{$\pm$6.6}} & 143.70{\scriptsize \textcolor{gray}{$\pm$246.5}} & \textbf{3.76}{\scriptsize \textcolor{gray}{$\pm$7.1}} \\
      Hemisphere & 4.83{\scriptsize \textcolor{gray}{$\pm$6.2}} & \phantom{0}43.20{\scriptsize \textcolor{gray}{$\pm$65.7\phantom{0}}} & \textbf{0.47}{\scriptsize \textcolor{gray}{$\pm$0.6}} \\
      Saddle & \textbf{1.87}{\scriptsize \textcolor{gray}{$\pm$3.5}} & \phantom{0}55.59{\scriptsize \textcolor{gray}{$\pm$76.8\phantom{0}}} & \underline{4.11}{\scriptsize \textcolor{gray}{$\pm$8.8}} \\
      Torus & \underline{5.01}{\scriptsize \textcolor{gray}{$\pm$7.9}} & 271.84{\scriptsize \textcolor{gray}{$\pm$295.3}} & \textbf{4.09}{\scriptsize \textcolor{gray}{$\pm$6.3}} \\
      \bottomrule
    \end{tabular}
    % }
    \vspace{-8pt}
    \label{tab:geodesic_mse}
\end{table}

\begin{table*}[tb!]
\setlength{\tabcolsep}{12.2pt}
\small
\centering
\caption{Single-cell trajectory inference results on Cite and Multi datasets with 50 and 100 PCA dimensions. Leave-one-out is performed and 1-Wasserstein distances between prediction and ground truth are reported. }
\label{tab:cite_multi}
% \scalebox{0.84}{
\begin{tabular}{lccccc}
\toprule
Data Dimension & &\multicolumn{2}{c}{50} & \multicolumn{2}{c}{100}\\
\cmidrule(lr){3-4} 
\cmidrule(lr){5-6} 
% \cmidrule(lr){7-8}
Alg.$\downarrow$ \quad Dataset$\rightarrow$ & & Cite & Multi & Cite & Multi \\
\midrule
\text{DSBM}~\citep{DSBM} & & 53.81\textcolor{gray}{$\pm$7.74} & 66.43\textcolor{gray}{$\pm$14.39} & 58.99\textcolor{gray}{$\pm$7.62} & 70.75\textcolor{gray}{$\pm$14.03} \\
\text{I-CFM}~\citep{Conditional_flow_matching} & & 41.83\textcolor{gray}{$\pm$3.28} & 49.78\textcolor{gray}{$\pm$4.43} & 48.28\textcolor{gray}{$\pm$3.28} & 57.26\textcolor{gray}{$\pm$3.86} \\
\text{OT-CFM}~\citep{Conditional_flow_matching} & & 38.76\textcolor{gray}{$\pm$0.40} & 47.58\textcolor{gray}{$\pm$6.62} & 45.39\textcolor{gray}{$\pm$0.42} & 54.81\textcolor{gray}{$\pm$5.86} \\
\text{[SF]$^2$M-Exact}~\citep{SFSFM} & & 40.01\textcolor{gray}{$\pm$0.78} & 45.34\textcolor{gray}{$\pm$2.83} & 46.53\textcolor{gray}{$\pm$0.43} & 52.89\textcolor{gray}{$\pm$1.99} \\
\text{[SF]$^2$M-Geo}~\citep{SFSFM} & & 38.52\textcolor{gray}{$\pm$0.29} & 44.80\textcolor{gray}{$\pm$1.91} & 44.50\textcolor{gray}{$\pm$0.42} & 52.20\textcolor{gray}{$\pm$1.96}\\
\text{WLF-SB}~\citep{WLF} & &
39.24\textcolor{gray}{$\pm$0.07} &
47.79\textcolor{gray}{$\pm$0.11} &
46.18\textcolor{gray}{$\pm$0.08} &
55.72\textcolor{gray}{$\pm$0.06} \\
\text{WLF-OT}~\citep{WLF} & &
36.17\textcolor{gray}{$\pm$0.03} & 38.74\textcolor{gray}{$\pm$0.06} &
42.86\textcolor{gray}{$\pm$0.04} &
47.37\textcolor{gray}{$\pm$0.05} \\
\text{WLF-UOT}~\citep{WLF} & &
34.16\textcolor{gray}{$\pm$0.04} &
36.13\textcolor{gray}{$\pm$0.02} &
41.08\textcolor{gray}{$\pm$0.04} &
45.23\textcolor{gray}{$\pm$0.01} \\
\text{OT-MFM}~\citep{MFM} & & 36.39\textcolor{gray}{$\pm$1.87} & 45.16\textcolor{gray}{$\pm$4.96} &
41.78\textcolor{gray}{$\pm$1.02} & 50.91\textcolor{gray}{$\pm$4.623} \\
\midrule
\textbf{\text{\methodshort} (Ours)} & &
\textbf{23.29}\textcolor{gray}{$\pm$0.83} &
\textbf{19.68}\textcolor{gray}{$\pm$1.93} &
\textbf{26.72}\textcolor{gray}{$\pm$0.99} & 
\textbf{27.04}\textcolor{gray}{$\pm$2.95} \\
Improvement over SOTA & &
\textcolor{ForestGreen}{$\downarrow$ 31.8\%} &
\textcolor{ForestGreen}{$\downarrow$ 45.5\%} &
\textcolor{ForestGreen}{$\downarrow$ 34.6\%} &
\textcolor{ForestGreen}{$\downarrow$ 40.2\%} \\
\bottomrule
\end{tabular}
% }
\end{table*}

% We generate datasets under $ \{0, 0.1, 0.3, 0.5\}$  noise scales and $\{3, 5, 10, 15\}$ dimensions. For each dataset, we randomly select 20 pairs of starting and ending points on the manifold. 

% To obtain approximate ground truth geodesics and geodesics lengths, we first compute the geodesics on the noiseless data using the analytical expressions (if available) or Dijkstra's algorithm. 

% Quantitatively, we evaluate these methods on the MSE criteria: the mean squared error between the predicted geodesic length and ground truth length, which is obtained analytically if the solution is available or through Dijkstra’s algorithm on noiseless data.
% \begin{align}
% \SwapAboveDisplaySkip
%     \text{Length MSE}=&\frac{1}{k}\sum_{i=1}^{k}(\hat l_{i} - l_i)^2,
% \end{align}
% where $k$ is the total number of geodesics, 
% % $m$ is the number of points in each geodesic the model predicts, $g_i$ is the $i$-th ground-truth geodesic, $\hat g_i=(g_{i}^{(1)},\dots,g_{i}^{(m)})$ is the $i$-th predicted geodesic,
% $l_i,\hat l_i$ are the lengths of the $i$-th ground truth and predicted geodesics.
% We obtain the ground truth geodesics analytically if the solution is available. Otherwise, we use Dijkstra’s algorithm on noiseless data.

% We benchmark all methods on the noisy, high-dimensional data, and compute the pairwise geodesics. 

\begin{figure}[tb!]
    \centering
    \vspace{-4pt}
    %%%% 2nd Row %%%%%
    \begin{minipage}[c]{0.33\columnwidth}
    \centering
    \includegraphics[trim={0 50pt 0 200pt},clip,width=\columnwidth]{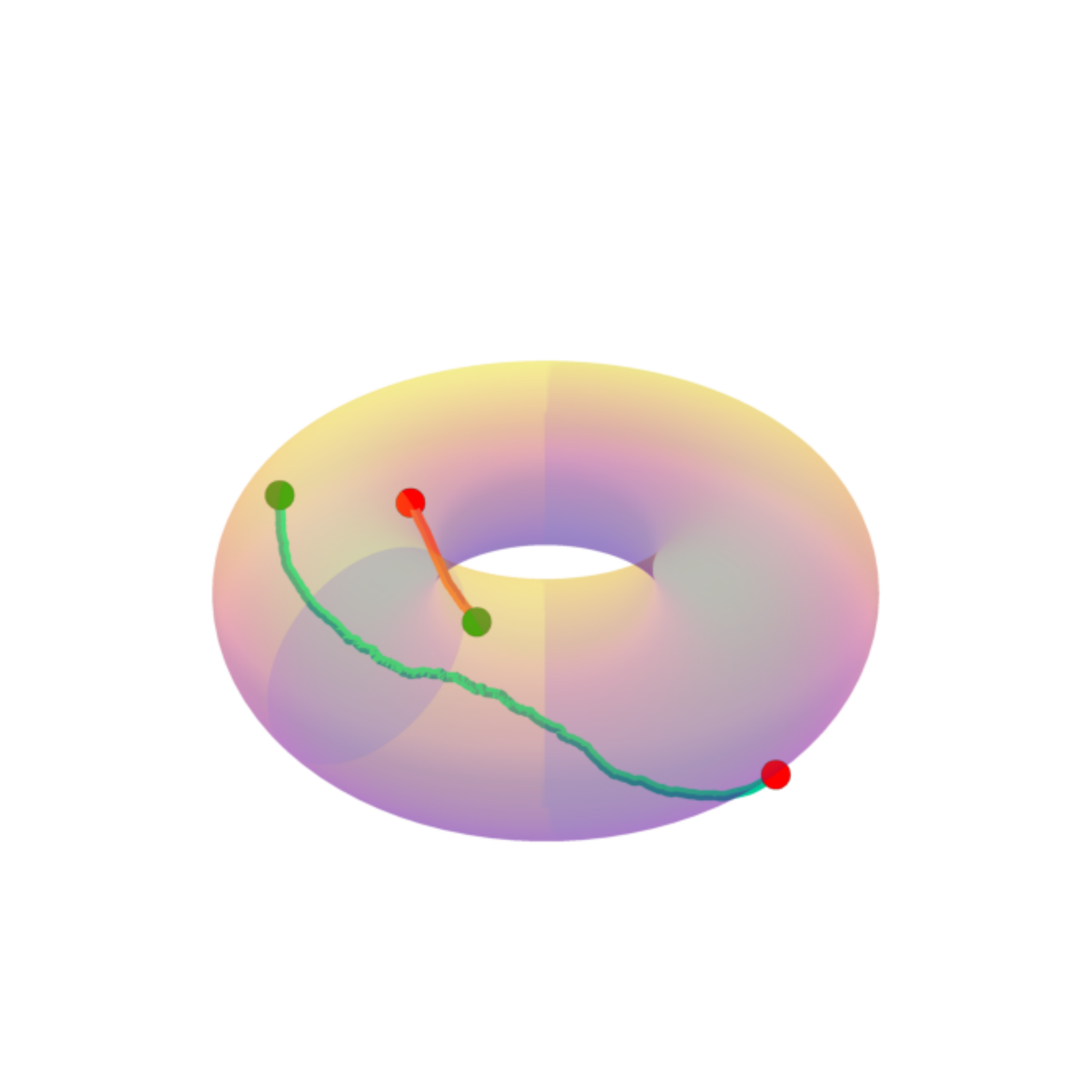}
    \end{minipage}
    % \hfill
     \hspace*{-17pt}
    \begin{minipage}[c]{0.33\columnwidth}
    \centering
    \includegraphics[trim={0 50pt 0 200pt},clip,width=\columnwidth]{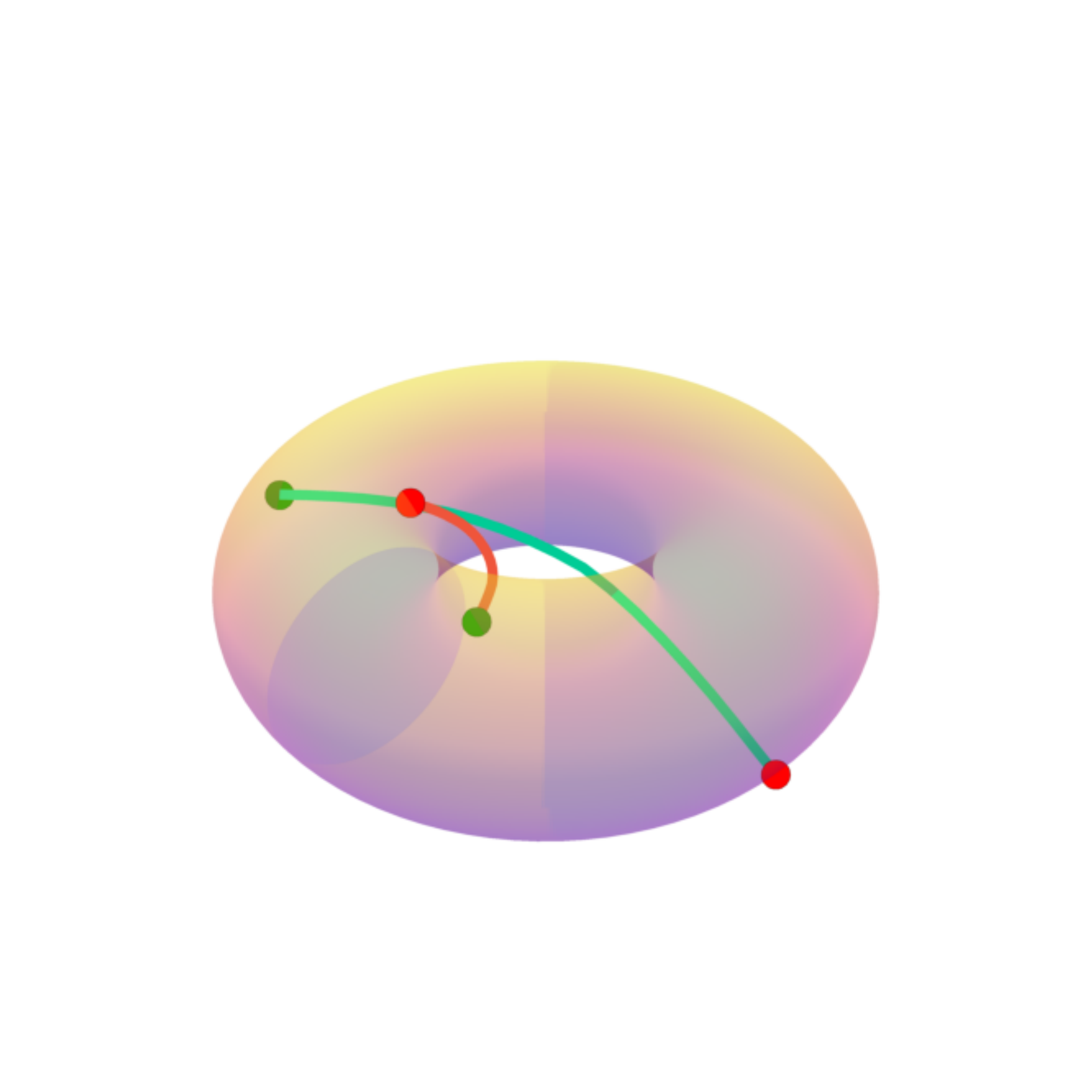}
    \end{minipage}
     % \hfill
     \hspace*{-17pt}
    \begin{minipage}[c]{0.33\columnwidth}
    \centering
    \includegraphics[trim={0 50pt 0 200pt},clip,width=\columnwidth]{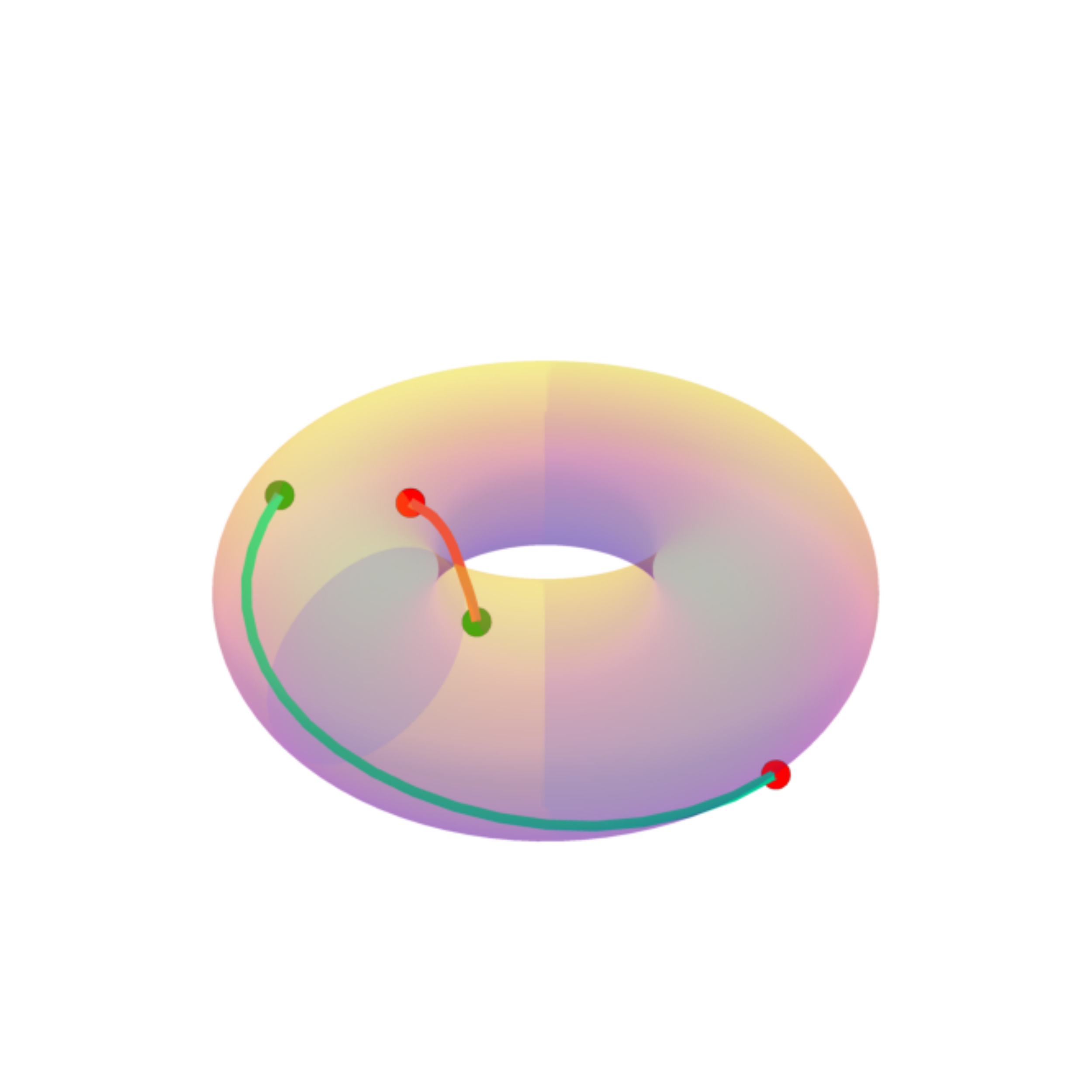}
    \end{minipage}
     %\hfill
    \hspace{-17pt}
    % \begin{minipage}[c]{0.28\columnwidth}
    % \centering
    % \includegraphics[width=\columnwidth]{fig/geovis_density_torus.pdf}
    % \includegraphics[width=\columnwidth]{fig/geovis_density_torus.pdf}
    % \end{minipage}
    \\[-16pt]
    %%%%% 3rd Row %%%%%
    \begin{minipage}[c]{0.33\columnwidth}
    \centering
    \includegraphics[trim={150pt 50pt 0 200pt},clip,width=\columnwidth]{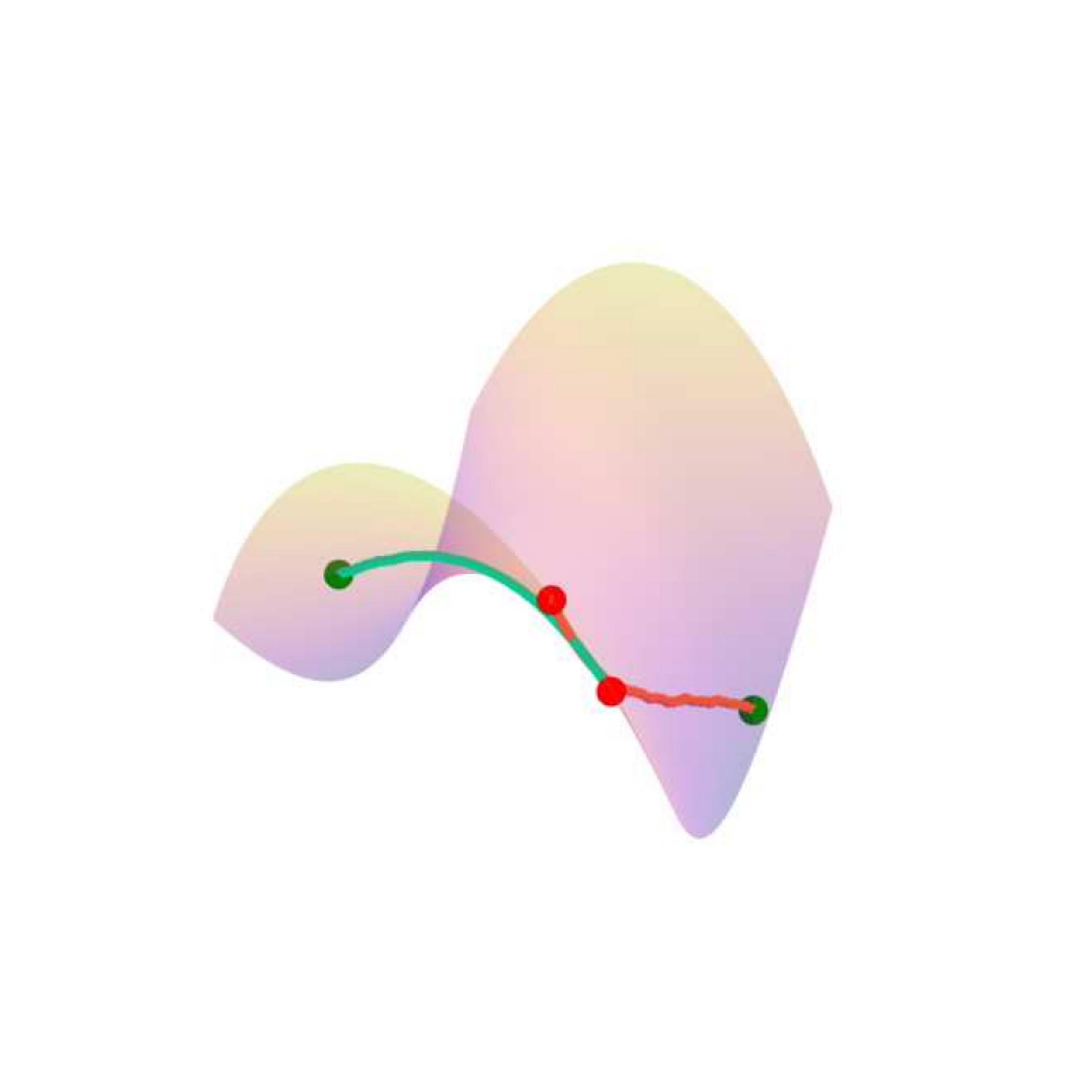}
    \end{minipage}
    % \hfill
     \hspace*{-17pt}
    \begin{minipage}[c]{0.33\columnwidth}
    \centering
    \includegraphics[trim={150pt 50pt 0 200pt},clip,width=\columnwidth]{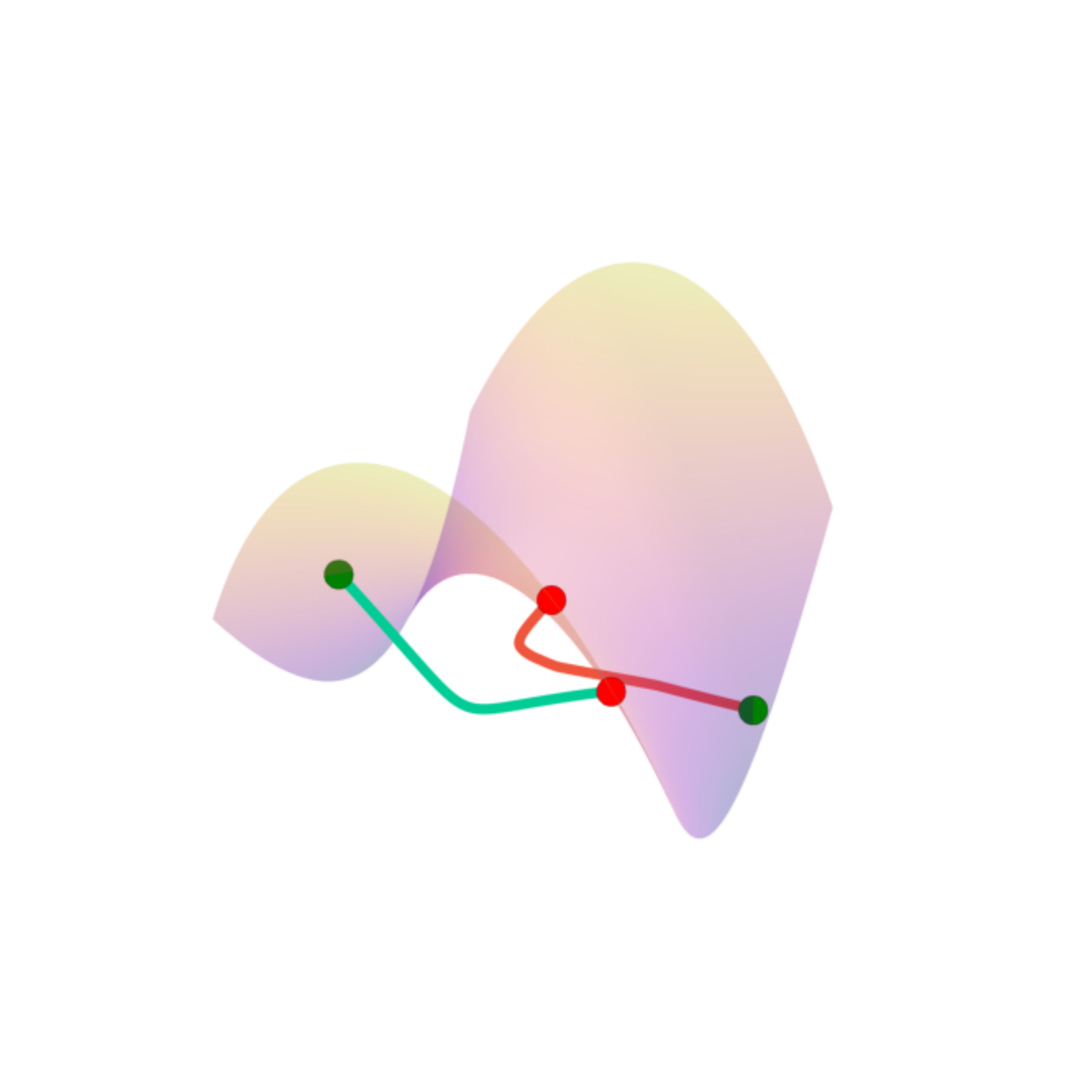}
    \end{minipage}
    % \hfill
     \hspace*{-17pt}
    \begin{minipage}[c]{0.33\columnwidth}
    \centering
    \includegraphics[trim={150pt 50pt 0 200pt},clip,width=\columnwidth]{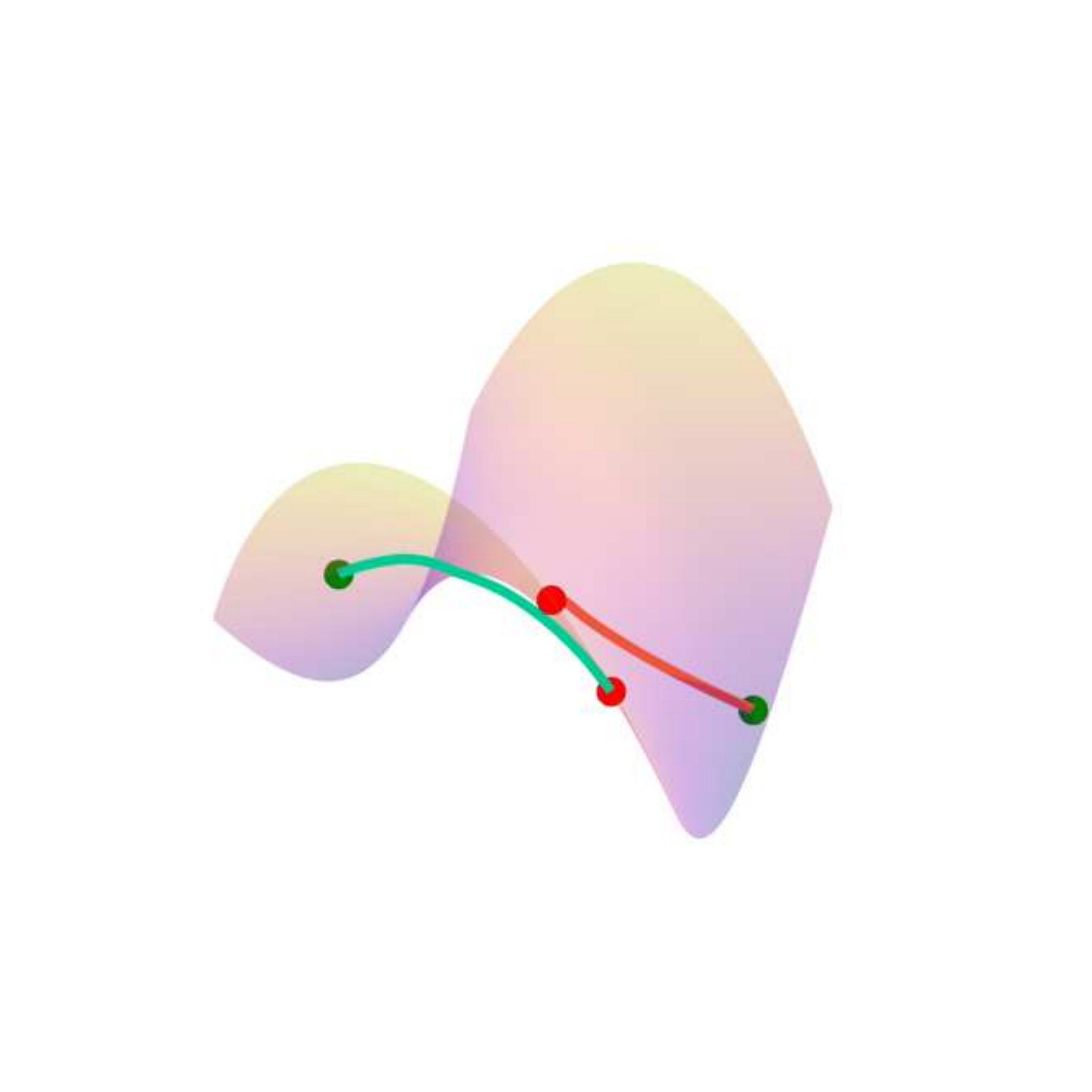}
    \end{minipage}
    %\hfill
     \hspace*{-17pt}
    % \begin{minipage}[c]{0.28\columnwidth}
    % \centering
    % \includegraphics[width=\columnwidth]{fig/geovis_density_saddle.pdf}
    % \end{minipage}
    \\[-10pt]

    %%%%%%% 4th Row %%%%%%%
    % \begin{minipage}[c]{0.27\columnwidth}
    % \centering
    % \includegraphics[width=\columnwidth]{fig/geovis_gt_hemisphere.pdf}
    % \end{minipage}
    % % \hfill
    %  \hspace*{-21pt}
    % \begin{minipage}[c]{0.27\columnwidth}
    % \centering
    % \includegraphics[width=\columnwidth]{fig/geovis_ours_hemisphere.pdf}
    % \end{minipage}
    % % \hfill
    %  \hspace*{-21pt}
    % \begin{minipage}[c]{0.27\columnwidth}
    % \centering
    % \includegraphics[width=\columnwidth]{fig/geovis_no_density_hemisphere.pdf}
    % \end{minipage}
    % % \hfill
    %  \hspace*{-21pt}
    % \begin{minipage}[c]{0.27\columnwidth}
    % \centering
    % \includegraphics[width=\columnwidth]{fig/geovis_density_hemisphere.pdf}
    % \end{minipage}

\vspace{-6pt}
\caption{
    Comparison of ground truth and learned geodesics.
    From left to right: 1) ground truth, 2) no warping, 3) \methodshort.
}
\label{fig:geovis}
\end{figure}

\begin{figure}[tb!]
    \vspace{-12pt}
    \centering
    \includegraphics[width=0.6\columnwidth]{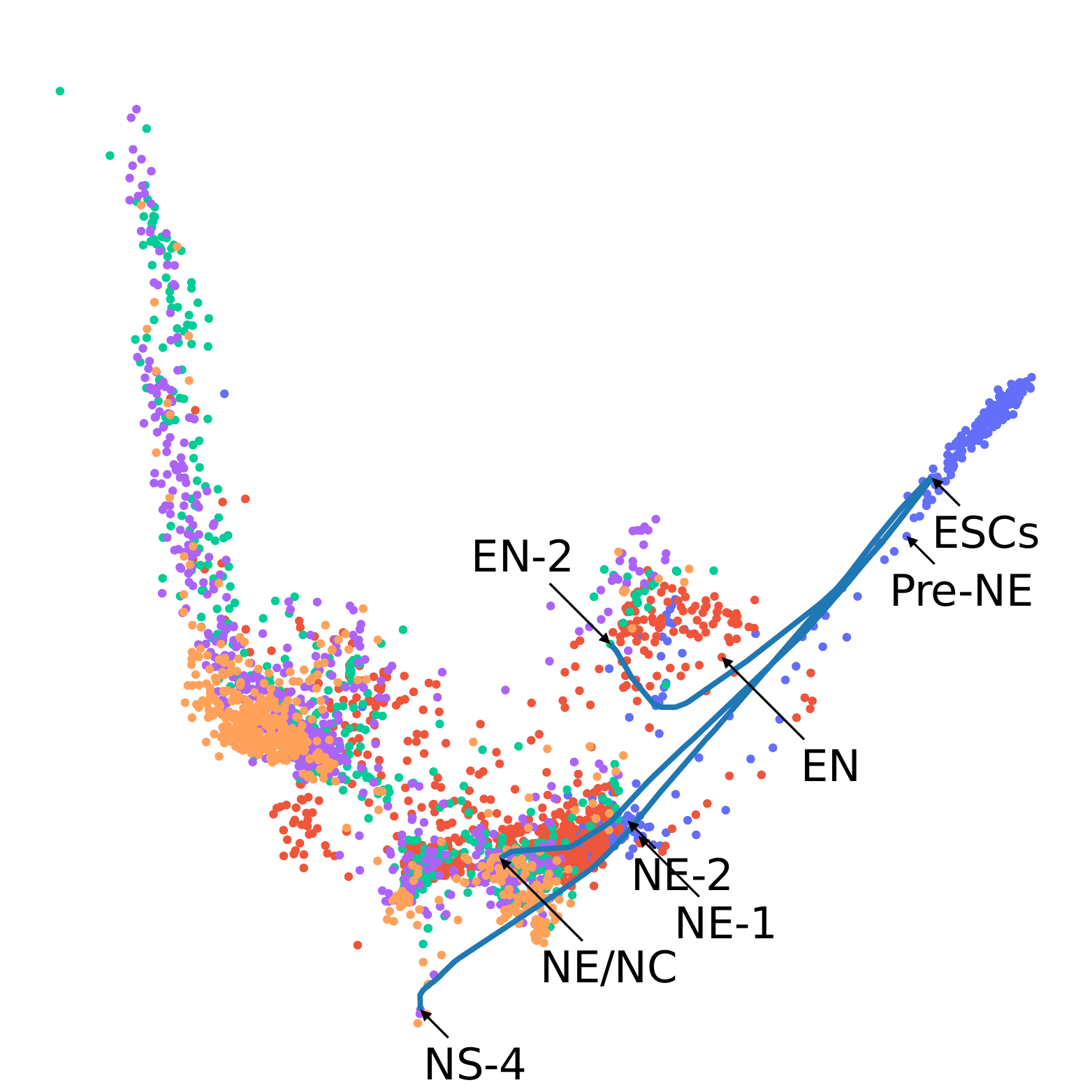}
    \vspace{-8pt}
    \caption{Geodesics learned on Embryoid Body data.
    }
    \label{fig:geod_eb}
    \vspace{-2pt}
\end{figure}

As shown in \Cref{tab:geodesic_mse}, \methodshort generally outperforms all other methods except for one case (Djikstra's on saddle). It is worth mentioning that Dijkstra's algorithm is only capable of connecting existing points but unable to generate points along the path. Directly using the metric without warping performs the worst by a big margin. We visualized the predicted geodesics on torus and saddle (\Cref{fig:geovis}). In general, trajectories generated by \methodshort stay on the manifold and are close to the ground truth geodesics, whereas some learned by the metric without the warping either deviate from the ground truth or directly cut through the manifold.
More details and results are provided in \Cref{appdx:results_geodesics_noisy_setting} and \Cref{appdx:results_geovis_full}.

In addition to toy datasets, we also visualized the geodesics learned on the Embryoid Body dataset (\Cref{fig:geod_eb}).
The starting points correspond to stem cells, while the ending points are selected at different lineages.
The predicted geodesics recover the corresponding differentiation branches, aligning with the biological understanding of the data.

% For qualitative evaluation, we visualize the predicted geodesics in \Cref{fig:geovis}.
% In general, trajectories generated by \methodshort stay on the manifold and are close to the ground truth geodesics, whereas some of these learned by local metric either deviate from the ground truth or cut through the manifold. In addition to toy datasets, we also visualize the geodesics learned on the Embryoid Body dataset (\Cref{fig:geod_eb}).
% The starting point corresponds to stem cells, while the ending points are selected at different lineages of differentiation.
% The predicted geodesics recover the corresponding differentiation branches, aligning with our biological understanding of the data.

%\subsection{Geodesics-Guided Flow Matching}
% \paragraph{Geodesics-guided Flow Matching}
\paragraph{Population Interpolation along geodesics}
% {\bf Geodesics-guided Flow Matching.}~

In the final application, we evaluate geodesics-guided population transport on simulated and real data.

For the simulated dataset, \methodshort transports the source population to the target population through geodesics, which means that the trajectories remain on the manifold and follow the shortest paths~(\Cref{fig:fm}).
See \Cref{appx:geod_fm} for details.

Finally, we considered single-cell trajectory inference on the CITE-seq and Multiome datasets from a NeurIPS competition \citep{CITE_and_Multi}. We performed the leave-one-timepoint-out cellular dynamics experiment in which points at one timepoint are excluded, and the goal is to infer the left-out points by interpolating between the remaining timesteps. \methodshort consistently outperforms all other methods by a large margin~(\Cref{tab:cite_multi}). See details in \Cref{appdx:results_single_cell_traj_inference}.

\begin{figure}[tb!]
    \centering
    % \vspace{-20pt}
    \begin{minipage}[c]{0.23\columnwidth}
    \centering
    \includegraphics[trim={30pt 75pt 50pt 75pt},clip,width=\columnwidth]{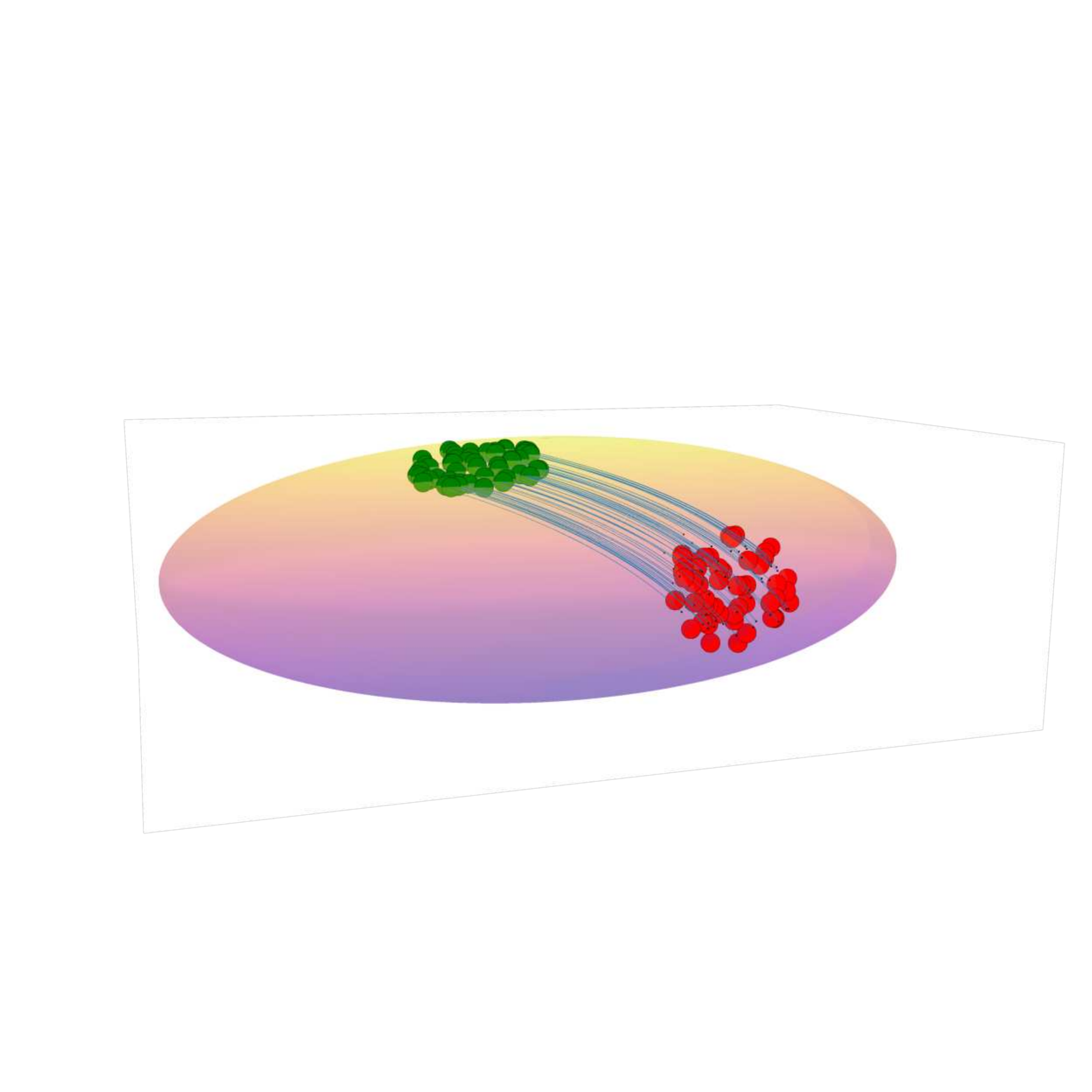}
    \end{minipage}
    \begin{minipage}[c]{0.23\columnwidth}
    \centering
    \includegraphics[trim={30pt 75pt 50pt 75pt},clip,width=\columnwidth]
    {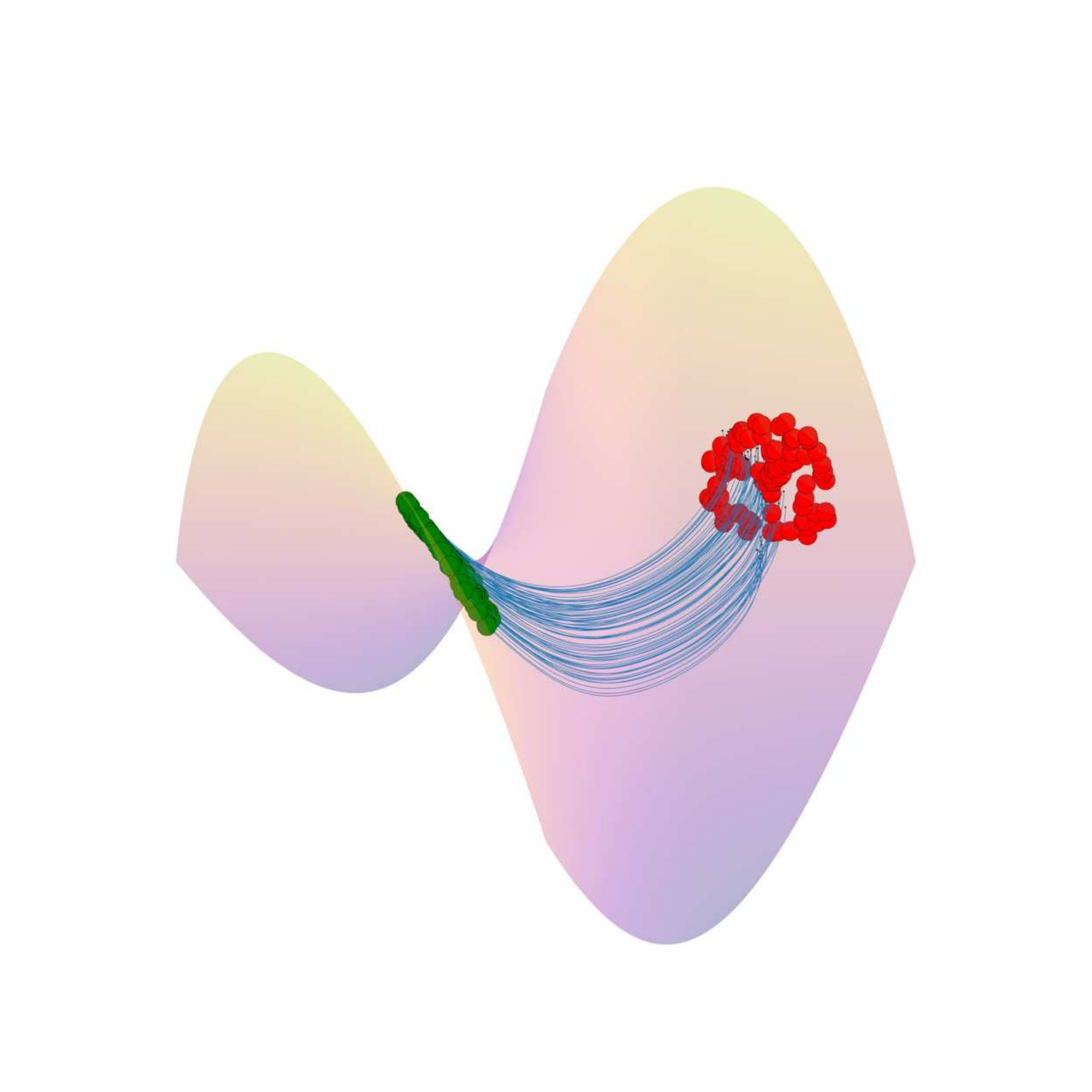}
    \end{minipage}
    \begin{minipage}[c]{0.23\columnwidth}
    \centering
    \includegraphics[trim={30pt 75pt 50pt 75pt},clip,width=\columnwidth]
    {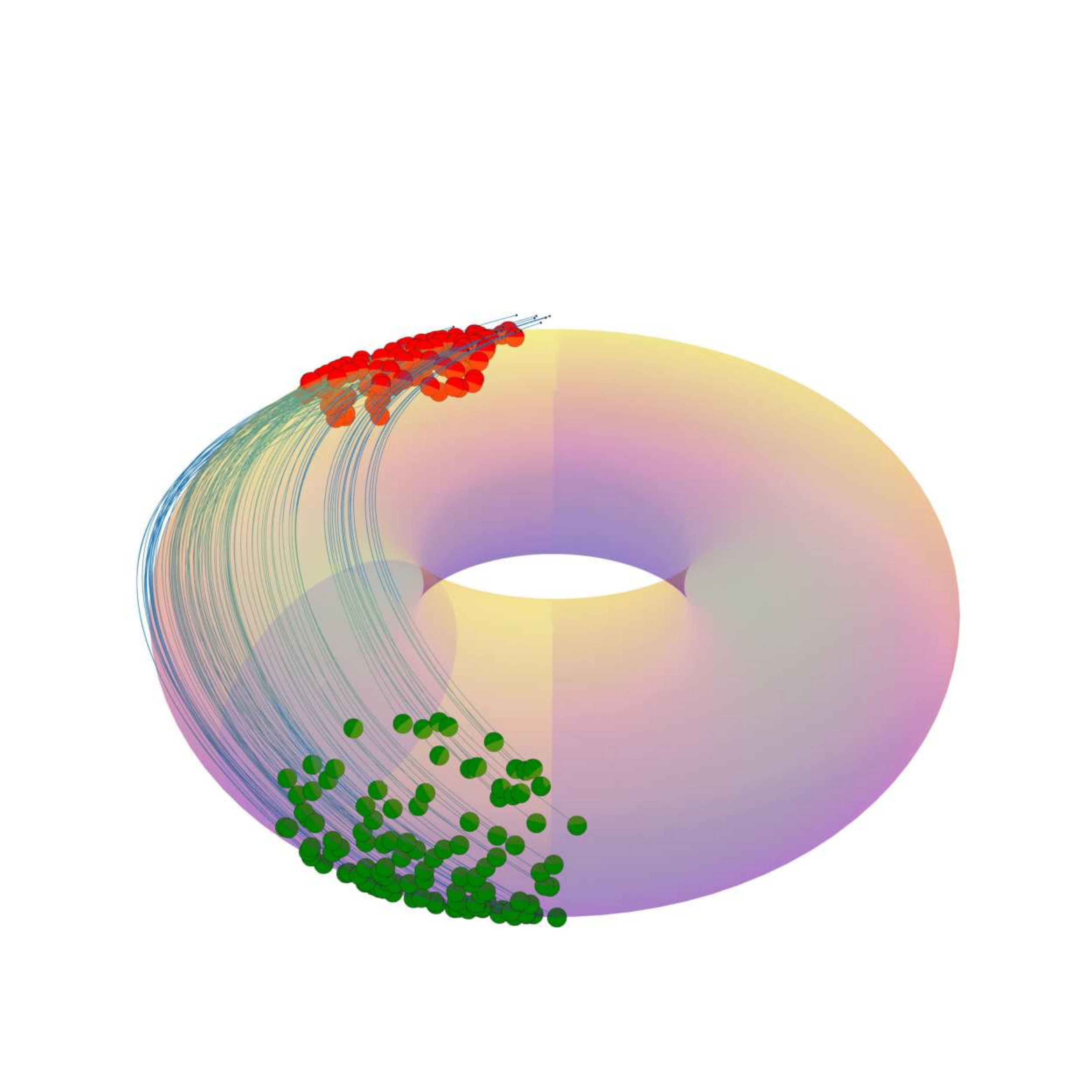}
    \end{minipage}
    \begin{minipage}[c]{0.23\columnwidth}
    \centering
    \includegraphics[trim={30pt 75pt 50pt 75pt},clip,width=\columnwidth]{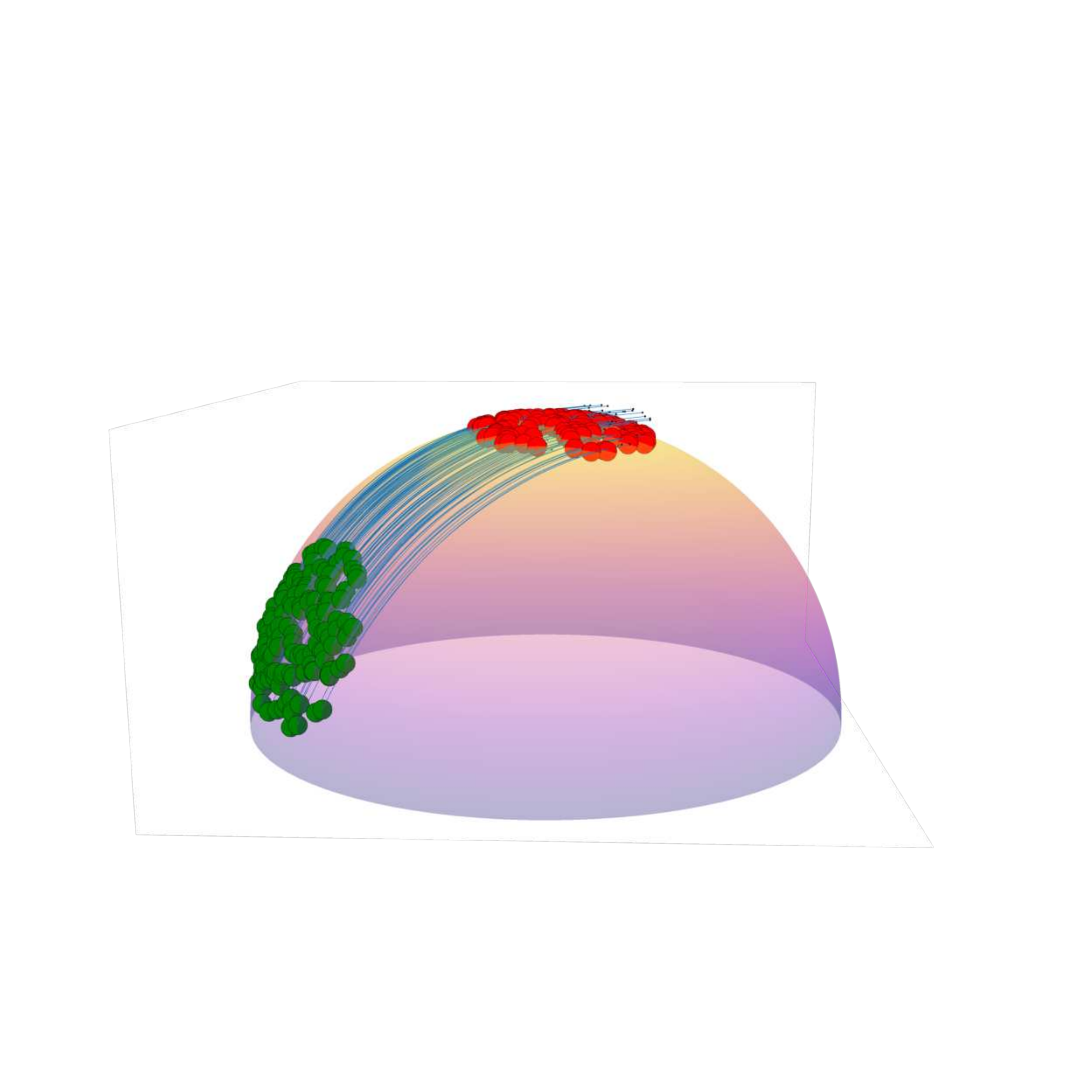}
    \end{minipage}
% \vspace{-4pt}
\caption{Transporting populations on toy manifolds.}
\label{fig:fm}
\vspace{-10pt}
\end{figure}

We have included details of our hyperparameter selection in \Cref{appx:hparams} and additional experiment results in \Cref{appx:addn_exp}.

\vspace*{-2pt}
\section{RELATED WORK}
\vspace*{-4pt}

\paragraph{Geometry-aware encoding} Non-linear dimensionality reduction methods such as PHATE or diffusion maps have proven useful in learning manifold structure from high-dimensional data. However, they have been difficult to extend to generate or sample new points~\citep{HeatGeo}. To address this, some prior works have tried to regularize an encoder to match the embeddings or distances obtained from dimensionality reduction methods~\citep{duque2020extendable, GRAE, DiffKillR, MDMRAE, NeuralFIM}, or by minizing the Gromov-Monge cost~\citep{lee2024monotonegenerativemodelinggromovmonge}. Despite embedding or distance preservation, these methods have not focused on generative modeling of points, can struggle in gaps for trajectory inference ~\citep{lee2024monotonegenerativemodelinggromovmonge}, or sometimes do not decode the data at all and simply provide embeddings~\citep{NeuralFIM}. As a result, it is difficult to use existing embeddings to generate or sample new points on and along these manifolds faithfully. \vspace{-8pt}

\paragraph{Interpolating between points} For interpolating between data points, traditional approaches often rely on linear interpolation or latent space traversal that does not align with complex data trajectories~\citep{Linear_interp, Latent_space_interp}. Some recent methods use a neural network to learn the gradient field, where optimal trajectories can be computed by following the gradient~\citep{MIOFlow, ImageFlowNet}. However, these methods suffer from error accumulation, which may lead to large deviations when the trajectory is sufficiently long. \vspace{-8pt}
\paragraph{Population transport} Transporting populations across experimental conditions, time points, or biological states is usually approached by flow matching~\citep{Flow_matching, Conditional_flow_matching} and bridge matching ~\citep{shi2023diffusionschrodingerbridgematching, thornton2022riemanniandiffusionschrodingerbridge}. 
Diffusion Schrödinger Bridge Matching ~\citep{shi2023diffusionschrodingerbridgematching} and Minibatch Optimal Transport Flow Matching ~\citep{tong2024improvinggeneralizingflowbasedgenerative} operate on Euclidean space without considering the underlying manifold, and thus cannot generate trajectories along the manifold. Simulation-Free Schrödinger Bridges ~\citep{tong2024simulationfreeschrodingerbridgesscore} requires closed-form conditional path distributions, and therefore, it’s not applicable to general manifold without analytic solutions.

Some of the recent work attempts to access the manifold and leverage the non-Euclidean metric. For example, Riemannian Diffusion Schrödinger Bridge ~\citep{thornton2022riemanniandiffusionschrodingerbridge} addresses the Schrödinger Bridge problem in non-Euclidean space but requires the metric to perform manifold projection. Flow Matching on General Geometries ~\citep{chen2024flowmatchinggeneralgeometries} requires closed-form solutions for computing geodesics on simple geometries and uses premetric instead of metric on general geometries.
Solving Wasserstein Lagrangian Flows ~\citep{WLF} does not learn the metric of the manifold but instead uses prefixed Wasserstein-2 and Wasserstein Fisher-Rao metrics on the statistical manifold of the probability measures, thus unable to transport populations faithfully along the data manifold. The work most comparable to GAGA’s metric learning framework is Metric Flow Matching ~\citep{MFM}, which learns the manifold metric but only considers a specific family of diagonal metrics: metric LAND and RBF, tunable by a few hyperparameters. 

% In contrast to these methods, GAGA does not assume a particular metric or a specific family of metrics nor requires an analytic solution for the metric. Rather, GAGA learns the metric of the underlying manifold directly from the input data and thus is applicable to any manifold dataset in general.
% \xin{cite the Gilad Lerman’s paper}

\vspace*{-2pt}
\section{DISCUSSION}
\vspace*{-4pt}
% \xin{move to appendix?}
% \section{Larger Scope, Limitations and future directions}\label{appx:discuss}
% Our work lies at the intersection of manifold learning, representation learning, and generative modeling. It combines the strengths of discrete manifold learning methods (e.g., Diffusion Maps, PHATE) with the flexibility of smooth deep learning, enabling explicit characterization of the Riemannian metric across the entire manifold in high-dimensional, noisy data.

% We introduce a novel paradigm of geometry-guided generation, which fundamentally differs from density-based approaches in popular generative models. This paradigm demonstrates significant effectiveness in reducing sampling bias in unbalanced data and in generating trajectories along the data manifold.
% Additionally, we showcase the utility of our inferred Riemannian metric in generating geodesics and enabling population transport along the manifold, with important applications in areas such as single-cell trajectory inference.
% A limitation of our work is the requirement to train multiple model components sequentially, which can increase complexity and computational time. Additionally, the Unadjusted Langevin Algorithm (ULA) used in our approach could be further optimized to enhance efficiency and performance.

We have explored encoding whether a given point lies on or off the manifold within our warped pullback metric. In the future, we aim to investigate additional priors and biases that could be incorporated into the metric to further guide population transport. For instance, encoding data sparsity to enable the generation of rare events or incorporating desired chemical properties for molecule generation.

Another promising direction is utilizing our learned metric to compute other geometric quantities and operators, such as curvatures, and log and exponential maps for manifold projections. 

% Furthermore, we plan to further investigate volume-guided generation using techniques like the Metropolis-Adjusted Langevin Algorithm, Hamiltonian Monte Carlo, and Riemannian Langevin Dynamics
% \citep{girolami2011riemann,wang2020fast}.

\vspace*{-2pt}
\section{CONCLUSION}
\vspace*{-4pt}

In this paper, we propose a geometry-aware generative autoencoder (\methodshort) that preserves geometry in latent embeddings and can generate new points uniformly on the data manifold, interpolate along the geodesics, and transport populations across the manifold. We circumvent the limitations of existing generative methods, which mainly match the modes of distributions, by training generalizable geometry-aware neural network embeddings, leveraging points both on and off the data manifold, and learning a novel warped Riemannian metric on data space that allows us to generate points from the data geometry. %via Riemannian pullback metric.

\section*{Acknowledgments}

This research was partially funded and supported by ESP M\'erite [G.H.], CIFAR AI Chair [G.W.], NSERC Discovery grant 03267 [G.W.], NIH grants (1F30AI157270-01, R01HD100035, R01GM130847, R01GM135929) [G.W.,S.K.], NSF Career grant 2047856 [S.K.], NSF grant 2327211 [S.K., G.W., M.P., I.A.], NSF/NIH grant 1R01GM135929-01 [M.H., S.K.], the Chan-Zuckerberg Initiative grants CZF2019-182702 and CZF2019-002440 [S.K.], the Sloan Fellowship FG-2021-15883 [S.K.], and the Novo Nordisk grant GR112933 [S.K.]. The tent provided here is solely the responsibility of the authors and does not necessarily represent the official views of the funding agencies. The funders had no role in study design, data collection and analysis, decision to publish, or preparation of the manuscript.

Chen Liu helped with refining abstract, introduction, and the schematic figure prior to submission and should be considered an author.

% \section*{Disclosure}
% This work is an extension of a previous workshop publication at the ICML 2024 Workshop on Geometry-grounded Representation Learning and Generative Modeling~\citep{GAGA_workshop}.
% \xin{add “Chen Liu helped with xxxx prior to submission and should be considered an author” in the acknowledgment.}

% \clearpage
% \newpage

\bibliography{references}

\begin{thebibliography}{}

\bibitem[Aggarwal et~al., 2001]{Distance_high_dim}
Aggarwal, C.~C., Hinneburg, A., and Keim, D.~A. (2001).
\newblock On the surprising behavior of distance metrics in high dimensional space.
\newblock In {\em Database theory—ICDT 2001: 8th international conference London, UK, January 4--6, 2001 proceedings 8}, pages 420--434. Springer.

\bibitem[Arjovsky et~al., 2017]{WGAN}
Arjovsky, M., Chintala, S., and Bottou, L. (2017).
\newblock Wasserstein generative adversarial networks.
\newblock In {\em International conference on machine learning}, pages 214--223. PMLR.

\bibitem[Benamou and Brenier, 2000]{benamou2000computational}
Benamou, J.-D. and Brenier, Y. (2000).
\newblock A computational fluid mechanics solution to the monge-kantorovich mass transfer problem.
\newblock {\em Numerische Mathematik}, 84(3):375--393.

\bibitem[Bhaskar et~al., 2023]{DYMAG}
Bhaskar, D., Zhang, Y., Xu, C., Sun, X., Fasina, O., Wolf, G., Nickel, M., Perlmutter, M., and Krishnaswamy, S. (2023).
\newblock Learning graph geometry and topology using dynamical systems based message-passing.
\newblock {\em arXiv preprint arXiv:2309.09924}.

\bibitem[Burkhardt et~al., 2022]{CITE_and_Multi}
Burkhardt, D., Bloom, J., Cannoodt, R., Luecken, M.~D., Krishnaswamy, S., Lance, C., Pisco, A.~O., and Theis, F.~J. (2022).
\newblock Multimodal single-cell integration across time, individuals, and batches.
\newblock {\em NeurIPS Competitions}.

\bibitem[Chen and Lipman, 2023]{chen2023flow}
Chen, R.~T. and Lipman, Y. (2023).
\newblock Flow matching on general geometries.
\newblock {\em arXiv preprint arXiv:2302.03660}.

\bibitem[Chen and Lipman, 2024]{chen2024flowmatchinggeneralgeometries}
Chen, R. T.~Q. and Lipman, Y. (2024).
\newblock Flow matching on general geometries.

\bibitem[Coifman and Lafon, 2006]{DiffusionMaps}
Coifman, R.~R. and Lafon, S. (2006).
\newblock Diffusion maps.
\newblock {\em Applied and computational harmonic analysis}, 21(1):5--30.

\bibitem[Do~Carmo and Flaherty, 1992]{do1992riemannian}
Do~Carmo, M.~P. and Flaherty, F. (1992).
\newblock {\em Riemannian geometry}, volume~2.
\newblock Springer.

\bibitem[Duque et~al., 2020]{duque2020extendable}
Duque, A.~F., Morin, S., Wolf, G., and Moon, K. (2020).
\newblock Extendable and invertible manifold learning with geometry regularized autoencoders.
\newblock In {\em 2020 IEEE International Conference on Big Data (Big Data)}, pages 5027--5036. IEEE.

\bibitem[Duque et~al., 2022]{GRAE}
Duque, A.~F., Morin, S., Wolf, G., and Moon, K.~R. (2022).
\newblock Geometry regularized autoencoders.
\newblock {\em IEEE transactions on pattern analysis and machine intelligence}, 45(6):7381--7394.

\bibitem[Fasina et~al., 2023]{NeuralFIM}
Fasina, O., Huguet, G., Tong, A., Zhang, Y., Wolf, G., Nickel, M., Adelstein, I., and Krishnaswamy, S. (2023).
\newblock Neural fim for learning fisher information metrics from point cloud data.
\newblock In {\em International Conference on Machine Learning}, pages 9814--9826. PMLR.

\bibitem[Fefferman et~al., 2016]{Manifold_hypothesis}
Fefferman, C., Mitter, S., and Narayanan, H. (2016).
\newblock Testing the manifold hypothesis.
\newblock {\em Journal of the American Mathematical Society}, 29(4):983--1049.

\bibitem[Goodfellow et~al., 2020]{GAN}
Goodfellow, I., Pouget-Abadie, J., Mirza, M., Xu, B., Warde-Farley, D., Ozair, S., Courville, A., and Bengio, Y. (2020).
\newblock Generative adversarial networks.
\newblock {\em Communications of the ACM}, 63(11):139--144.

\bibitem[Heimberg et~al., 2016]{heimberg2016low}
Heimberg, G., Bhatnagar, R., El-Samad, H., and Thomson, M. (2016).
\newblock Low dimensionality in gene expression data enables the accurate extraction of transcriptional programs from shallow sequencing.
\newblock {\em Cell systems}, 2(4):239--250.

\bibitem[Huang et~al., 2022]{MDMRAE}
Huang, J., Busch, E., Wallenstein, T., Gerasimiuk, M., Benz, A., Lajoie, G., Wolf, G., Turk-Browne, N., and Krishnaswamy, S. (2022).
\newblock Learning shared neural manifolds from multi-subject fmri data.
\newblock In {\em 2022 IEEE 32nd International Workshop on Machine Learning for Signal Processing (MLSP)}, pages 01--06. IEEE.

\bibitem[Huguet et~al., 2022]{MIOFlow}
Huguet, G., Magruder, D.~S., Tong, A., Fasina, O., Kuchroo, M., Wolf, G., and Krishnaswamy, S. (2022).
\newblock Manifold interpolating optimal-transport flows for trajectory inference.

\bibitem[Huguet et~al., 2024]{HeatGeo}
Huguet, G., Tong, A., De~Brouwer, E., Zhang, Y., Wolf, G., Adelstein, I., and Krishnaswamy, S. (2024).
\newblock A heat diffusion perspective on geodesic preserving dimensionality reduction.
\newblock {\em Advances in Neural Information Processing Systems}, 36.

\bibitem[Jindal et~al., 2018]{scRNAseq_discovery1}
Jindal, A., Gupta, P., Jayadeva, and Sengupta, D. (2018).
\newblock Discovery of rare cells from voluminous single cell expression data.
\newblock {\em Nature communications}, 9(1):4719.

\bibitem[Kapusniak et~al., 2024]{MFM}
Kapusniak, K., Potaptchik, P., Reu, T., Zhang, L., Tong, A., Bronstein, M., Bose, A.~J., and Di~Giovanni, F. (2024).
\newblock Metric flow matching for smooth interpolations on the data manifold.
\newblock {\em arXiv preprint arXiv:2405.14780}.

\bibitem[Krawczyk, 2016]{Imbalanced_data}
Krawczyk, B. (2016).
\newblock Learning from imbalanced data: open challenges and future directions.
\newblock {\em Progress in artificial intelligence}, 5(4):221--232.

\bibitem[Lee et~al., 2024]{lee2024monotonegenerativemodelinggromovmonge}
Lee, W., Yang, Y., Zou, D., and Lerman, G. (2024).
\newblock Monotone generative modeling via a gromov-monge embedding.

\bibitem[Liao et~al., 2024]{DSE}
Liao, D., Liu, C., Christensen, B.~W., Tong, A., Huguet, G., Wolf, G., Nickel, M., Adelstein, I., and Krishnaswamy, S. (2024).
\newblock Assessing neural network representations during training using noise-resilient diffusion spectral entropy.
\newblock In {\em 2024 58th Annual Conference on Information Sciences and Systems (CISS)}, pages 1--6. IEEE.

\bibitem[Lipman et~al., 2022]{Flow_matching}
Lipman, Y., Chen, R.~T., Ben-Hamu, H., Nickel, M., and Le, M. (2022).
\newblock Flow matching for generative modeling.
\newblock {\em arXiv preprint arXiv:2210.02747}.

\bibitem[Liu et~al., 2024]{CUTS}
Liu, C., Amodio, M., Shen, L.~L., Gao, F., Avesta, A., Aneja, S., Wang, J.~C., Del~Priore, L.~V., and Krishnaswamy, S. (2024).
\newblock Cuts: A deep learning and topological framework for multigranular unsupervised medical image segmentation.
\newblock In {\em proceedings of Medical Image Computing and Computer Assisted Intervention -- MICCAI 2024}, volume LNCS 15008. Springer Nature Switzerland.

\bibitem[Liu et~al., 2025a]{DiffKillR}
Liu, C., Liao, D., Parada-Mayorga, A., Ribeiro, A., DiStasio, M., and Krishnaswamy, S. (2025a).
\newblock {DiffKillR: Killing and Recreating Diffeomorphisms for Cell Annotation in Dense Microscopy Images}.
\newblock In {\em ICASSP 2025-2025 IEEE International Conference on Acoustics, Speech and Signal Processing (ICASSP)}. IEEE.

\bibitem[Liu et~al., 2025b]{ImageFlowNet}
Liu, C., Xu, K., Shen, L.~L., Huguet, G., Wang, Z., Tong, A., Bzdok, D., Stewart, J., Wang, J.~C., Del~Priore, L.~V., and Krishnaswamy, S. (2025b).
\newblock {ImageFlowNet: Forecasting Multiscale Trajectories of Disease Progression with Irregularly-Sampled Longitudinal Medical Images}.
\newblock In {\em ICASSP 2025-2025 IEEE International Conference on Acoustics, Speech and Signal Processing (ICASSP)}. IEEE.

\bibitem[Mart{\'\i}nez-Minaya et~al., 2018]{Distribution_matching_challenges}
Mart{\'\i}nez-Minaya, J., Cameletti, M., Conesa, D., and Pennino, M.~G. (2018).
\newblock Species distribution modeling: a statistical review with focus in spatio-temporal issues.
\newblock {\em Stochastic environmental research and risk assessment}, 32:3227--3244.

\bibitem[Mi et~al., 2021]{Latent_space_interp}
Mi, L., He, T., Park, C.~F., Wang, H., Wang, Y., and Shavit, N. (2021).
\newblock Revisiting latent-space interpolation via a quantitative evaluation framework.
\newblock {\em arXiv preprint arXiv:2110.06421}.

\bibitem[Michelis and Becker, 2021]{Linear_interp}
Michelis, M.~Y. and Becker, Q. (2021).
\newblock On linear interpolation in the latent space of deep generative models.
\newblock {\em arXiv preprint arXiv:2105.03663}.

\bibitem[Miyato et~al., 2018]{miyato2018spectral}
Miyato, T., Kataoka, T., Koyama, M., and Yoshida, Y. (2018).
\newblock Spectral normalization for generative adversarial networks.
\newblock {\em arXiv preprint arXiv:1802.05957}.

\bibitem[Moitra and Risteski, 2020]{moitra2020fast}
Moitra, A. and Risteski, A. (2020).
\newblock Fast convergence for langevin diffusion with manifold structure.
\newblock {\em arXiv preprint arXiv:2002.05576}.

\bibitem[Moon et~al., 2018]{mfd_review}
Moon, K.~R., Stanley~III, J.~S., Burkhardt, D., van Dijk, D., Wolf, G., and Krishnaswamy, S. (2018).
\newblock Manifold learning-based methods for analyzing single-cell rna-sequencing data.
\newblock {\em Current Opinion in Systems Biology}, 7:36--46.

\bibitem[Moon et~al., 2019]{PHATE}
Moon, K.~R., Van~Dijk, D., Wang, Z., Gigante, S., Burkhardt, D.~B., Chen, W.~S., Yim, K., Elzen, A. v.~d., Hirn, M.~J., Coifman, R.~R., et~al. (2019).
\newblock Visualizing structure and transitions in high-dimensional biological data.
\newblock {\em Nature biotechnology}, 37(12):1482--1492.

\bibitem[Neklyudov et~al., 2024]{WLF}
Neklyudov, K., Brekelmans, R., Tong, A., Atanackovic, L., Makhzani, A., et~al. (2024).
\newblock A computational framework for solving wasserstein lagrangian flows.
\newblock In {\em Forty-first International Conference on Machine Learning}.

\bibitem[Shi et~al., 2023]{shi2023diffusionschrodingerbridgematching}
Shi, Y., Bortoli, V.~D., Campbell, A., and Doucet, A. (2023).
\newblock Diffusion schr\"odinger bridge matching.

\bibitem[Shi et~al., 2024]{DSBM}
Shi, Y., De~Bortoli, V., Campbell, A., and Doucet, A. (2024).
\newblock Diffusion schr{\"o}dinger bridge matching.
\newblock {\em Advances in Neural Information Processing Systems}, 36.

\bibitem[Sun et~al., 2025]{Spatial_discovery2}
Sun, X., Xu, C., Rocha, J.~F., Liu, C., Hollander-Bodie, B., Goldman, L., DiStasio, M., Perlmutter, M., and Krishnaswamy, S. (2025).
\newblock Hyperedge representations with hypergraph wavelets: Applications to spatial transcriptomics.
\newblock In {\em ICASSP 2025-2025 IEEE International Conference on Acoustics, Speech and Signal Processing (ICASSP)}. IEEE.

\bibitem[Thornton et~al., 2022]{thornton2022riemanniandiffusionschrodingerbridge}
Thornton, J., Hutchinson, M., Mathieu, E., Bortoli, V.~D., Teh, Y.~W., and Doucet, A. (2022).
\newblock Riemannian diffusion schr\"odinger bridge.

\bibitem[Tong et~al., 2024a]{tong2024improvinggeneralizingflowbasedgenerative}
Tong, A., Fatras, K., Malkin, N., Huguet, G., Zhang, Y., Rector-Brooks, J., Wolf, G., and Bengio, Y. (2024a).
\newblock Improving and generalizing flow-based generative models with minibatch optimal transport.

\bibitem[Tong et~al., 2020]{TrajectoryNet}
Tong, A., Huang, J., Wolf, G., Van~Dijk, D., and Krishnaswamy, S. (2020).
\newblock Trajectorynet: A dynamic optimal transport network for modeling cellular dynamics.
\newblock In {\em International conference on machine learning}, pages 9526--9536. PMLR.

\bibitem[Tong et~al., 2024b]{tong2024simulationfreeschrodingerbridgesscore}
Tong, A., Malkin, N., Fatras, K., Atanackovic, L., Zhang, Y., Huguet, G., Wolf, G., and Bengio, Y. (2024b).
\newblock Simulation-free schr\"odinger bridges via score and flow matching.

\bibitem[Tong et~al., 2023]{Conditional_flow_matching}
Tong, A., Malkin, N., Huguet, G., Zhang, Y., Rector-Brooks, J., Fatras, K., Wolf, G., and Bengio, Y. (2023).
\newblock Improving and generalizing flow-based generative models with minibatch optimal transport.
\newblock {\em arXiv preprint arXiv:2302.00482}.

\bibitem[Tong et~al., 2022]{tong2022fixing}
Tong, A., Wolf, G., and Krishnaswamy, S. (2022).
\newblock Fixing bias in reconstruction-based anomaly detection with lipschitz discriminators.
\newblock {\em Journal of Signal Processing Systems}, 94(2):229--243.

\bibitem[Tong et~al., 2024c]{SFSFM}
Tong, A.~Y., Malkin, N., Fatras, K., Atanackovic, L., Zhang, Y., Huguet, G., Wolf, G., and Bengio, Y. (2024c).
\newblock Simulation-free schr{\"o}dinger bridges via score and flow matching.
\newblock In {\em International Conference on Artificial Intelligence and Statistics}, pages 1279--1287. PMLR.

\bibitem[Van~de Sande et~al., 2023]{scRNAseq_discovery2}
Van~de Sande, B., Lee, J.~S., Mutasa-Gottgens, E., Naughton, B., Bacon, W., Manning, J., Wang, Y., Pollard, J., Mendez, M., Hill, J., et~al. (2023).
\newblock Applications of single-cell rna sequencing in drug discovery and development.
\newblock {\em Nature Reviews Drug Discovery}, 22(6):496--520.

\bibitem[Van~Dijk et~al., 2018]{van2018recovering}
Van~Dijk, D., Sharma, R., Nainys, J., Yim, K., Kathail, P., Carr, A.~J., Burdziak, C., Moon, K.~R., Chaffer, C.~L., Pattabiraman, D., et~al. (2018).
\newblock Recovering gene interactions from single-cell data using data diffusion.
\newblock {\em Cell}, 174(3):716--729.

\bibitem[Wang et~al., 2022]{ATAQseq_discovery1}
Wang, Y., Sun, X., and Zhao, H. (2022).
\newblock Benchmarking automated cell type annotation tools for single-cell atac-seq data.
\newblock {\em Frontiers in Genetics}, 13:1063233.

\bibitem[Zappia et~al., 2017]{zappia2017splatter}
Zappia, L., Phipson, B., and Oshlack, A. (2017).
\newblock Splatter: simulation of single-cell rna sequencing data.
\newblock {\em Genome biology}, 18(1):174.

\bibitem[Zhao et~al., 2022]{Spatial_discovery1}
Zhao, T., Chiang, Z.~D., Morriss, J.~W., LaFave, L.~M., Murray, E.~M., Del~Priore, I., Meli, K., Lareau, C.~A., Nadaf, N.~M., Li, J., et~al. (2022).
\newblock Spatial genomics enables multi-modal study of clonal heterogeneity in tissues.
\newblock {\em Nature}, 601(7891):85--91.

\end{thebibliography}
\bibliographystyle{apalike}

%%%%%%%%%%%%%%%%%%% Checklist %%%%%%%%%%%%%%%%%%%%
\include{checklist}

%%%%%%%%%%%%%%%%%%% Appendix %%%%%%%%%%%%%%%%%%%%%

\begin{appendix}

\crefalias{section}{appsec}
\crefalias{subsection}{appsec}
\crefalias{subsubsection}{appsec}

% Rename supplementary figures.
\setcounter{figure}{0}
\renewcommand{\thefigure}{Appx.~\arabic{figure}}
\renewcommand{\theHfigure}{Appx.~\arabic{figure}}
% Rename supplementary tables.
\setcounter{table}{0}
\renewcommand{\thetable}{Appx.~\arabic{table}}
\renewcommand{\theHtable}{Appx.~\arabic{table}}
% Rename supplementary equations.
\setcounter{equation}{0}
\renewcommand{\theequation}{Appx.~\arabic{equation}}
\renewcommand{\theHequation}{Appx.~\arabic{equation}}

\onecolumn

\section*{\LARGE Appendix}
\label{sec:appendix}

\addtocontents{toc}{\protect\setcounter{tocdepth}{2}}

\renewcommand{\contentsname}{\centering Table of Contents}
\tableofcontents

\clearpage
\newpage

% Add page number
\pagestyle{plain}
\pagestyle{pagenumbers}

\section{Manifold Learning and Diffusion Geometry}\label{appx:mfd_learning}

The \textit{Manifold Hypothesis} states that data are often sampled \textit{on} or \textit{near} an intrinsically low-dimensional manifold within high-dimensional Euclidean space. Manifold learning techniques aim to uncover and recreate this manifold in a lower-dimensional space.

Many manifold learning approaches use data \textit{diffusion geometry}, which extracts geometric features from an approximation of heat flow on the data. Diffusion geometry models a high-dimensional point cloud as a graph by applying a kernel $\mathcal{K}$ (e.g., the Gaussian kernel $e^{-\frac{||z_1 - z_2||^2}{\sigma}}$) to the pairwise Euclidean distances between data points.

The kernel $\mathcal{K}$ is normalized to obtain a row-stochastic matrix $P$, where $P(z_1, z_2) = \frac{\mathcal{K}(z_1, z_2)}{||\mathcal{K}(z_1, \cdot)||_1}$. This matrix $P$ encodes transition probabilities between points. Powering $P^t$ represents a $t$-step random walk. Long-range or spurious connections are given less weight through this iterated walk than robust on-manifold paths, allowing the resulting point-wise \textit{diffusion probabilities} to recover manifold geometry even in the presence of sparsity and noise. Methods like Diffusion Maps, PHATE, and HeatGeo use diffusion probabilities to define a \textit{statistical distance} between data points~\citep{DiffusionMaps, PHATE, HeatGeo}.

\section{Riemannian Manifolds \& Metrics}
\label{appx:reimannian}

The \textit{Manifold Hypothesis} motivates our use of Riemannian geometry. Formally, an $n$-dimensional manifold $\mathcal{N}$ is a topological space that is locally homeomorphic to $\mathbb{R}^n$. Intuitively, while the global structure of $\mathcal{N}$ can be complex, every small region is similar to the Euclidean space.

A Riemannian manifold $(\mathcal{N},g)$ is endowed with a Riemannian metric $g$, which defines an inner product on the tangent space at each point. At each $x\in\mathcal{N}$, the metric $g_x$ assigns an inner product to tangent vectors $X,Y\in T_x\mathcal{N}$ via
$$
g_x(X,Y)=X^Tg(x)Y,
$$
where (with a slight abuse of notation) $g(x)$ denotes an $n\times n$ matrix representing the inner product on $T_x\mathcal{N}$. This metric allows us to measure angles and lengths. In particular, the length of a tangent vector $X$ is given by
$$
\|X\|=\sqrt{g_x(X,X)},
$$
and the length of a smooth curve $c:[0,T]\to\mathcal{N}$ is defined as
$$
L(c)=\int_0^T \sqrt{g_{c(t)}\bigl(\dot{c}(t),\dot{c}(t)\bigr)}\,dt.
$$
This expression computes the distance traveled along the curve, much like measuring a winding road on a flat map.

If the manifold is parametrized by a function $f(z)$ with $z\in\mathcal{D}$, its volume (or area, in the two-dimensional case) is calculated by
$$
\int_{\mathcal{D}} \sqrt{\det g(x)}\,dx.
$$
Here, $\sqrt{\det g(x)}$, called the volume element, quantifies how much local space is present at the point $x$.

\subsection{The Pullback Metric}
A key element of our approach is the \textit{Riemannian pullback metric}. Suppose we have a map between manifolds, $f:\mathcal{M}\to(\mathcal{N},g)$. At each point $x\in\mathcal{M}$, the differential
$$
df_x:T_x\mathcal{M}\to T_{f(x)}\mathcal{N}
$$
provides a linear approximation of $f$. Using this differential, the pullback metric $f^*g$ on $\mathcal{M}$ is defined by
$$
f^*g(X,Y)=g(df_xX,df_xY),
$$
for any tangent vectors $X,Y\in T_x\mathcal{M}$.

Intuitively, the pullback metric equips $\mathcal{M}$ with the geometry of $(\mathcal{N},g)$ as determined by $f$. It allows us to measure lengths, angles, and distances on $\mathcal{M}$ in a manner that reflects the geometry of the target space. This construction is fundamental to our method, as it bridges the geometry of $\mathcal{M}$ with the geometry provided by $f$.

For further details, we refer the reader to \citet[Chapters 0 and 1]{do1992riemannian}.

\section{Obtaining the Function $s(x)$}
\label{appx:auxiliary_dimension}

Recall that $s(x)$ provides an auxiliary dimension that complements the encoder $f_\theta$, where the value represents the deviation from the manifold. $s(x) \approx 0$ if $x$ is on the manifold, and $s(x)$ increases as $x$ moves away from the manifold.

\subsection{Approach 1: Discriminator}\label{appx:lipshitz}

There are various ways to assign the value in the auxiliary dimension. In our implementation, we employ a discriminative network~\citep{GAN} to predict whether a point is on or off the manifold.

To train the GAN-style discriminator, we first generate negative samples away from the data manifold in the data space by adding high-dimensional Gaussian noise to the data~(\Cref{expr:negative_pts}), where $c$ is a constant chosen such that the space away from the manifold is in the support of the distribution of $\check x$.

\vspace{-6pt}
\begin{equation}
    \check x_i = x_i+\epsilon_i ,~~\epsilon_i\sim\mathcal N(0, c I)
\label{expr:negative_pts}
\end{equation}
\vspace{-12pt}

Then, we define a discriminator $w_\psi$ that maps from the data space to a score, optimized by the loss function in \Cref{loss:off_manifold_penalty} inspired by Wasserstein Generative Adversarial Networks~\citep{WGAN}.

\vspace{-8pt}
\begin{align}
    \mathcal L_w(\psi)
    =
    \mathbb{E}_{\check x}\left[ w_\psi(\check x) \right]
    - \mathbb{E}_x\left[ w_\psi(x) \right] + \operatorname{Var}_{x}(w_\psi(x))
    % -\frac{1}{N}\sum_{i=1}^N w_\psi(x_i) + \operatorname{var}(\{w_\psi(x_i),i=1,...,N\}) .
\label{loss:off_manifold_penalty}
\end{align}
\vspace{-6pt}

% with the expectation and variance computed from the sample $\{x_1, ..., x_N \}$. 

$w_\psi$ is a Lipschitz function due to weight clipping and spectral normalization~\citep{WGAN, miyato2018spectral}.
The variance term is added to encourage the discriminator to have uniform predictions. Finally, we define the \methodshort embedding with auxiliary dimension in \Cref{expn:extn}.

% \begin{align}
%     f^{+}(x)
%     =
%     \left(\begin{matrix}
%         f_\theta(x)
%         \\
%         s(x)
%     \end{matrix}\right) ,
%     \label{expn:extn}
% \end{align}
% % where \mbox{$s(x)=\beta\frac{w_M-w_\psi(x)}{w_\psi(x)-w_m}$} with \mbox{$w_M = \max w_\psi(x)$}, \mbox{$w_m=\min w_\psi(x)$}, and $\beta$ is a hyperparameter.
% where \mbox{$s(x)=\beta{(\bar w-w_\psi(x))}$} with \mbox{$\bar w = \mathbb E_x[ w_\psi(x)]$}, and $\beta$ is a hyperparameter.

We have the following lemma showing that the condition ``$s(x) \approx 0$ if $x$ is on the manifold, and $s(x)$ increases as $x$ moves away from the manifold'' is achieved:
\begin{lemma}
\label{prop:extn}
    Suppose $w_\psi$ is $L$-Lipschitz, and $\max_{i,j}||x_i-\check x_j||\leq M$. for any $ \epsilon>0$, if 
    % $\mathcal L_{w}(\psi)\leq-L\sqrt{2c}\frac{\Gamma(\frac{n+1}{2})}{\Gamma(\frac{n}{2})}+\epsilon$
    $\mathcal L_{w}(\psi)\leq-LM+\epsilon$
    , we have $\mathbb E_x[s(x)^2]\leq \epsilon$.
\end{lemma}

\subsection{Approach 2: Gaussian Process}\label{appx:gp}

Alternative to the discriminator, we can also obtain $s(x)$ using the variance of a Gaussian process. We take advantage of the observation that the uncertainty (covariance) of a Gaussian process increases as the evaluation point moves away from the seen training point. We use an radial basis function kernel

\begin{equation}
    K(x,x')=\exp\left(-\frac{||x-x'||^2}{2\sigma^2}\right)
\end{equation}

in the model, and define $s(x)$ to be the posterior variance 

\begin{equation}
    s(x):=K(x,x)-K(x,X)[K(X,X)+\sigma_n^2I]^{-1}K(X,x),
\end{equation}

where $X=\{x_1,\dots,x_N\}$ is the data; $K(x, X):=(K(x,x_1),K(x,x_2),\dots,K(x,x_N)); K(X,x):=K(x, X)^T;$ and $ K(X,X):=\left(K(x_i,x_j)\right)_{i=1,\dots,N}^{j=1,\dots,N}$.

% \section{Geodesic Parameterization and Computation}

% \subsection{Parameterization of curve}
\section{Curve Parameterization for Generating Along Geodesics}
\label{append:curve_param}

We parameterize the curves using an interpolation between starting and ending points, with a linear term and a non-linear term parameterized by an MLP $
\gamma_\eta$.
\begin{align}
    c_\eta(x_0,x_1,t)
    =
    tx_1+(1-t)x_0+(1-(2t-1)^2) \gamma_\eta(x_0,x_1,t) ,
    \label{expn:geob}
\end{align}

% \subsection{Algrorithm for Geodesic Flow Matching}
% We use geodesic flow matching with a minibatch OT setup similar to~\citep{tong2023improving}
% \input{GRaM-workshop/algorithm_fm_geod.tex}

% \section{Justifications of Theoretical Assumptions}
% \subsection{\cref{lem:bilip}}

\section{Proofs of Lemmas and Propositions}\label{appx:proof}
\subsection{\cref{prop:dist_match}}
    For Riemannian manifolds $(\mathcal M,g_{\mathcal M}),(\mathcal N,g_{\mathcal N})$ and diffeomorphism $f:\mathcal M\to\mathcal N$, if $f$ is a local isometry, i.e., there exists $\epsilon>0, $ such that for any $x_0,x_1\in\mathcal M, d_{\mathcal M}(x_0,x_1)<\epsilon\implies d_{\mathcal M}(x_0,x_1)=d_{\mathcal N}(f(x_0),f(x_1))$, then we have $g_{\mathcal M}=f^*g_{\mathcal N}.$
\begin{proof}
    We first prove that the two metrics agree on vector norms. That is, for any $ u\in T_{x}\mathcal M, g_{\mathcal N}(df u,df u)=g_{\mathcal M}(u,u)$.:
    
    $\forall z\in\mathcal N,\forall$ smooth curve $\gamma(t)\subset\mathcal N,$ and let $\xi(t)=f^{-1}(\gamma(t))$. Then there exists $\delta>0$ such that $\forall 0<t<\delta$ 
    \begin{align}
        \int_{0}^{t}\sqrt{g_{\mathcal M}(\dot\xi(\tau),\dot\xi(\tau))}d\tau<\epsilon
    \end{align}
    We have
    \begin{align}
    \SwapAboveDisplaySkip
        \int_{0}^{t}\sqrt{g_{\mathcal M}(\dot\xi(\tau),\dot\xi(\tau))}d\tau=\int_{0}^{\gamma^{-1}\circ \xi(t)}\sqrt{g_{\mathcal N}(\dot\gamma(\tau),\dot\gamma(\tau))}d\tau
    \end{align}
    Take $t\to 0$, we have $g_{\mathcal N}(df u,df u)=g_{\mathcal M}(u,u)$ where $u=\dot\xi(0).$

    Next we use the identity
    \begin{align}
    \SwapAboveDisplaySkip
        \left<u,v\right>=\frac{1}{4}\left(\left<u+v,u+v\right>-\left<u-v,u-v\right>\right)
    \end{align}
    for any 2-form $\left<\cdot,\cdot\right>$, and apply to $g_{\mathcal M}, g_{\mathcal N}$, we have
    \begin{align}
        g_{\mathcal N}(df u,df v)=g_{\mathcal M}(u,v)\forall u,v\in T_{x}\mathcal M.
    \end{align}
    % $\forall z\in\mathcal N$ and two smooth curves $\gamma(t),\xi(t)\subset\mathcal N,\gamma(0)=\xi(0)=z, g_{\mathcal N}()$ 
\end{proof}

\subsection{\cref{prop:extn}}
    Suppose $w_\psi$ is $L$-Lipshitz, and $\max_{i,j}||x_i-\check x_j||\leq M$. $\forall \epsilon>0$, if $\mathcal L_{w}(\psi)\leq-LM+\epsilon$, we have $\mathbb E_x[s(x)^2]\leq \epsilon$.
\begin{proof}
    Denote $p_{\text{on}}$ the data distribution and $p_{\text{off}}$ the distribution of off-manifold points defined \cref{expr:negative_pts}.
    
    $\forall x\sim p_{\text{on}}, \check x\sim p_{\text{off}}$, since $w_\psi$ is $L$-Lipshitz, $\mathbb |w_\psi(\check x)-w_\psi(x)|\leq L||\check x-x||<LM.$ 
    
    Taking expectaion, we have $\mathbb E_{\check x}[w_\psi(\check x)]-E_{x}[w_\psi(x)]\geq-LM.$
    
    Thus, $\mathcal L_{w}(\psi)\leq-LM+\epsilon\implies \mathbb E[s(x)^2]= \operatorname{Var}_{x}(w_{\psi}(x))=\mathcal L_{w}(\psi)-(\mathbb E_{\check x}[w_\psi(\check x)]-E_{x}[w_\psi(x)])\leq \epsilon$.
\end{proof}

\subsection{\cref{lem:bilip}}
    If there exists $\alpha \in \mathbb{R}$ such that for any 
    % $x, \check x, \alpha||x-\check x||\leq |w_\psi(x)-w_\psi(\check x)|$.
    $x, \check x, \alpha||x-\check x||\leq |s(x)-s(\check x)|$.
    Then for any $x,\check x, ||f^{+}(x)-f^{+}(\check x)||\geq \alpha \beta ||x-\check x||$.
    Furthermore, denoting \mbox{$\mathcal D_{\mathcal M}(y):=\inf\nolimits_{x\in\mathcal M}||x-y||$} and %the distance from the point $y$ to the manifold $\mathcal M$
    \mbox{$\mathcal D_{f^{+}(\mathcal M)}(y):=\inf\nolimits_{x\in\mathcal M}||f^{+}(x)-f^{+}(y)||$}, then for any $\check x,$ we have $ \mathcal D_{f^{+}(\mathcal M)}(\check x)\geq \alpha \beta \mathcal D_{\mathcal M}(\check x)$.
\begin{proof}
    Because $r(x)
    =
    \left(\begin{matrix}
        f_\theta(x)
        \\
        s(x)
    \end{matrix}\right)$, where $s(x)=\beta{(\bar w-w_\psi(x))}$, we directly compute:
    \begin{align}
        ||r(x)-r(\check x)||^2=&||f_\theta(x)-f_\theta(\check x)||^2+|s(x)-s(\check x)|^2\\
        \geq& |s(x)-s(\check x)|^2\\
        \geq& \beta^2|w_{\psi}(x)-w_{\psi}(\check x)|^2\\
        \geq& \beta^2 \alpha^2 ||x-\check x||^2,
    \end{align}
    we have $||r(x)-r(\check x)||\geq \beta \alpha||x-\check x||$.

    Taking infimum over $x\in\mathcal M$, we have $\mathcal D_{r(\mathcal M)}(\check x)\geq \beta \alpha \mathcal D_{\mathcal M}(\check x)$
\end{proof}

\subsection{\cref{prop:volume_guidence}}
    Suppose $f_{\mathrm{target}}(x)=\lambda s(x)-\log(f_{vol})(x)$ is $\alpha$-strongly convex for some constant $\alpha>0$, i.e. $\nabla^2 f(x)\succeq \alpha I$, then the distribution of $X$ in \Cref{eqn:ld_vol} converges exponentially fast in Wasserstein distance to a distribution supported on the data manifold, whose restriction on the manifold is proportional to the volume distribution function.

\begin{proof}
    The proof follows from equation (1.4.9) in this textbook \url{https://chewisinho.github.io/main.pdf}:
    Suppose $f_{\mathrm{target}}$ is $\alpha$-strongly convex, for any $X_t\sim \mu_t, Y_t\sim\nu_t$ following the Langevin dynamics, initialized at $X_0\sim \mu_0,Y_0\sim\nu_0$, we have 
    \begin{equation}
        W_{2}^{2}(\mu_t,\nu_t)\leq e^{-2\alpha t}W_{2}^{2}(\mu_0,\nu_0).
    \end{equation}
   Now we check that $p(x)=\frac{1}{Z}e^{-f_{\mathrm{target}}(x)}$, where $Z=\int e^{-f_{\mathrm{target}}(x)}dx$, corresponds to a stochastic process governed by this SDE by showing that it satisfies the Fokker-Planck equation.
   \begin{align*}
       \text{LHS: }\frac{\partial p(x)}{\partial t}=0\\
       \text{RHS: }\nabla\cdot(p(x)\nabla f_{\mathrm{target}}(x))+\Delta p(x)\\
       =\frac{1}{Z}(\nabla\cdot(e^{-f_{\mathrm{target}}(x)}\nabla f_{\mathrm{target}}(x))+\Delta e^{-f_{\mathrm{target}}(x)})\\
       =\frac{1}{Z}(-\nabla\cdot\nabla e^{-f_{\mathrm{target}}(x)}+\Delta e^{-f_{\mathrm{target}}(x)})\\
       =0
   \end{align*}
   Therefore, for any initialization $X_0\sim \mu_0$, let $Y_0\sim p$, we have  
    \begin{equation}
        W_{2}^{2}(\mu_t,p)\leq e^{-2\alpha t}W_{2}^{2}(\mu_0,p).
    \end{equation}
    where $p(x)=e^{-\lambda s(x)}
    f_{vol}(x)$.
    Since $s(x)\approx 0$ if $x$ in on the manifold, and is large when $x$ is away from the manifold, we have $p(x)\approx f_{vol}(x)$ if $x$ is on the manifold, and $p(x)\approx 0$ if $x$ is away from the manifold.
\end{proof}
\subsection{\cref{prop:geod}}
Assume that the $\omega$-thickening of the manifold $\mathcal{M} \subset \mathbb{R}^n$, defined as
$
\mathcal{M}^{\omega} := \{ x \in \mathbb{R}^n : \inf_{m \in \mathcal{M}} d(x, m) < \omega \},
$
(i.e., the set of points whose distance from $\mathcal{M}$ is less than $\omega$) maps into a subset of the $\epsilon$-thickening of $f(\mathcal{M})$. Here, the $\epsilon$-thickening is defined analogously, with $\epsilon$ chosen such that for every $x \in f(\mathcal{M})$, the ball
$
B_\epsilon(x) := \{ y \in \mathbb{R}^n : \|y - x\| < \epsilon \}
$
intersects $f(\mathcal{M})$ in exactly one connected component.

Then, for any smooth curve $c:[0,1] \to \mathbb{R}^n$ connecting $x_0$ and $x_1$ (i.e., $c(0) = x_0$ and $c(1) = x_1$), there exists a smooth curve $c':[0,1] \to \mathcal{M}$ lying entirely \textbf{on the manifold} (with $c'(0) = x_0$ and $c'(1) = x_1$) such that
$
\mathcal{L}_{\text{Geo}}(c') \leq \mathcal{L}_{\text{Geo}}(c) - \alpha^2 \beta^2 \frac{1}{M} \sum_{m=1}^M \bigl( \mathcal{D}_{\mathcal{M}}(c(t_m)) - \mathcal{D}_{\mathcal{M}}(c(t_{m-1})) \bigr)^2 + \xi.
$
Here, $\alpha$ is defined as in \Cref{lem:bilip}, $\mathcal{D}_{\mathcal{M}}$ denotes the distance from a point to $\mathcal{M}$ (also as in \Cref{lem:bilip}), and $\xi$ is a fixed positive constant independent of $x_t$ and $\beta$.
\begin{proof}

    % For any smooth $c:[0,1]\to\mathbb R^n$, satisfying $c(0)=x_0,c(1)=x_1$, we can find a set $\{x_t\in\mathcal M:||f(x_t)-f(c(t))||<\epsilon, t\in[0,1]\}$. 
    % Furthermore, by the smoothness of $f$ and $c$, there exists a uniform $K>0$ independent on $c,c'$ such that we can find a curve $c'(t)\in\mathcal M, s.t. c'(0)=x_0,c'(1)=x_1, |\int \dot c(t)^TJ_f^TJ_f\dot c(t)-\dot c'(t)^TJ_f^TJ_f\dot c'(t)dt|<K\epsilon$.
    % Consider a smooth $c:[0,1] \to \mathbb{R}^n$ with $c(0)=x_{1}, c(0)=x_{1}$ which lies within the $\epsilon$-thickening of $\mathcal{M}$. 
    % By our `denoising autoencoder' assumption, the image $f(c)$ lies within $f(\mathcal{M})$. 
    % We additionally assume that $f$ restricted to $\mathcal{M}$ is bijective; hence the preimage of $f(c)$ is unique, and, by the smoothness of $f$, can be parameterized as a curve $c':[0,1] \to \mathbb{R}^n$ with the same endpoints as $c$. 
    % Moreover, with identical images, the lengths of $c$ and $c'$ computed through the pullback metric $J^T_{f}J_{f}$ are equal.

    Consider a smooth $c:[0,1] \to \mathbb{R}^n$ with $c(0)=x_{1}, c(0)=x_{1}$ which lies within the ${} \omega {}$-thickening of $\mathcal{M}$. 
    We construct an open cover of its image $f(c)$ as the collection of open balls $\{ B_{\epsilon}(c(t)) : t \in [0,1]\}$. By compactness, this admits a finite subcover at some collection of times $\{ t_{1} \dots t_{N} \}$. 
    For each $t_{i}$, we can choose point $c'[t_{i}]$ from $B_{\epsilon}(c(t_{i})) \cap f(\mathcal{M})$. By the continuity of ${} f \circ c {}$, these are all part of the same connected component of $f(\mathcal{M})$, hence there exists a curve $c':[0,1] \to \mathbb{R}^n$ with the same endpoints as $c$, whose image contains $\{c'[t_{i}]\}$. 
    Furthermore, by the smoothness of $f$ and $c$, there exists a uniform $K>0$ independent of $c,c'$ such that $|\int \dot c(t)^TJ_f^TJ_f\dot c(t)-\dot c'(t)^TJ_f^TJ_f\dot c'(t)dt|<K\epsilon$.
    Following \cref{prop:extn}, because $c'\in\mathcal M$, we also have $|\int \dot c'(t)^TJ_s^TJ_s \dot c'(t)|<\epsilon'$ for some uniform $\epsilon'>0$ independent on $c,c'$.
    
    We can decompose the pullback metric as
    \begin{align}
        J^T_{r}J_r=J_f^TJ_f+J_s^TJ_s.
    \end{align}
    and compute the difference
    \begin{align}
        \mathcal L_{\text{Geo}}(c)-\mathcal L_{\text{Geo}}(c')=&\frac{1}{M}\sum_{m=1}^M (\dot c(t)^TJ_f^TJ_f\dot c(t)+ \dot c(t)^TJ_s^TJ_s \dot c(t)-(\dot c'(t)^TJ_f^TJ_f\dot c'(t)+ \dot c'(t)^TJ_s^TJ_s \dot c'(t)))\\
        = &\frac{1}{M}\sum_{m=1}^M (\dot c(t)^TJ_f^TJ_f\dot c(t)-\dot c'(t)^TJ_f^TJ_f\dot c'(t) + \dot c(t)^TJ_s^TJ_s \dot c(t) + \dot c'(t)^TJ_s^TJ_s \dot c'(t))\\
        \geq & -K \epsilon - \epsilon' + \frac{1}{M}\sum_{m=1}^M\dot c(t)^TJ_s^TJ_s \dot c(t).\\
        \geq & -K \epsilon - \epsilon' -\epsilon''+ \frac{1}{M}\sum_{m=1}^M\dot (s(c(t_m))-s(c(t_{m-1}))^2\\
        \geq & -K \epsilon - \epsilon' -\epsilon''+ \frac{1}{M}\sum_{m=1}^M\dot (s(c(t_m))-s(c(t_{m-1}))^2.\\
        \geq & -K \epsilon - \epsilon' -\epsilon''+ \frac{1}{M}\alpha\beta\sum_{m=1}^M\dot (D_{\mathcal M}(c(t_m))-D_{\mathcal M}(c(t_{m-1}))^2,\\
    \end{align}
    where $\epsilon', \epsilon''$ are positive constants independent on $x_t,\beta$.
\end{proof}

\subsection{\cref{prop:geod_minim}}
    When $\mathcal L_{\text{Geo}}$ is minimized, $\underset{m=1,\dots,M}{\max} \mathcal D_{\mathcal M}(c(t_m))\le \frac{\sqrt{\xi}}{\alpha\beta}$, i.e., for sufficiently large $\beta$, $c(t)$ is close to the manifold with a maximum distance of $\frac{\sqrt{\xi}}{\alpha\beta}$.
    Furthermore, let $c'(t)$ be a geodesic between $x_0$ and $x_1$ under the metric $g_{\mathcal M}$, we have $\frac{1}{M} \sum_{m=1}^{M} {g_{\mathcal M} (\dot c, \dot c)}(x_0,x_1,t_m)\leq \frac{1}{M} \sum_{m=1}^{M} {g_{\mathcal M} (\dot c', \dot c')}(x_0,x_1,t_m)+\xi'\frac{\sqrt{\xi}}{\alpha\beta}$ for some positive constant $\xi'$. That is, $c$ approximately minimizes the energy (and hence curve length) under $g_{\mathcal M}.$
    % Furthermore, let $c'(t)$ be the geodesic between $x_0$ and $x_1$ (assume unique), after reparametrization matching the velocities (e.g. natural parameters), we have $\underset{m=1,\dots,M}{\max}||c(t_m)-c'(t_m)||\leq \frac{\sqrt{\xi}}{l\beta}$
\begin{proof}
    Suppose $c$ minimizes $\mathcal L_{\text{Geo}}$. 
    Then by \cref{prop:geod}, there exists $c'$ such that 
    \begin{align}
        \mathcal L_{\text{Geo}}(c')\leq\mathcal L_{\text{Geo}}(c)-\alpha^2\beta^2 \frac{1}{M}\sum_{m=1}^M (D_{\mathcal M}(c(t_m))-D_{\mathcal M}(c(t_{m-1})))^2+\xi.
    \end{align}
    On the other hand, because $c$ is a minimizer, we have 
    \begin{align}
        \mathcal L_{\text{Geo}}(c)\leq \mathcal L_{\text{Geo}}(c').
    \end{align}
    Combining them, we have
    \begin{align}
        \alpha^2\beta^2 \frac{1}{M}\sum_{m=1}^M (D_{\mathcal M}(c(t_m))-D_{\mathcal M}(c(t_{m-1})))^2\leq \mathcal L_{\text{Geo}}(c)-\mathcal L_{\text{Geo}}(c')+\xi
        \leq \xi.
    \end{align}
    Rearrange $t_0,\dots,t_M$ with a permutation $\sigma$ such that $D_{\mathcal M}(t_{\sigma(0)})\leq \dots\leq D_{\mathcal M}(t_{\sigma(M)})$, and because $D_{\mathcal M}(t_0)=0$ (the minimum), WLOG, let $t_{\sigma(0)}=0$.
    We have 
    \begin{align}
        \alpha^2\beta^2 \frac{1}{M}\sum_{m=1}^M (D_{\mathcal M}(c(t_{\sigma(m)}))-D_{\mathcal M}(c(t_{\sigma(m-1)})))^2\leq& \xi\\
        \implies \alpha^2\beta^2 (\frac{1}{M}\sum_{m=1}^M (D_{\mathcal M}(c(t_{\sigma(m)}))-D_{\mathcal M}(c(t_{\sigma(m-1)})))^2\leq& \xi \text{ (by Jensen's inequality)}\\
        \implies \alpha^2\beta^2 (D_{\mathcal M}(c(t_{\sigma(M)})-D_{\mathcal M}(c(t_{\sigma(0)}))^2\leq&\xi\\
        \implies \underset{m=1,\dots,M}{\max} D_{\mathcal M}(c(t_m))=D_{\mathcal M}(c(t_{\sigma(M)}))\leq \frac{\sqrt{\xi}}{\alpha\beta}.
    \end{align}
    The proof for the second part follows from the Lipshitz property of $s(x)$ and the smoothness of $f$ in \cref{prop:geod}.
\end{proof}
\subsection{\Cref{prop:flow_matching}}
    At the convergence of \Cref{alg:flow_matching}, 
    \begin{equation}\label{expn:optimal_geods}
        x(t)=x_0 + \int_{0}^t v_\nu (x_0,\tau)d\tau
    \end{equation}
    % \Cref{expn:optimal_geods}
    are geodesics between points in $\mathcal X$ and points in $\mathcal Y$ following the optimal transport plan that minimizes the geodesic lengths. 
\begin{proof}
    We first prove that when \Cref{eq:geo_curve_loss} and \Cref{eq:flow_loss} are minimized, \Cref{expn:optimal_geods} yields geodesics from $x_0\in\mathcal X$ to $x_1\in\mathcal Y$. 
    This is because by \Cref{prop:geod}, the curves $c_\eta$ are geodesics. When \Cref{eq:flow_loss} is minimized, $v_\nu$ approximates the gradient of $c_\eta$, and its integration starts at the same point $x_0$ approximates $c_\eta$.

    The rest follows from the the proof of Algorithm 3 in~\citep{Conditional_flow_matching}.

\end{proof}

\section{Additional Convergence Proposition for Volume-Guided Generation}\label{appx:vol_conv}
\begin{proposition}
Suppose when $X_t$ is initialized near the manifold $\mathcal{M}$, it stays in the neighborhood $\mathcal D:=\{X\in\mathcal M:\mathcal D_\mathcal M(X)\leq s\}$ near $\mathcal M$ with high probability up to time $T$, and that $\exp(-f_{\mathrm{target}}(x))$ satisfies Poincar\'e's inequality along and across the level sets of $f_{\mathrm{target}}$ near the manifold, then the distribution in \Cref{eqn:ld_vol} converges exponentially fast in total variation distance to a distribution supported on the data manifold, whose restriction on the manifold is proportional to the volume distribution function.
\end{proposition}
\begin{proof}
    This is a direct application of \citet[Theorem 4]{moitra2020fast}.
\end{proof}
These assumptions are attainable in our setup, given the fact that the manifold is bounded because the $\lambda s(x)$ term prevents the points from going far away from the manifold. In addition, if we restrict the domain to a ball containing the manifold, the function $e^{-f_{\mathrm{target}}(x)}$ is differentiable and hence satisfies Poincare’s inequality on this compact domain.

\section{Experiment Details}
\subsection{Geometry-aware autoencoder}
\label{appdx:experiment_details_ae}
\begin{table}[b!]
    \setlength{\tabcolsep}{2pt}
    \vspace*{-5pt}
    \centering
    \caption{Average DEMaP and DRS on simulated single-cell datasets over different noise settings.}
    % \scalebox{0.82}{
    \begin{tabular}[b]{ccccc}
      \toprule
      & Objective & State Space & DEMaP ($\uparrow$) & DRS ($\uparrow$) \\
      \midrule
      Autoencoder & $\mathcal{L}_\text{Recon}$ & Clusters & 0.347{\scriptsize \textcolor{gray}{$\pm$0.117}} & 0.642{\scriptsize \textcolor{gray}{$\pm$0.129}} \\
      \methodshort & $\mathcal{L}_\text{Recon}, \mathcal{L}_\text{Dist}$ & Clusters & \textbf{0.645}{\scriptsize \textcolor{gray}{$\pm$0.195}} & \textbf{0.667}{\scriptsize \textcolor{gray}{$\pm$0.165}} \\
      \midrule
      Autoencoder & $\mathcal{L}_\text{Recon}$ & Trajectories & 0.433{\scriptsize \textcolor{gray}{$\pm$0.135}} & \textbf{0.587}{\scriptsize \textcolor{gray}{$\pm$0.148}} \\
      \methodshort & $\mathcal{L}_\text{Recon}, \mathcal{L}_\text{Dist}$ & Trajectories & \textbf{0.600}{\scriptsize \textcolor{gray}{$\pm$0.191}} & 0.559{\scriptsize \textcolor{gray}{$\pm$0.143}} \\ 
      \bottomrule
    \end{tabular}
    % }
    \label{tab:demap_drs_extended}
    \vspace*{-12pt}
\end{table}
\subsubsection{Datasets: Splatter}
We evaluate our geometry-aware autoencoder on simulated scRNA-seq datasets Splatter\citep{zappia2017splatter}. Splatter uses parametric models to simulate cell populations with multiple cell types, structures, and differentiation patterns. Specifically, we evaluate on single-cell data of group and path structures with biological coefficient of variation (bcv) parameters $ \{0, 0.18, 0.25, 0.5\}$. A higher bcv corresponds to a lower signal-to-noise ratio. The cellular state space is a simulation parameter indicating whether the cells are arranged in clusters or trajectories in the data space. In Splatter, it is specified by the \texttt{method} parameter, where clusters correspond to \texttt{groups} and trajectories correspond to \texttt{paths}.

\subsubsection{Evaluation Criteria}
For the encoder, we leverage DEMaP~\citep{PHATE} to measure the correlation between Euclidean distances in latent space and ground truth geodesic distances in original data space.
\begin{align}
    \text{DEMaP}(f)=\frac{2}{N(N-1)}\sum_{i<j}\operatorname{Corr}(||f(x_i)-f(x_j)||_2,d_{ij}),
\end{align}
where $f$ is the encoder to be evaluated, $\operatorname{Corr}$ is Pearson correlation, $x_i, x_j$ are points from test data, and $d_{ij}$ is the ground truth geodesic distance between $x_i, x_j$, computed from shortest path distance under noiseless setting. 

For decoder evaluation, we propose a novel criteria, DRS (Denoised Reconstruction Score), to account for the noisy and sparse nature of single-cell data. DRS computes the correlation between reconstructed genes and denoised genes through denoising and imputation method MAGIC\citep{van2018recovering}.
\begin{align}
    \text{DRS($f, h$)}=&\frac{1}{N_{\text{gene}}}\sum_{i=1}^{N_{\text{gene}}}\operatorname{Corr}(y_i,y^{\text{MAGIC}}_i),
\end{align}
where $f, h$ are the encoder and decoder pair, $y_i=\operatorname{PCA^{-1}}(h(f(x_i))$, $y^{\text{MAGIC}}_i=\operatorname{PCA^{-1}}(\operatorname{MAGIC}(x_i))$. $\operatorname{PCA^{-1}}$ here is the inverse PCA operator since the original data are first PCA transformed and then fed into the autoencder. Therefore we use inverse PCA to map the reconstructed points back to the gene space for evaluation.

\subsection{Volume-guided Generation on Manifold}
\subsubsection{Generate imbalanced data on toy manifolds}
We generate imbalanced data on hemishpere, saddle, and paraboloid. \Cref{tab:toy_analytical} shows their parametrizations and volume elements.

In order to generate imbalanced data on the manifold, we generate $3,000$ points following a bivariate Gaussian distribution $\mathcal N\left(\left(\begin{matrix}
    1\\1
\end{matrix}\right),\left(\begin{matrix}
    2&0\\0&2
\end{matrix}\right)\right)$, with range restricted to $[-2,2]\times[-2,2]$. These points are used as parameters $(u,v)$, which we use to compute $(x,y,z)$ with the parametrizations in \Cref{tab:toy_analytical}.
These points $(x,y,z)\in\mathbb R^3$ are used as training points for \methodshort.

\begin{table}[tb!]
    \centering
    \begin{tabular}{ccc}
    \toprule
    \text{Manifold} & \text{Parametrization $(u,v)$} & \text{Volume Element $f_{{vol}}(u,v)$} \\
    \midrule
    Hemisphere &
    $\begin{cases} 
    x = u \\
    y = v \\
    z = \sqrt{1 - u^2 - v^2}
    \end{cases}$ &
    $\frac{1}{\sqrt{1 - u^2 - v^2}}$ \\
    \midrule
    Saddle &
    $\begin{cases} 
    x = u \\
    y = v \\
    z = u^2 - v^2
    \end{cases}$ &
    $\sqrt{1 + 4u^2 + 4v^2}$ \\
    \midrule
    Paraboloid &
    $\begin{cases} 
    x = u \\
    y = v \\
    z = u^2 + v^2
    \end{cases}$ &
    $\sqrt{1 + 4u^2 + 4v^2}$ \\
    \bottomrule
    \end{tabular}
    \caption{Parameterizations and volume elements of toy manifolds. We have access to the analytical forms of the volume elements computed from the parameterizations. We use them as the ground truth in evaluation.}
    \label{tab:toy_analytical}
\end{table}

\subsubsection{Details on evaluation metric for volume guided generation}\label{appx:vol_eval}
We evaluate the generated points by comparing its density estimation with the ground truth volume element in the parameter space.
We first convert the generated points in $\mathbb R^3$ back to the parameter space using $u=x,v=y$.
Then, we use apply kernel density estimation to the parameters $(u,v)\in\mathbb R^2$. 
We use a Gaussian kernel and use Scott's rule to determine the bandwidth.
To avoid the error from boundary effects of kernel density estimation, as well as the numerical instability of the volume element computation of the hemisphere near the boundary, we mask out the points near the boundary by only computing kernel density estimation and volume element on $\{(u,v):u^2+v^2<0.8\}$ for hemisphere, and $\{(u,v):|u|,|v|<1.6\}$ for saddle and paraboloid.
\subsection{Generating along geodesics}
\label{appdx:exp_details_geodesics}
\subsubsection{Datasets: Simulated manifolds}
We generate four toy manifolds: ellipsoid, torus, saddle, and hemisphere in $\mathbb{R}^3$. We add Gaussian noise of different scales to the original toy manifolds and rotate the data to higher dimensions using a random rotation matrix. We simulate datasets under $ \{0, 0.1, 0.3, 0.5\}$  noise scales and $\{3, 5, 10, 15\}$ dimensions. For each dataset, we randomly select 20 pairs of starting and ending points on the manifold. 

% To obtain approximate ground truth geodesics and geodesics lengths, we first compute the geodesics on the noiseless data using the analytical expressions (if available) or Dijkstra's algorithm. 

We benchmark all methods on the noisy, high-dimensional data, and compute the pairwise geodesics. 

\subsubsection{Evaluation Criteria}

Quantitatively, we evaluate these methods on the MSE criteria: the mean squared error between the predicted geodesic length and ground truth length.
\begin{align}
\SwapAboveDisplaySkip
    \text{Length MSE}=&\frac{1}{k}\sum_{i=1}^{k}(\hat l_{i} - l_i)^2,
\end{align}
where $k$ is the total number of geodesics, 
% $m$ is the number of points in each geodesic the model predicts, $g_i$ is the $i$-th ground-truth geodesic, $\hat g_i=(g_{i}^{(1)},\dots,g_{i}^{(m)})$ is the $i$-th predicted geodesic,
$l_i,\hat l_i$ are the lengths of the $i$-th ground truth and predicted geodesics.
We obtain the ground truth geodesics analytically if the solution is available or using Dijkstra’s algorithm on noiseless data otherwise.

\subsection{Geodesics-guided flow matching}\label{appx:geod_fm}
\subsubsection{Datasets: Randomly sampled populations on toy manifolds}
To showcase GAGA's ability on transporting distributions on manifolds, we generate four toy manifolds:  ellipsoid, torus, saddle, and hemisphere in $\mathbb{R}^3$. To simulate starting and ending distributions, we first randomly sample two points on the manifold as the starting and ending center and then sample $N$ points near these selected centers. We compute and visualize the flow paths between the two distributions.

\subsection{Hyperparameters}\label{appx:hparams}
We chose our hyperparameters using grid search on the validation sets. The hyperparameters we used for our experiments are the following:

For autoencoder training, both the encoder and decoder are multi-layer MLPs with hidden dimensions [256, 128, 64], [64, 128, 256] respectively. Each intermediate fully connected layer is followed by a spectral normalization layer, a batch normalization layer, a ReLU layer, and a dropout layer with 0.2 dropout probability. 

For training the discriminator $s(x)$, we use a multi-layer MLP with hidden dimensions  [256, 128, 64], and each intermediate fully connected layer is followed by a spectral normalization layer, a batch normalization layer, and a ReLU layer. 

For the geodesic-guided flow matching model, we use a multi-layer MLP with hidden dimensions [192, 192, 192] for curve parameterization and a multi-layer MLP with hidden dimensions [64, 64, 64] for the flow matching model. 

All models were trained with AdamW optimizer with learning rate 1e-3 and 1e-4 weight decay. The autoencoder was trained with 200 maximum epochs, the discriminator with 100 maximum epochs, and geodesic-guided flow matching with 100 maximum epochs. We used early stopping for all models, and the patience used is 50. 

We used the same set of loss weights in all experiments reported: $\lambda_1 = 77.4$, $\lambda_2 = 0.32$, $\zeta = 0.5$ for autoencoder loss (\Cref{loss:total}). $\beta = 10$ for the extended embedding (\Cref{expn:extn}). For volume-guided generation, we used $\lambda=10$ (\Cref{eqn:ld_vol}). For geodesic-guided flow matching, we used $\lambda_3 = 1$ and $\lambda_4 = 1$ (\Cref{eqn:loss_flow_matching}).

For applying GAGA on new datasets, we recommend starting with a relatively larger $\lambda_1$ and a smaller $\lambda_2$ for training the autoencoder. We found that $\lambda_1 = 77.4 $ and $\lambda_2=0.32 $ generally work well for single cell datasets. The much smaller $\lambda_2$ encourages the neural network to focus more on learning a good latent space instead of reconstructing the original signal since learning a latent space that preserves manifold distances is much more challenging than reconstruction. In addition, biological data are often very noisy, so better reconstruction does not necessarily aid in learning better representations. The decay parameter $\zeta$ encourages the latent space to focus more on matching local distances. We recommend starting with a relatively large $\beta$ for the extended embedding and a large $\lambda$ for volume-guided generation since it would place a significant penalty when generated points stray off from the manifold. In practice, we found $\beta=8$ and $\beta=10$ both work well in our experiments. For the geodesic-guided population transport, we recommend starting with equal $\lambda_3$ and $\lambda_4$ since we want to learn both the flow and the geodesic transportation path.

\section{Additional Experiment Results}\label{appx:addn_exp}

\subsection{Geometry-aware autoencoder under increasingly noisy data}

\begin{figure}[!ht]
\centering
\includegraphics[width=0.7\textwidth]{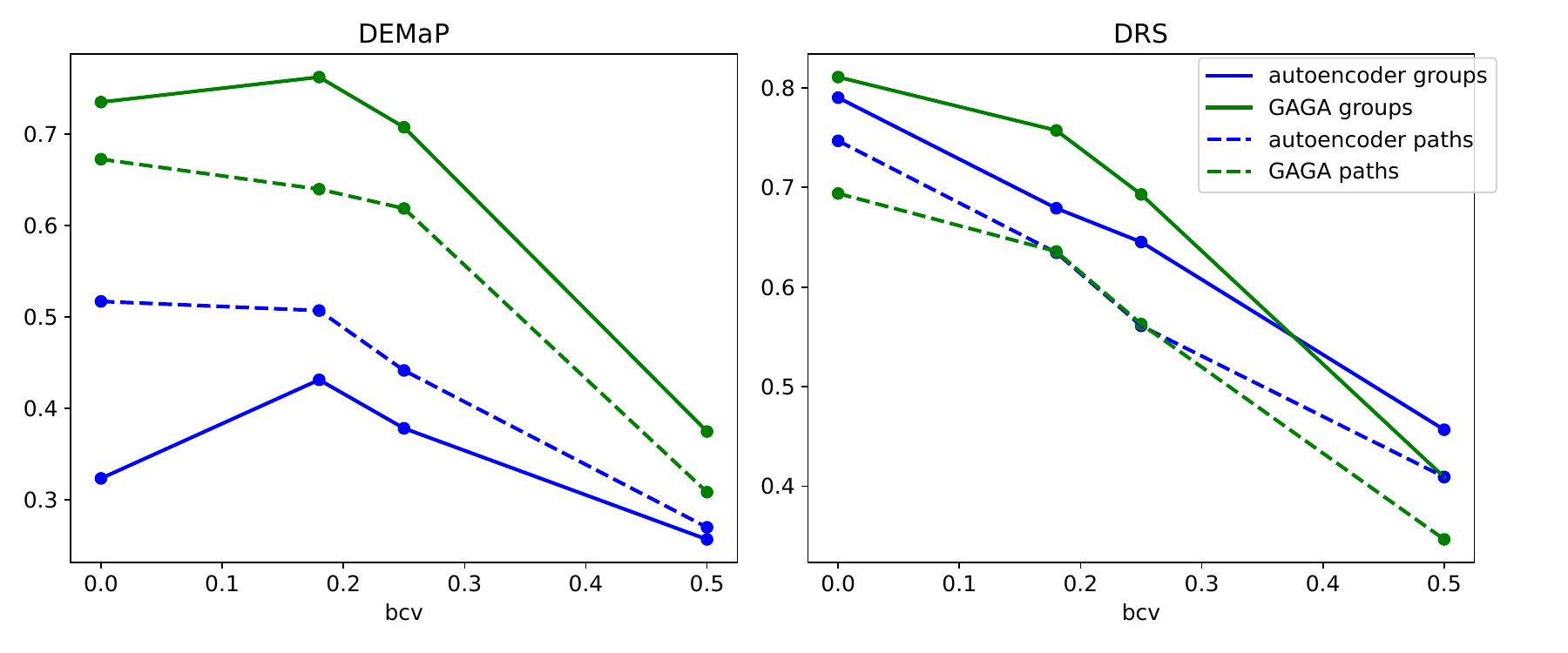}
\vspace*{-10pt}
\caption{Comparison for GAGA and standard autoencoder on increasingly noisy single-cell datasets. }
\label{fig:demap_bcv}
\end{figure}

In \Cref{fig:demap_bcv}, we observe that GAGA consistently outperforms standard autoencoder on DEMaP under increasingly noisy sinle-cell data simulated with increasing bcv parameter.
Moreover, we can see that GAGA generally rivals the standard autoencoder on DRS, indicating our distance-matching loss does not detract from data reconstruction.

\subsection{Visualizing GAGA's latent embeddings}
Qualitatively, we visualize the latent embeddings of GAGA on real-world scRNA-seq dataset EB, embryoid body data generated over 27 day time course~\citep{PHATE}.
We show that GAGA is able to capture geometric structures in the data, which are essential for biological insights and interpretations.
In addition to PHATE, we trained GAGA with two other geodesic distances obtained under different settings of HeatGeo~\citep{HeatGeo}.
We can see from \Cref{fig:embedding} that GAGA captures both local and global geometric structures such as clusters, branches, and paths. Moreover, \Cref{fig:embedding} shows that GAGA can match closely with the embedding method that it's based on, preserving the latent space of the original dimension reduction method and, at the same time, capable of generalizing to unseen points.

\begin{figure}[htbp]
    \vspace*{-10pt}
    \centering
    \includegraphics[width=0.98\textwidth]{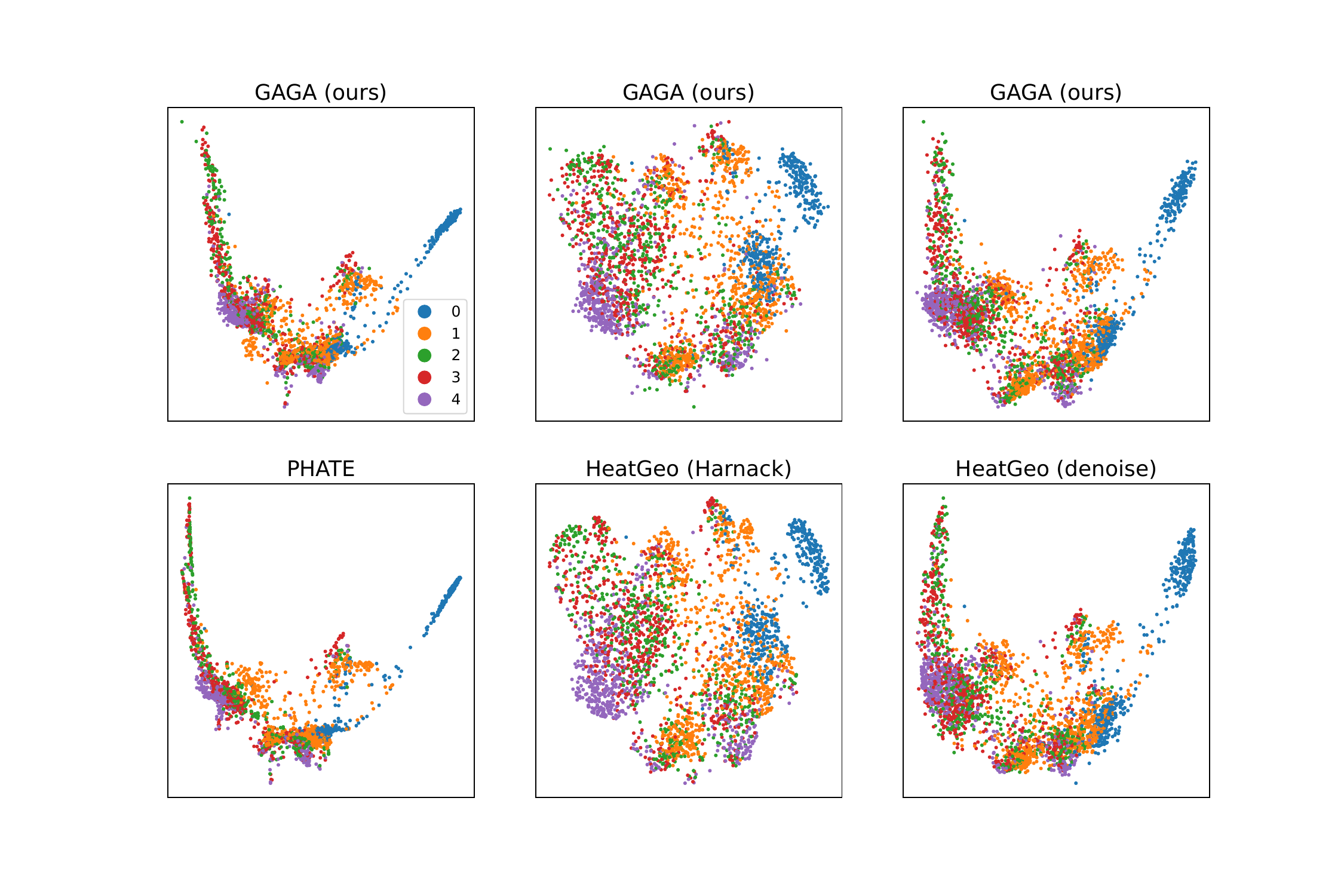}
    \vspace*{-25pt}
    \caption{
    Visualization of the embedding shows GAGA preserves local and global structures.
    }
    \label{fig:embedding}
\end{figure}

\subsection{Volume-guided Generation on Manifold}\label{appdx:results_unif_gen_full}
In \Cref{fig:unif_toy_full} (B,C,D) we show that the densities of the points generated by \methodshort are closer to the ground truth volume elements compared to the original data points, indicating that \methodshort largely reduces data imbalance.
In addition, \Cref{fig:unif_toy_full} (A) shows that the generated points stay on the data manifold and cover the sparse regions well in the original data.
\begin{figure}[htbp]
    \centering
    \includegraphics[width=0.9\textwidth]{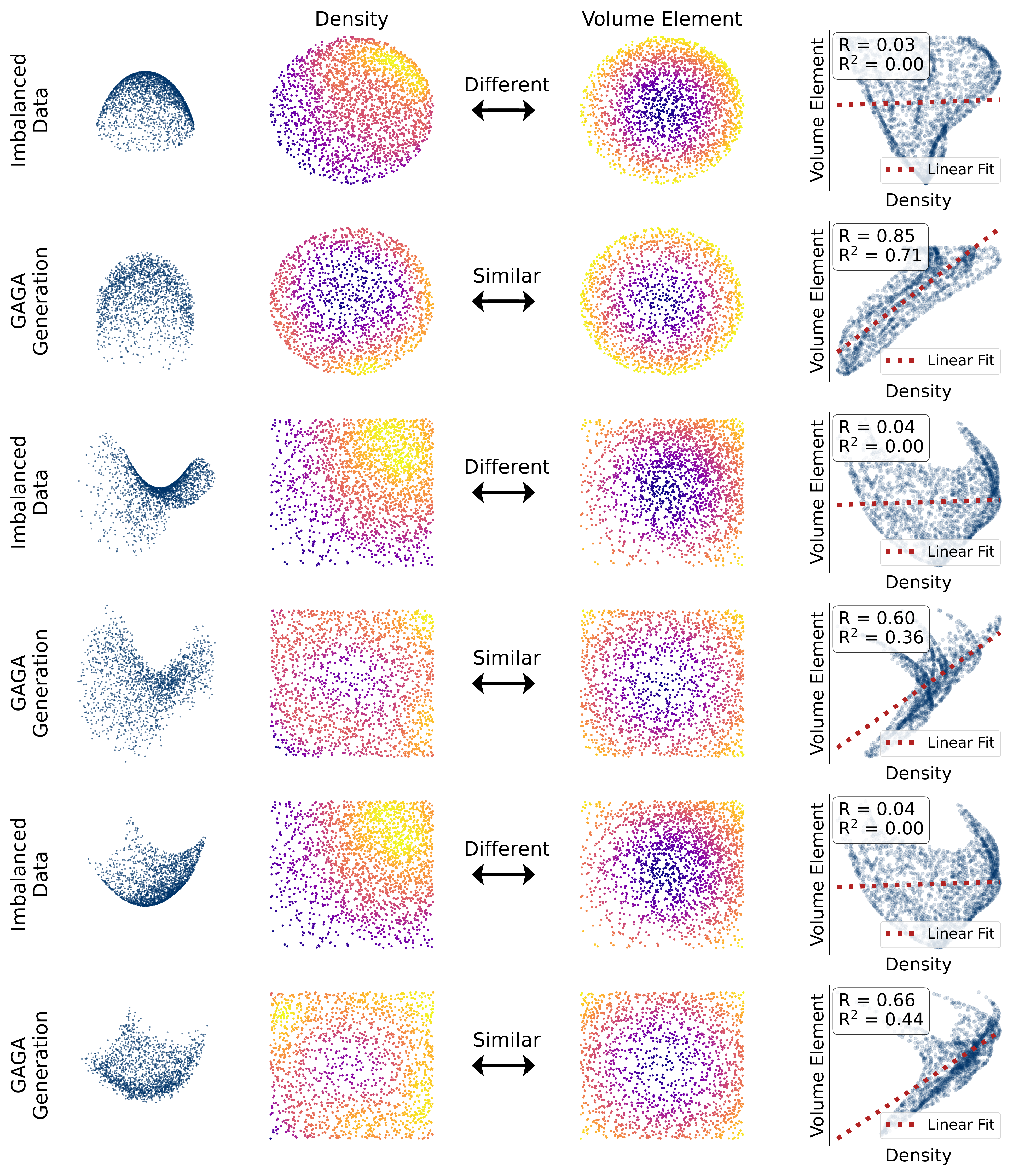}
    \caption{
        Geometry-aware generation with \methodshort on hemisphere saddle, and paraboloid. \textbf{(A)} Generated points remain on the manifold, and are more evenly distributed compared to raw data. \textbf{(B)} Kernel density estimation. \textbf{(C)} Ground truth volume elements computed analytically. \textbf{(D)} In raw data, density does not correlate to volume element, indicating data imbalance. \methodshort generation corrects the imbalance indicated by higher correlation between volume element and density.
    }
    \label{fig:unif_toy_full}
\end{figure}
\subsection{Comparing Volume-Guided Generation with On-Manifold Generation}\label{appx:RFM}
To assess the faithfulness of generated points to the data geometry, we compared our method with Riemannian Flow matching (RFM)~\citep{chen2023flow}. 
% We were unable to find the code of Diffusion Schrödinger Bridge (Thornton et al., 2022), so we reached out to the authors but haven't heard back from them yet. We are happy to add the comparison once we have access to the code. 
Notably, found that RFM only supports a number of specific manifolds in their implementation and that among the manifolds we conducted our experiment on, it only supports the hemisphere (supported through the sphere implementation). We used the hyperparameters for the sphere manifold (the volcano experiment) in their codebase. We present the comparison result in the following table:
\begin{table}[h]
    \centering
    \begin{tabular}{lcc}
        \toprule
        \textbf{Data} & \textbf{$R$} & \textbf{$R^{2}$} \\
        \midrule
        Original & -0.26 & 0.07 \\
        RFM Generated & -0.15 & 0.02 \\
        GAGA (Ours) Generated & 0.85 & 0.71 \\
        \bottomrule
    \end{tabular}
    \caption{Correlation $R$ and $R^2$ between the density and the volume element, where larger $R$ and $R^2$ indicate more faithful generation along the data geometry.}
    \label{tab:r_comparison}
\end{table}

We computed correlation $R$ and $R^2$ between the density and the volume element, where larger $R$ and $R^2$ indicate more faithful generation along the data geometry. Here “Original” refers to the original unbalanced dataset on which the models are trained. “RFM Generated” is the data generated by RFM, and “GAGA (Ours) Generated” is the data generated by our method. 

 Please refer to \Cref{appx:vol_eval} for detailed descriptions of how we generated the original dataset and computed the evaluation metrics. 

We observe that the RFM generated data exhibit weak correlations with the volume element, similar to the original unbalanced data. This occurs because the flow matching model learns the density of the training data rather than its geometry, making it unable to address sampling bias effectively. This limitation is illustrated in \Cref{fig:von_mises}.

\subsection{Geodesic computation in noisy data setting}
\label{appdx:results_geodesics_noisy_setting}
To better demonstrate the effectiveness of our method, especially in noisy data settings, we compared our method to Dijkstra’s algorithm in a setting of noise=0.7, dimension=15. We computed the mean squared error of geodesic lengths over 20 pairs of starting/ending points. To get a rigorous sense of significance, we used a Wilcoxon signed-rank test to compute the p-values (the null hypothesis is that the errors of the two methods are the same).

\begin{table}[h]
    \centering
    \begin{tabular}{lccc}
        \toprule
        \textbf{Dataset} & \textbf{GAGA (Ours)} & \textbf{Dijkstra} & \textbf{p-value} \\
        \midrule
        Ellipsoid  & 0.22  & 0.79  & 4.22e-03  \\
        Hemisphere & 2.25  & 5.67  & 3.22e-04  \\
        Saddle     & 2.73  & 6.34  & 1.99e-03  \\
        Torus      & 0.93  & 2.14  & 0.73      \\
        \bottomrule
    \end{tabular}
    \caption{Mean squared error of geodesic lengths of GAGA (Ours) vs. Dijkstra across different datasets under 0.7 noise scale and 15 dimensions.}
    \label{tab:gaga_vs_dijkstra}
\end{table}

We observe that our method significantly outperforms Dijkstra’s algorithm across all manifolds except the torus. 

Beyond the quantitative benchmarks, we would like to emphasize several fundamental advantages of our method over Dijkstra’s algorithm:
1) point generation: GAGA generates new points along the geodesic, whereas Dijkstra’s algorithm only connects existing points.
2) smoothness: GAGA learns smooth curves, while the curves produced by Dijkstra’s algorithm are discrete and prone to jittering, especially in the presence of noisy data.
3) geodesic insights: The smoothness of GAGA-generated curves allows us to compute other geometric quantities, such as velocities, providing valuable insights into the underlying manifold. For instance, we can compute the curvature of the geodesic.

\subsection{Visualizing geodesics on toy manifolds}
\label{appdx:results_geovis_full}

\Cref{fig:geovis_full} shows the geodesics of different methods on the same set of starting and ending points on multiple toy manifolds.
Each row corresponds to one manifold and each column corresponds to one method. From left to right column, the method is 1) ground truth, 2) GAGA, 3) local metric, 4) density regularization.
Density refers to geodesics learned with using density regularization.

We can see that GAGA generally outperforms all the other methods except Djikstra's on the saddle datasets. Directly using the local metric performs the worst, lagging far behind all other methods. 
The inferior performance of the local metric again illustrates the challenges of staying on the manifold while optimizing for the shortest path.

\begin{figure}[!thb]
    \centering
    \hspace*{5pt}
    \begin{minipage}[c]{0.27\textwidth}
    \centering
    \includegraphics[width=\textwidth]{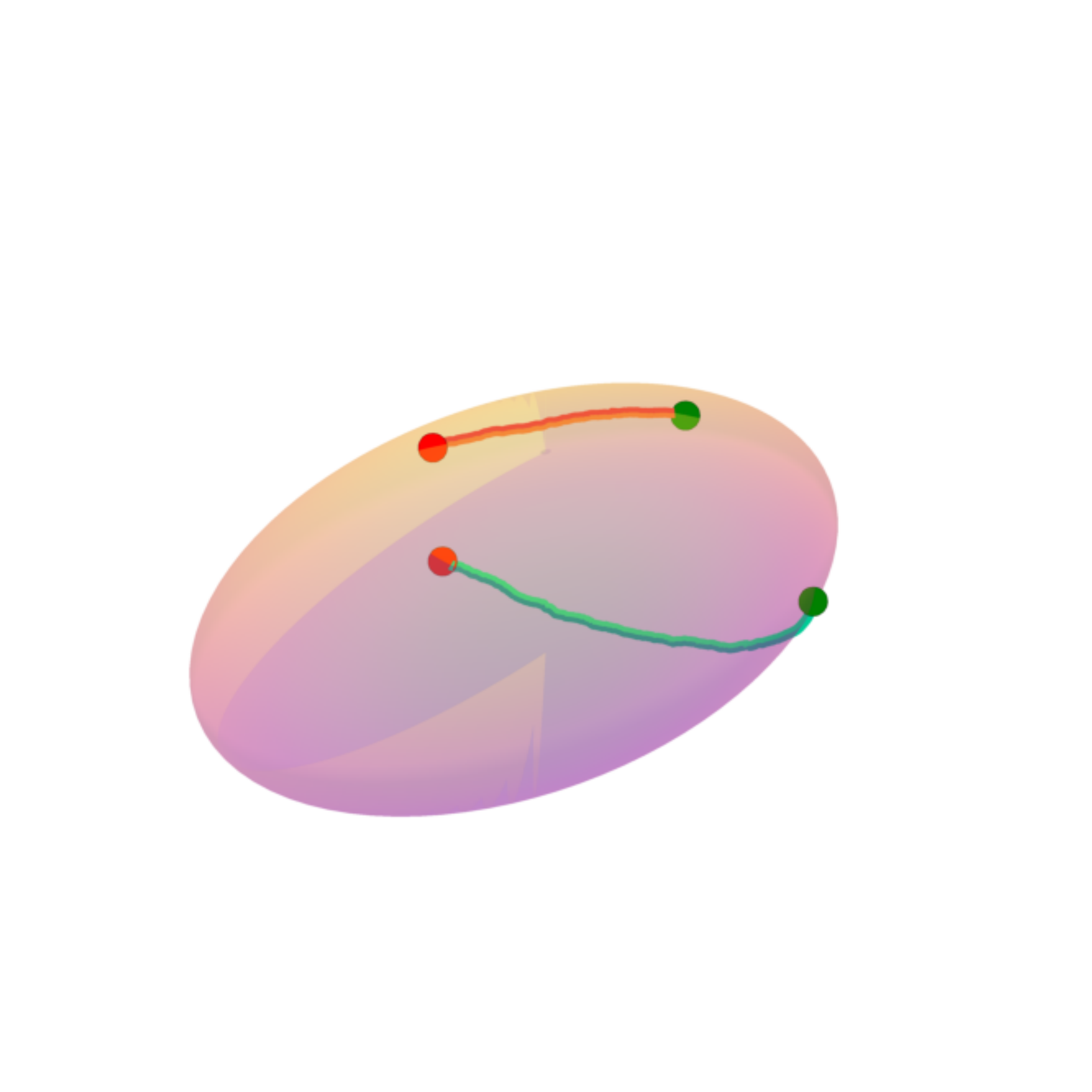}
    \end{minipage}
    % \hfill
    \hspace*{-27pt}
    \begin{minipage}[c]{0.27\textwidth}
    \centering
    \includegraphics[width=\textwidth]{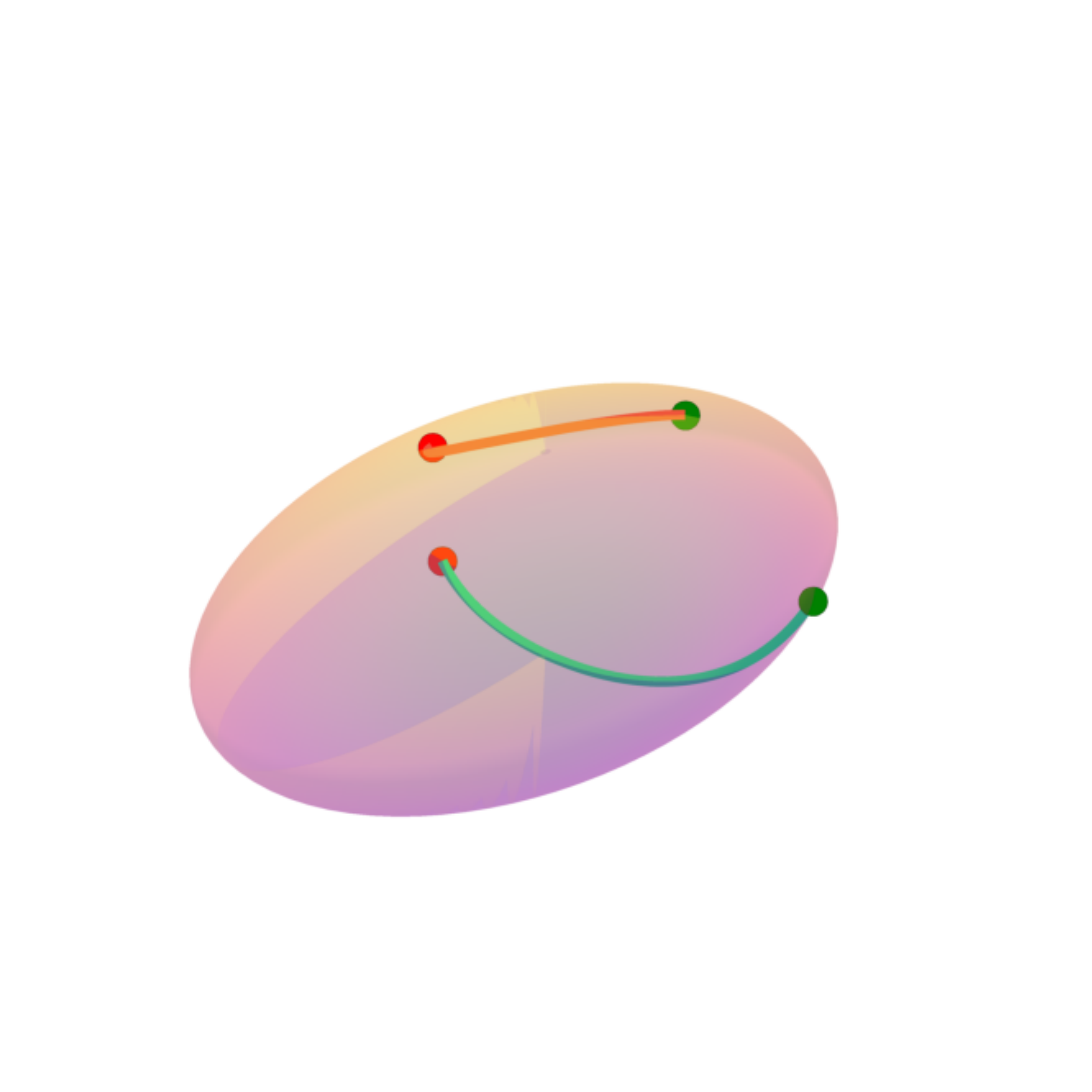}
    \end{minipage}
     % \hfill
     \hspace*{-27pt}
    \begin{minipage}[c]{0.27\textwidth}
    \centering
    \includegraphics[width=\textwidth]{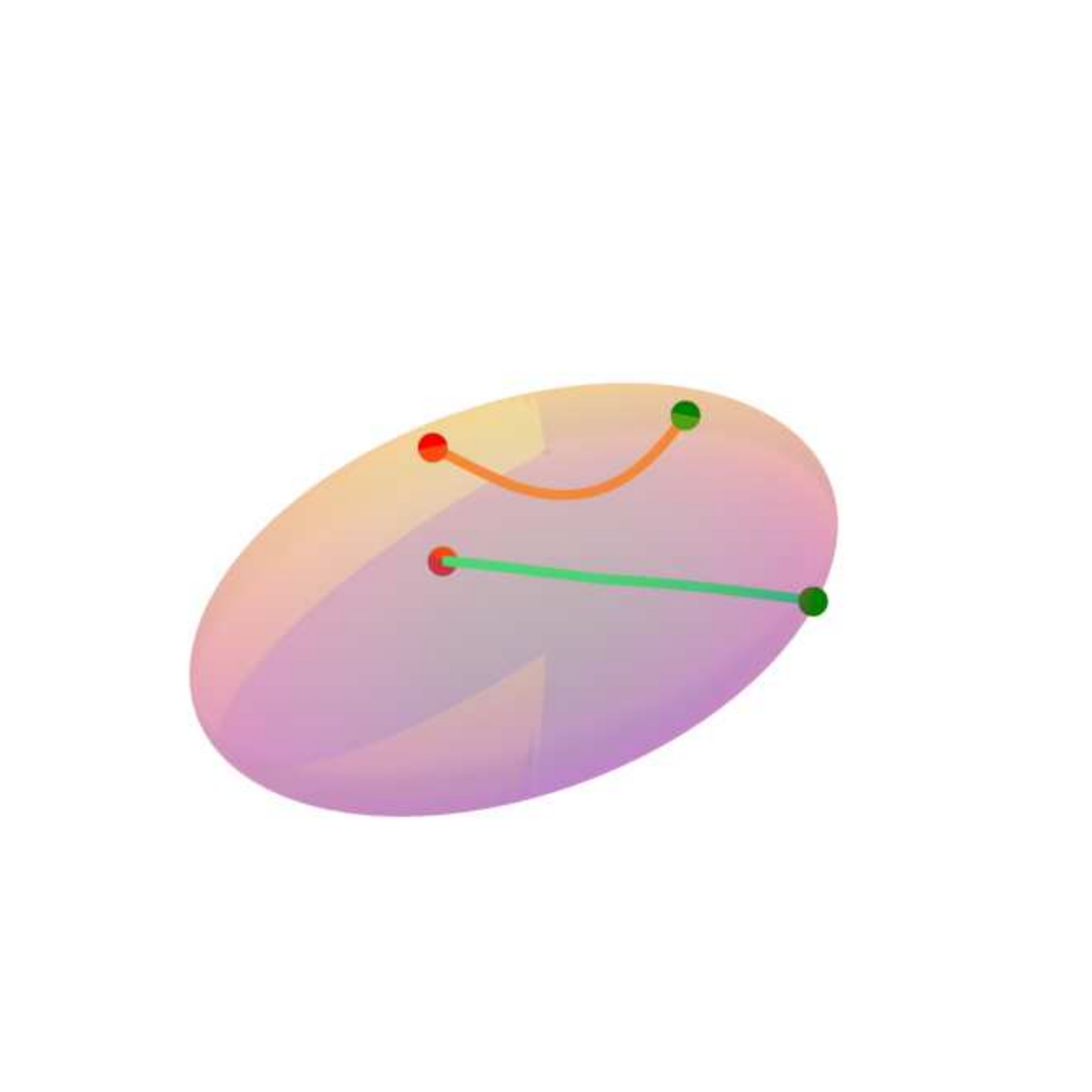}
    \end{minipage}
     % \hfill
     \hspace*{-27pt}
    \begin{minipage}[c]{0.27\textwidth}
    \centering
    \includegraphics[width=\textwidth]{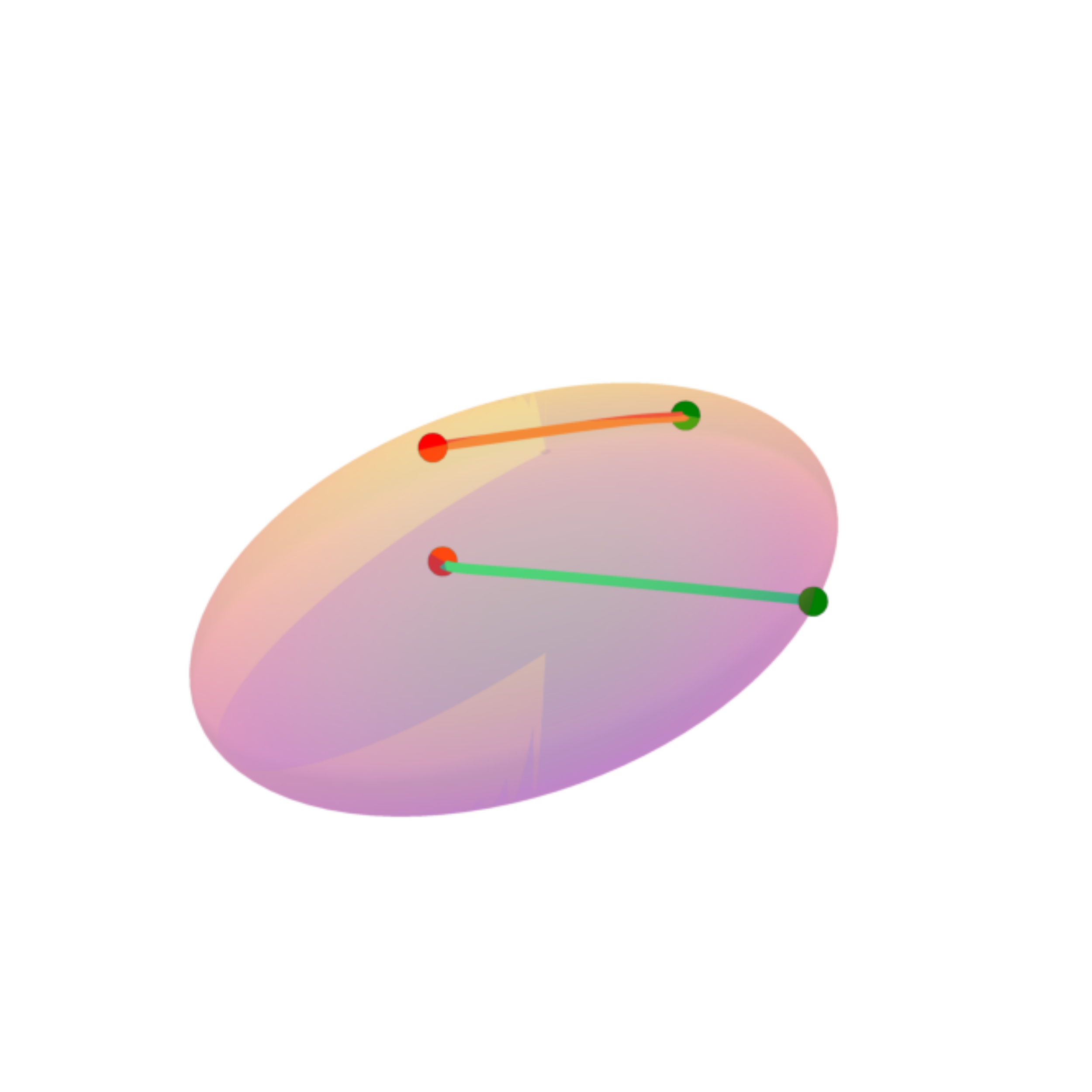}
    \end{minipage}
    \\[-30pt]

    %%%% 2nd Row %%%%%
    \begin{minipage}[c]{0.27\textwidth}
    \centering
    \includegraphics[width=\textwidth]{fig/geovis_gt_torus.pdf}
    \end{minipage}
    % \hfill
     \hspace*{-21pt}
    \begin{minipage}[c]{0.27\textwidth}
    \centering
    \includegraphics[width=\textwidth]{fig/geovis_ours_torus.pdf}
    \end{minipage}
     % \hfill
     \hspace*{-21pt}
    \begin{minipage}[c]{0.27\textwidth}
    \centering
    \includegraphics[width=\textwidth]{fig/geovis_no_density_torus.pdf}
    \end{minipage}
     % \hfill
     \hspace*{-21pt}
    \begin{minipage}[c]{0.28\textwidth}
    \centering
    \includegraphics[width=\textwidth]{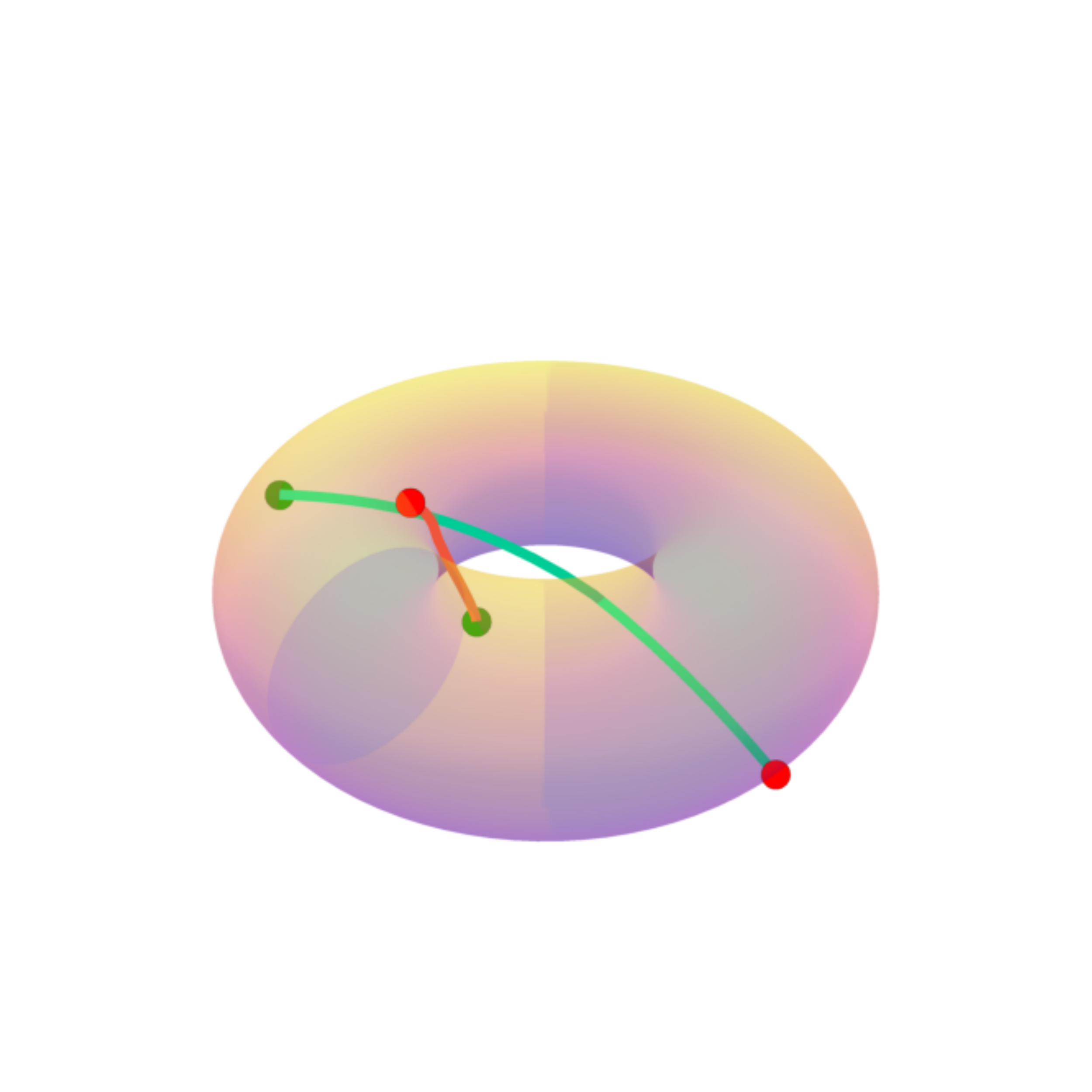}
    \end{minipage}
    \\[-30pt]

    %%%%% 3rd Row %%%%%
    \begin{minipage}[c]{0.27\textwidth}
    \centering
    \includegraphics[width=\textwidth]{fig/geovis_gt_saddle.pdf}
    \end{minipage}
    % \hfill
     \hspace*{-21pt}
    \begin{minipage}[c]{0.27\textwidth}
    \centering
    \includegraphics[width=\textwidth]{fig/geovis_ours_saddle.pdf}
    \end{minipage}
    % \hfill
     \hspace*{-21pt}
    \begin{minipage}[c]{0.27\textwidth}
    \centering
    \includegraphics[width=\textwidth]{fig/geovis_no_density_saddle.pdf}
    \end{minipage}
    % \hfill
     \hspace*{-21pt}
    \begin{minipage}[c]{0.27\textwidth}
    \centering
    \includegraphics[width=\textwidth]{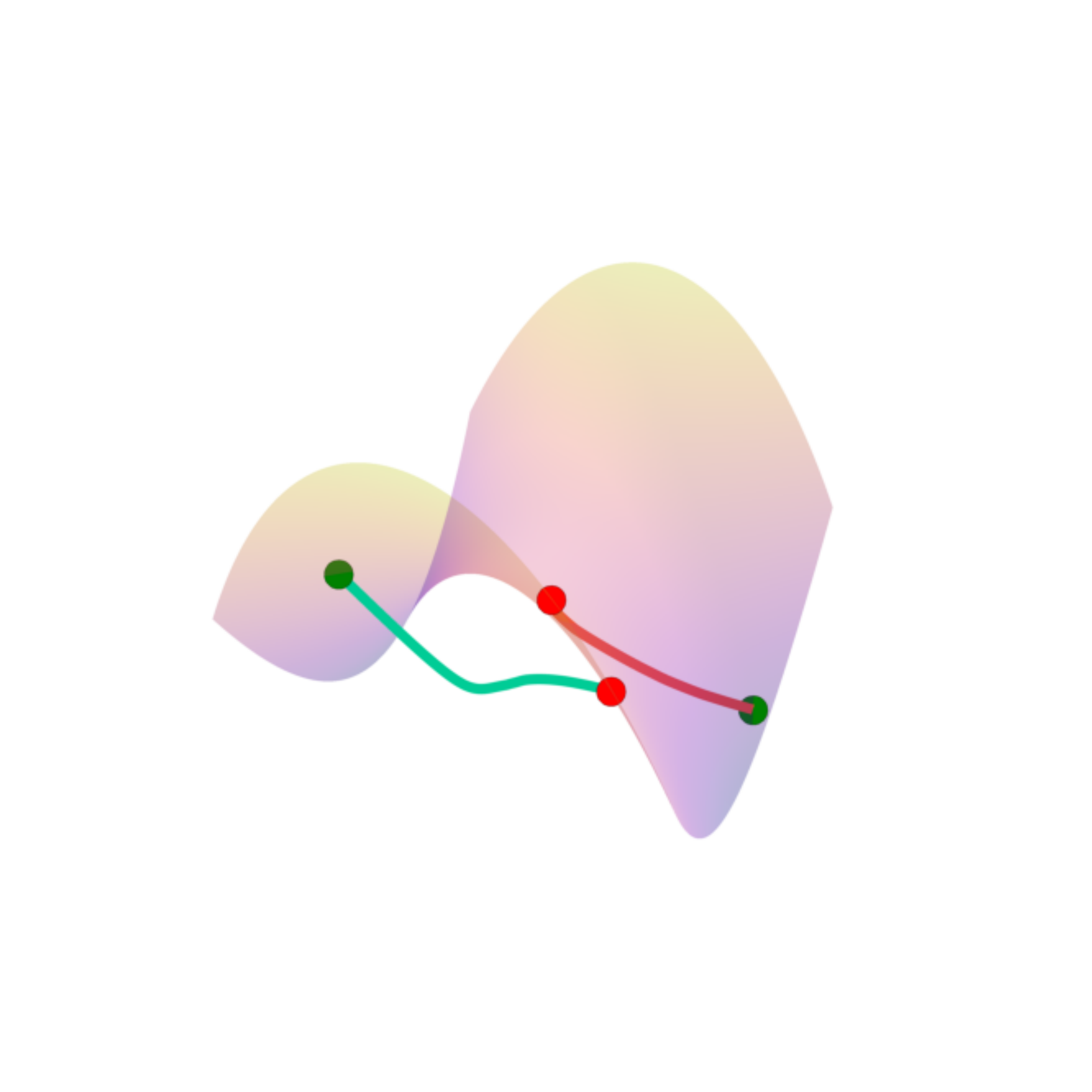}
    \end{minipage}
    \\[-30pt]

    %%%%%%% 4th Row %%%%%%%
    \begin{minipage}[c]{0.27\textwidth}
    \centering
    \includegraphics[width=\textwidth]{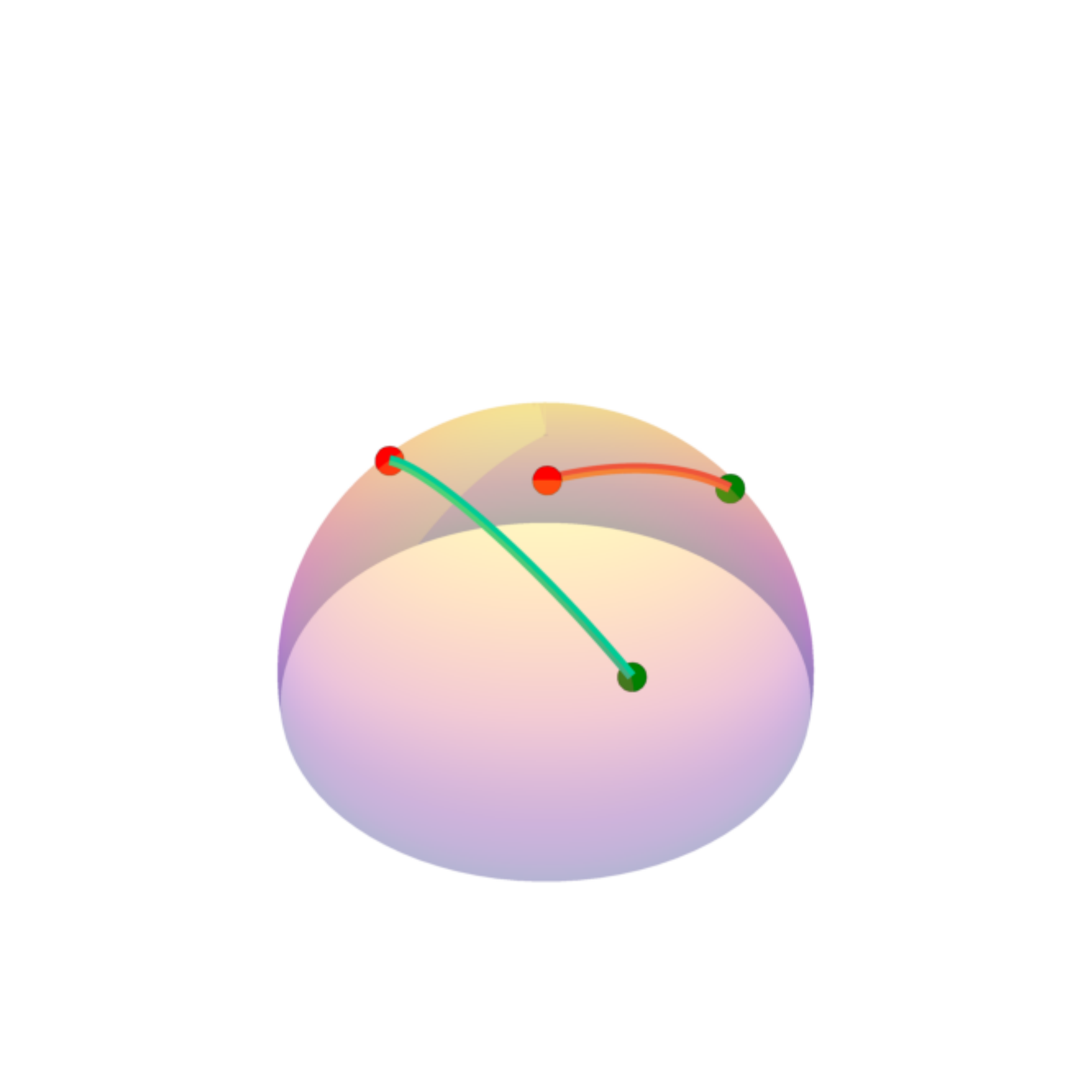}
    \end{minipage}
    % \hfill
     \hspace*{-21pt}
    \begin{minipage}[c]{0.27\textwidth}
    \centering
    \includegraphics[width=\textwidth]{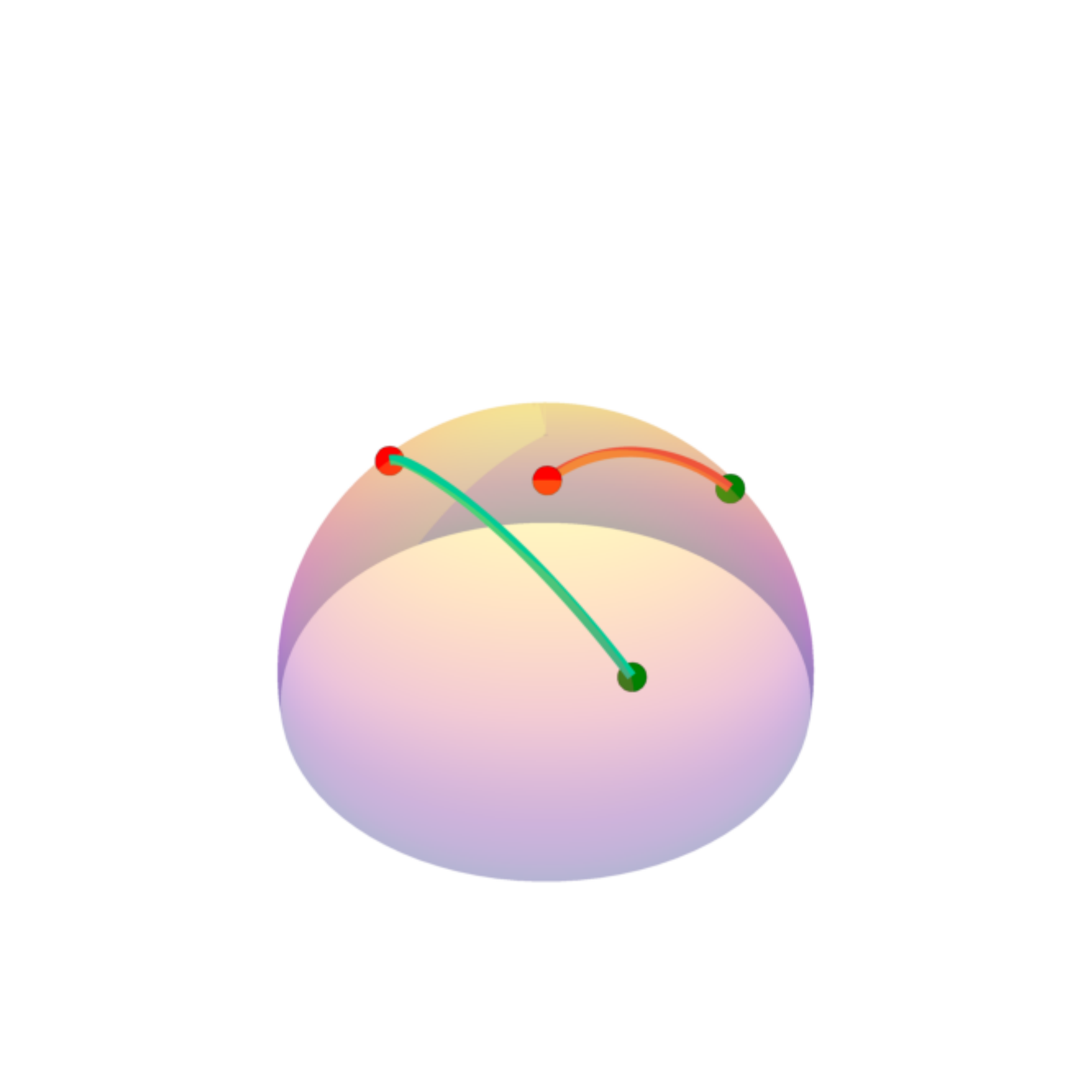}
    \end{minipage}
    % \hfill
     \hspace*{-21pt}
    \begin{minipage}[c]{0.27\textwidth}
    \centering
    \includegraphics[width=\textwidth]{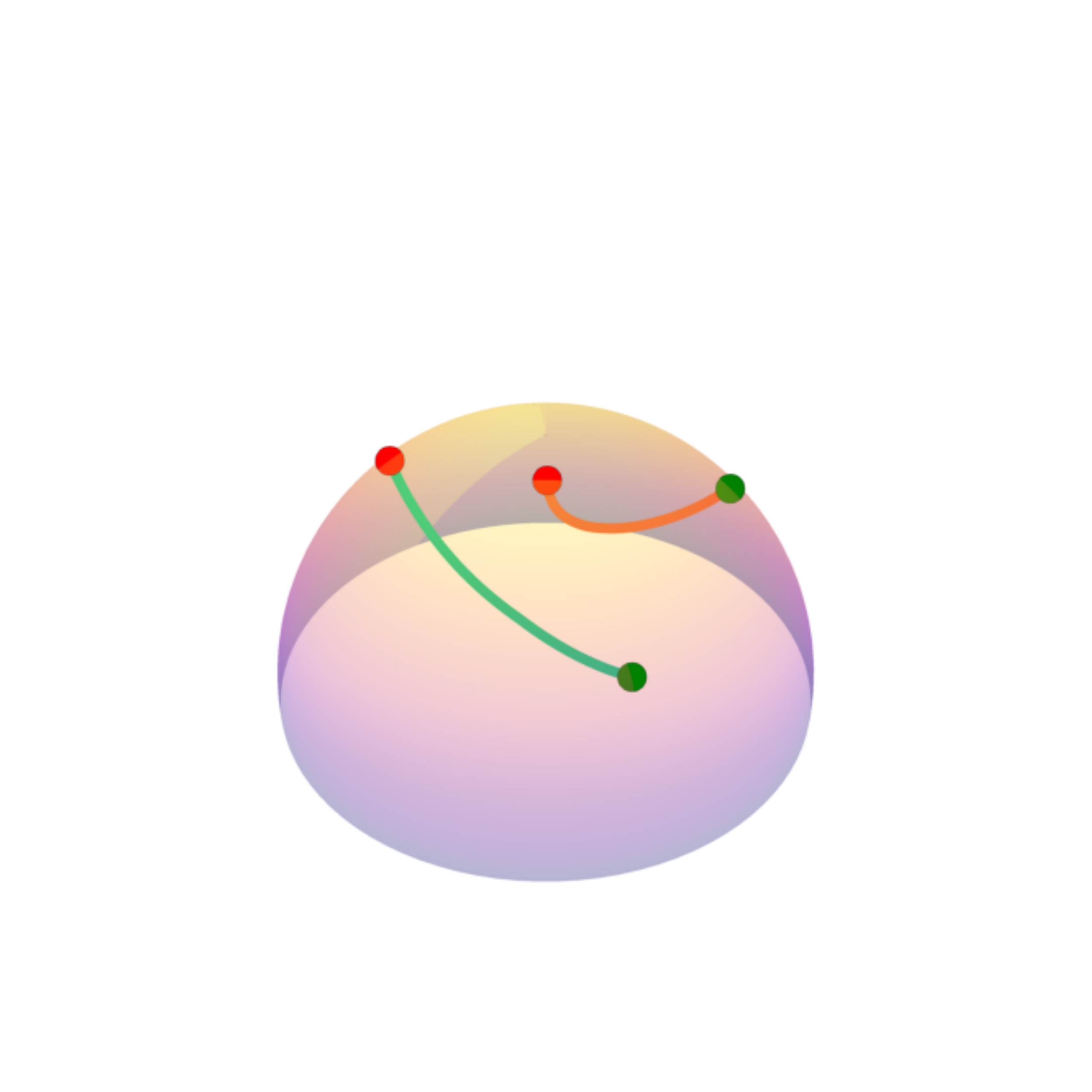}
    \end{minipage}
    % \hfill
     \hspace*{-21pt}
    \begin{minipage}[c]{0.27\textwidth}
    \centering
    \includegraphics[width=\textwidth]{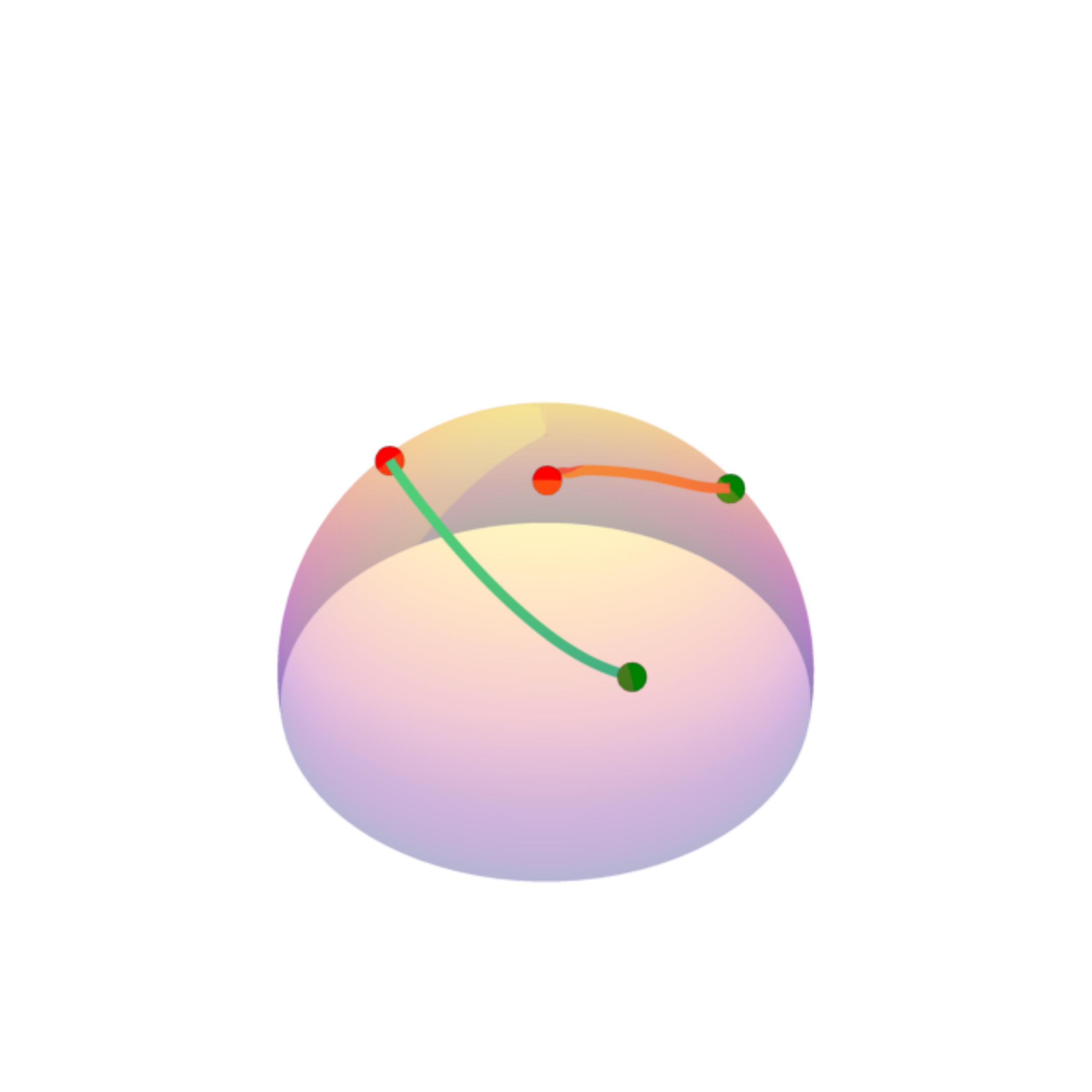}
    \end{minipage}
    
\caption{
    Comparison of geodesics.
    From left to right columns: 1) ground truth, 2) GAGA, 3) local metric, 4) density regularization.
}
\label{fig:geovis_full}
\end{figure}

\subsection{Single-cell trajectory inference}\label{appdx:results_single_cell_traj_inference}
Single-cell trajectory inference, a central task in cellular dynamics, aims to predict the continuous trajectories of cells over time. Specifically, we conducted left-one-timepoint-out experiment in which cells at one specific timepoint were excluded, and the goal is to predict the left-out cells using the cells from the remaining timepoints~\citep{TrajectoryNet}.

We repurposed the Cite and Multi single-cell datasets from the Multimodal Single-cell Integration Challenge at NeurIPS 2022~\citep{CITE_and_Multi}. Following the experiment setup in~\citep{SFSFM}, we trained and evaluated GAGA on donor 13176. For the Cite dataset, we combined both train and test inputs to obtain 29394 cells spanning from days 2, 3, 4, 7. For the Multi dataset, we used the train targets to obtain 35396 cells from days 2, 3, 4, 7. 

To perform left-one-timepoint-out experiment, we excluded day 3 and day 4, respectively, and used the remaining cells to infer the left-out populations. The train and test split ratio is 9:1, and the left-out timepoint was excluded from the training set. Our models were trained on the training set and evaluated on the test set. To reconstruct the left-out cells $X_{t}$ at time $t$ in the test set, GAGA generates the population level trajectories between $X_{t-1}$ and ${X_{t+1}}$ in the test set, and we use the points generated along the trajectories as the predicted cells $\hat{X_{t}}$. We ran experiments on 50 and 100 PCA dimensions of cells and the average Wasserstain-1 distance across the left-out timepoints was reported. The numbers listed for other methods were taken from the corresponding work.

\end{appendix}
\pagestyle{pagenumbers}

% \bibliography{AISTATS2025PaperPack/references_appendix}
% \bibliographystyle{apalike}

\end{document}